\documentclass{article}

\PassOptionsToPackage{numbers, compress}{natbib}
\usepackage[final]{neurips_2024}

\usepackage[utf8]{inputenc} % allow utf-8 input
\usepackage[T1]{fontenc}    % use 8-bit T1 fonts
\usepackage{hyperref}       % hyperlinks
\usepackage{url}            % simple URL typesetting
\usepackage{booktabs}       % professional-quality tables
\usepackage{amsfonts}       % blackboard math symbols
\usepackage{nicefrac}       % compact symbols for 1/2, etc.
\usepackage{microtype}      % microtypography
\usepackage[dvipsnames]{xcolor}         % colors
\usepackage{listings}
\usepackage{amsmath,amsthm,latexsym,amssymb,wasysym}
\usepackage{multirow}

\usepackage{bbm}
\usepackage{makecell}
\usepackage{wrapfig}

\usepackage{mathtools}
\usepackage{subcaption}
\usepackage{caption}

\usepackage[capitalize,noabbrev]{cleveref}

\theoremstyle{plain}
\newtheorem{theorem}{Theorem}[section]
\newtheorem{proposition}[theorem]{Proposition}

\newtheorem{corollary}[theorem]{Corollary}
\theoremstyle{definition}
\newtheorem{definition}[theorem]{Definition}

\theoremstyle{remark}
\newtheorem{remark}[theorem]{Remark}

\usepackage{array}
\newcommand{\PreserveBackslash}[1]{\let\temp=\\#1\let\\=\temp}
\newcolumntype{C}[1]{>{\PreserveBackslash\centering}p{#1}}
\newcolumntype{R}[1]{>{\PreserveBackslash\raggedleft}p{#1}}
\newcolumntype{L}[1]{>{\PreserveBackslash\raggedright}p{#1}}

\newcommand{\ip}[2]{{\langle #1, \, #2 \rangle}}

\usepackage{enumitem}

\usepackage[textsize=tiny]{todonotes}

\usepackage{smile}
\usepackage{amsmath,amsfonts,bm}

\def\eqref#1{eq.~(\ref{#1})}
\def\Eqref#1{Eq.~(\ref{#1})}
\def\1{\bm{1}}

\def\rd{{\mathrm{d}}}

\def\rmG{{\mathbf{G}}}

\def\rmK{{\mathbf{K}}}

\def\vx{{\bm{x}}}

\DeclareMathAlphabet{\mathsfit}{\encodingdefault}{\sfdefault}{m}{sl}
\SetMathAlphabet{\mathsfit}{bold}{\encodingdefault}{\sfdefault}{bx}{n}

\def\sR{{\mathbb{R}}}

\newcommand{\R}{\mathbb{R}}

\title{Dual Cone Gradient Descent for Training Physics-Informed Neural Networks}

\author{%
    Youngsik Hwang \\
    Artificial Intelligence Graduate School\\
    UNIST\\ 
   \texttt{hys3835@unist.ac.kr} \\
  \And
  Dong-Young Lim\thanks{Corresponding author.}\\
  Department of Industrial Engineering\\
  Artificial Intelligence Graduate School\\
    UNIST\\ 
   \texttt{dlim@unist.ac.kr} \\
}

\begin{document}

\maketitle

\begin{abstract}
Physics-informed neural networks (PINNs) have emerged as a prominent approach for solving partial differential equations (PDEs) by minimizing a combined loss function that incorporates both boundary loss and PDE residual loss. Despite their remarkable empirical performance in various scientific computing tasks, PINNs often fail to generate reasonable solutions, and such pathological behaviors remain difficult to explain and resolve. In this paper, we identify that PINNs can be adversely trained when gradients of each loss function exhibit a significant imbalance in their magnitudes and present a negative inner product value. To address these issues, we propose a novel framework for multi-objective optimization, \textit{Dual Cone Gradient Descent} (DCGD), which adjusts the direction of the updated gradient to ensure it falls within a dual cone region. This region is defined as a set of vectors where the inner products with both the gradients of the PDE residual loss and the boundary loss are non-negative. Theoretically, we analyze the convergence properties of DCGD algorithms in a non-convex setting. On a variety of benchmark equations, we demonstrate that DCGD outperforms other optimization algorithms in terms of various evaluation metrics. In particular, DCGD achieves superior predictive accuracy and enhances the stability of training for failure modes of PINNs and complex PDEs, compared to existing optimally tuned models. Moreover, DCGD can be further improved by combining it with popular strategies for PINNs, including learning rate annealing and the Neural Tangent Kernel (NTK). Codes are available at \texttt{ \textcolor{blue}{\url{https://github.com/youngsikhwang/Dual-Cone-Gradient-Descent}}}.
\end{abstract}

\section{Introduction}\label{sec:intro}

Physics-informed Neural Networks (PINNs) proposed in \citet{raissi2019physics} have created a new paradigm in deep learning for solving forward and inverse problems involving partial differential equations (PDEs). The key idea of PINNs is to integrate physical constraints, governed by PDEs, into the loss function of neural networks. This is in turn equivalent to finding optimal parameters for the neural network by minimizing a loss function that combines boundary loss and PDE residual loss. Thanks to their strong approximation ability and mesh-free advantage, PINNs have achieved great success in a wide range of applications~\citep{STRELOW2023gas, sharma2023review, Wiecha21nano, islam2021extraction, smith2022hyposvi, verma2024climode,ni2023sliced}.%, including fluid mechanics \citep{STRELOW2023gas, sharma2023review}, nano-photonics \cite{Wiecha21nano}, molecular dynamics \cite{islam2021extraction}, geoscience \cite{smith2022hyposvi} and so on.

Building upon this success, the applications of PINNs have been extended to solve other functional equations, including integro-differential equations \cite{yuan2022pinn}, fractional PDEs \cite{pang2019fpinn}, and stochastic PDEs \cite{zhang2020spde}. Moreover, numerous variants of PINNs have been developed to enhance their computational efficiency and accuracy via domain decomposition methods \citep{kharazmi2021hp, jagtap2021extended}, advanced neural network architectures \citep{wu2022physics,zhao2023pinnsformer,Cho2023Hyper,Cho2023Separable,han2023hierarchical}, modified loss functions \citep{yu2022gradient, son2023enhanced, NEURIPS2022_374050dc}, different sampling strategies \citep{wang2022respecting, wu2023comprehensive, Daw2023R3}, and probabilistic PINNs \citep{vadeboncoeur2023fully, vadeboncoeur23aRandomgrid}. 

Despite these achievements, several studies have reported that PINNs often fail to learn correct solutions for given problems ranging from highly complex to relatively simple PDEs \citep{Krishnapriyan2021failure, wang2021understanding, wang2022NTK}. Due to the unclear nature of pathologies in the training of PINNs, it has become a critical research topic to explain and mitigate these phenomena. For example, \citep{steger2022how, wong2022learning} observed that PINNs tends to get stuck at trivial solutions while violating given PDE constraints over collocation points. The imbalance between PDE residual loss and boundary loss was explored in \citet{wang2021understanding}, and a spectral bias of PINNs was studied in \citet{wang2022NTK}. \citet{pmlr-v202-yao23c} discussed the gap between the loss function and the actual performance. Even with the insights from the aforementioned studies, a comprehensive understanding of PINN's failure modes remains largely unexplored in various scenarios. 

In this paper, we explore these mysterious challenges from a novel perspective of multi-objective optimization. We first provide a geometric analysis showing that PINNs can be adversely trained when the gradients of each loss function exhibit a significant imbalance in their magnitudes, coupled with a negative inner product value. Based on this finding, we characterize a dual cone region where both PDE residual loss and boundary loss can be decreased harmoniously without the adverse training phenomenon. We then propose a novel optimization framework, \textit{Dual Cone Gradient Descent} (DCGD), for training PINNs which updates the gradient direction to be contained in the dual cone region at each iteration. Furthermore, we study the convergence properties of DCGD in a non-convex setting. In particular, we find that DCGD can converge to a Pareto-stationary point. We validate the superior empirical performance and universal applicability of DCGD through extensive experiments.

\section{Preliminaries}

\paragraph{Notation.} The Euclidean scalar product is denoted by $\ip{\cdot}{\cdot}$, with $\|\cdot\|$ standing for the Euclidean norm (where the dimension of the space may vary depending on the context). For a subspace $W$ of a vector space $V$, its orthogonal complement $W^\perp$ is defined as 
$$
W^\perp:= \{v\in V| \ip{u}{v}=0, \quad u\in W \}.
$$
For a vector $v\in V$, the projection of $v$ on a nontrivial subspace $W$ is denoted by $v_{\| W}$. Unless otherwise specified, $V$ represents $\R^d$ throughout the paper. 

\paragraph{Related Works.} Among various research directions in PINNs, we focus on reviewing optimization strategies for PINNs. These can be broadly categorized into three main approaches: adaptive loss balancing, gradient manipulation, and Multi-Task Learning (MTL). As an example of adaptive loss balancing algorithms, \citet{wang2021understanding} proposed a learning rate annealing (LRA) algorithm that balances the loss terms by utilizing gradient statistics. \citet{wang2022NTK} utilized the eigenvalues of the Neural Tangent Kernel (NTK) to address the disparity in convergence rates among different losses of PINNs. For gradient manipulation algorithms, the Dynamic Pulling Method (DPM) was proposed in \cite{kim2021dpm} to prioritize the reduction of the PDE residual loss. In \cite{bahmani2021training}, the authors used the PCGrad algorithm, proposed in \cite{Yu2020PCGrad}, for training PINNs to address multi-task learning challenges. \citet{li2023physics} developed an adaptive gradient descent algorithm (AGDA) that resolves the conflict by projecting boundary condition loss gradient to the normal plane of the PDE residual loss gradient. \citet{pmlr-v202-yao23c} recently developed MultiAdam, a scale-invariant optimizer, to mitigate the domain scaling effect in PINNs. Another important line of gradient manipulation involves Multi-Task Learning (MTL) algorithms, which optimize a single model to perform multiple tasks simultaneously \citep{sener2018mgda, desideri2012mgda, Yu2020PCGrad, Liu2021CAGrad, liu2021IMTL, navon2022nash, senushkin2023AlignmentMTL}. We will discuss that several MTL algorithms can be unified within the proposed DCGD framework. 
\paragraph{Physics-Informed Neural Networks.} Let $\Omega \subseteq \R^d $ be a domain and $\partial \Omega$ be the boundary of $\Omega$. We consider the following nonlinear PDEs:
\begin{equation}\label{eq:pde}
\begin{split}
    \mathcal{N}[u](\vx) &= f(\vx), \quad \vx \in \Omega \\
    \mathcal{B}[u](\vx) &= g(\vx), \quad \vx \in \partial \Omega
\end{split}
\end{equation}
where $\mathcal{N}$ and $\mathcal{B}$  denote a nonlinear differential operator and a boundary condition operator, respectively. We approximate $u(\vx)$ by a deep neural network $u(\vx;\theta)$ parameterized by $\theta$. To train the neural network, the framework of PINNs minimizes the total loss function $\cL (\theta)$, which is a weighted sum of PDE residual loss $\cL_r(\theta)$ and boundary condition loss $\cL_b(\theta)$, defined by:
\begin{align}
 \cL(\theta) := \omega_r \cL_r(\theta) &+ \omega_b \cL_b(\theta)\quad \mbox{with }  \label{eq:total_loss} \\
 \cL_r(\theta) := \frac{1}{N_r} \sum^{N_r}_{i=1}{|\cN[u(\cdot;\theta)](\vx_r^{i}) - f(\vx_r^{i})|^2}, &\quad 
 \cL_b(\theta) := \frac{1}{N_b} \sum^{N_b}_{i=1}{|\cB[u(\cdot;\theta)](\vx_b^i)-g(\vx_b^i)|^2}, \nonumber
\end{align}
where $\omega_r, \omega_b \geq 0$ are weights of each loss term, $\{\vx_r^i\}_{i=1}^{N_r}$ denotes a set of collocation points that are randomly sampled in $\Omega$, and $\{\vx_b^i\}_{i=1}^{N_b}$ the boundary sample points. Here, we set $\omega_r=\omega_u=1$ throughout the paper.  We note that the training of PINNs falls into the category of multi-objective learning due to its structure of the loss function $\cL(\theta)$ in \Eqref{eq:total_loss}. 
\vspace{-5pt}

\section{Empirical Observations and Issues in Training PINNs}\label{sec:motivation}
\begin{wrapfigure}{r}{.4\textwidth}
\vspace{-10pt}
    \centering
    \includegraphics[width=0.4\textwidth]{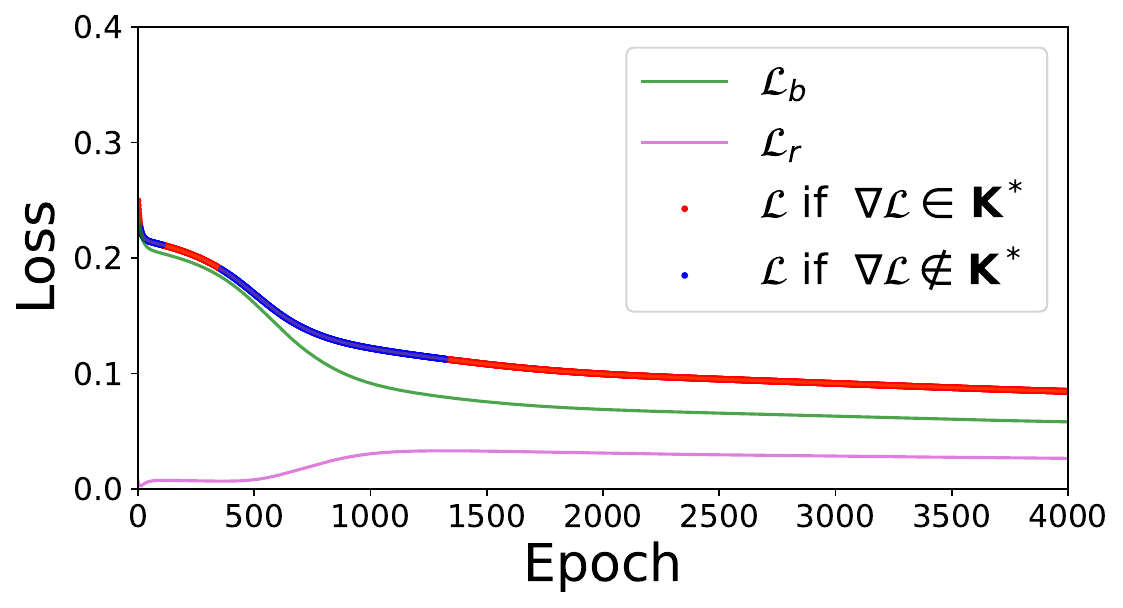} \caption{Training curves for the total loss $\cL$ $(:=\cL_r+\cL_b)$, PDE residual loss $\cL_r$, and boundary loss $\cL_b$ for viscous Burgers' equation.}\label{fig:adverse_training} %The PINN is trained using SGD. When $\cL \notin \rmK^*$, one of $\cL_r$ and $\cL_b$ may increase adversely even if $\cL$ continuously decreases.}\label{fig:adverse_training}
%\vspace{-10pt}
\end{wrapfigure}

This section investigates issues that are frequently observed during the training of PINNs in the context of multi-objective learning. The parameter for the PINN solution $u(\vx;\theta)$ is typically estimated by minimizing the total loss function $\cL(\theta)$ with a (stochastic) gradient descent method\footnote{In practice, adaptive gradient descent algorithms such as ADAM~\cite{kingma2014adam} are widely employed.}:

\[
\theta_{t+1} = \theta_t - \lambda \nabla \cL(\theta_t), \quad t \in \mathbb N
\]

where $\nabla \cL(\theta)$ is the gradient of the total loss function $\cL(\theta)$ with respect to $\theta$. However, a careless adoption of standard gradient descent methods may lead to an incorrect solution, as reducing the total loss does not necessarily imply a decrease in both the PDE residual loss and boundary loss. This phenomenon is clearly illustrated in Figure~\ref{fig:adverse_training}, which displays the curves of the total loss, PDE residual loss, and boundary loss over epochs for solving the viscous Burger's equation. Notably, while the total loss consistently decreases throughout the training, the PDE loss adversely increases.

\paragraph{Conflicting and dominating gradients in PINNs.} This issue is highly related with discrepancies in the direction and magnitude between two gradients of the PDE residual and boundary loss. Specifically, we define two gradients to be \textit{conflicting} at the $t$-th iteration if they have a negative inner product value, i.e., $\frac{\pi}{2}<\phi_t\leq \pi$ where $\phi_t$ is the angle between $\nabla \cL_r(\theta_t)$ and $\nabla \cL_b(\theta_t)$. When there are conflicting gradients, parameter updates to minimize one loss function might increase the other, leading to an inefficient learning process such as oscillating between optimizing for the two loss functions and resulting in degraded solution quality \citep{yu2022gradient}. Another problem arises when one gradient is much larger than the other, i.e., $\|\nabla \cL_r (\theta_t)\| \ll \|\nabla \cL_b (\theta_t)\|$ or $\|\nabla \cL_r (\theta_t)\| \gg \|\nabla \cL_b (\theta_t)\|$. The significant differences\footnote{Our subsequent analysis in Section~\ref{subsec:dual_cone} will clearly identify the extent to which significant differences in gradient magnitude lead to a challenge.} in the magnitudes of gradients in PINNs might create a situation where the optimization algorithm primarily minimizes one loss function while neglecting the other. This often results in slow convergence and overshooting, as the smaller gradient, though neglected, may be more crucial in finding a better solution. To mitigate the imbalance in the gradients, loss balancing approaches to rescale the weights of each loss term have been proposed \citep{wang2021understanding, wang2022NTK}.

\begin{figure}%{r}{.45\textwidth}
%\vspace{-10pt}
        \centering 
    \begin{subfigure}{0.25\textwidth}
        \includegraphics[width=\textwidth]{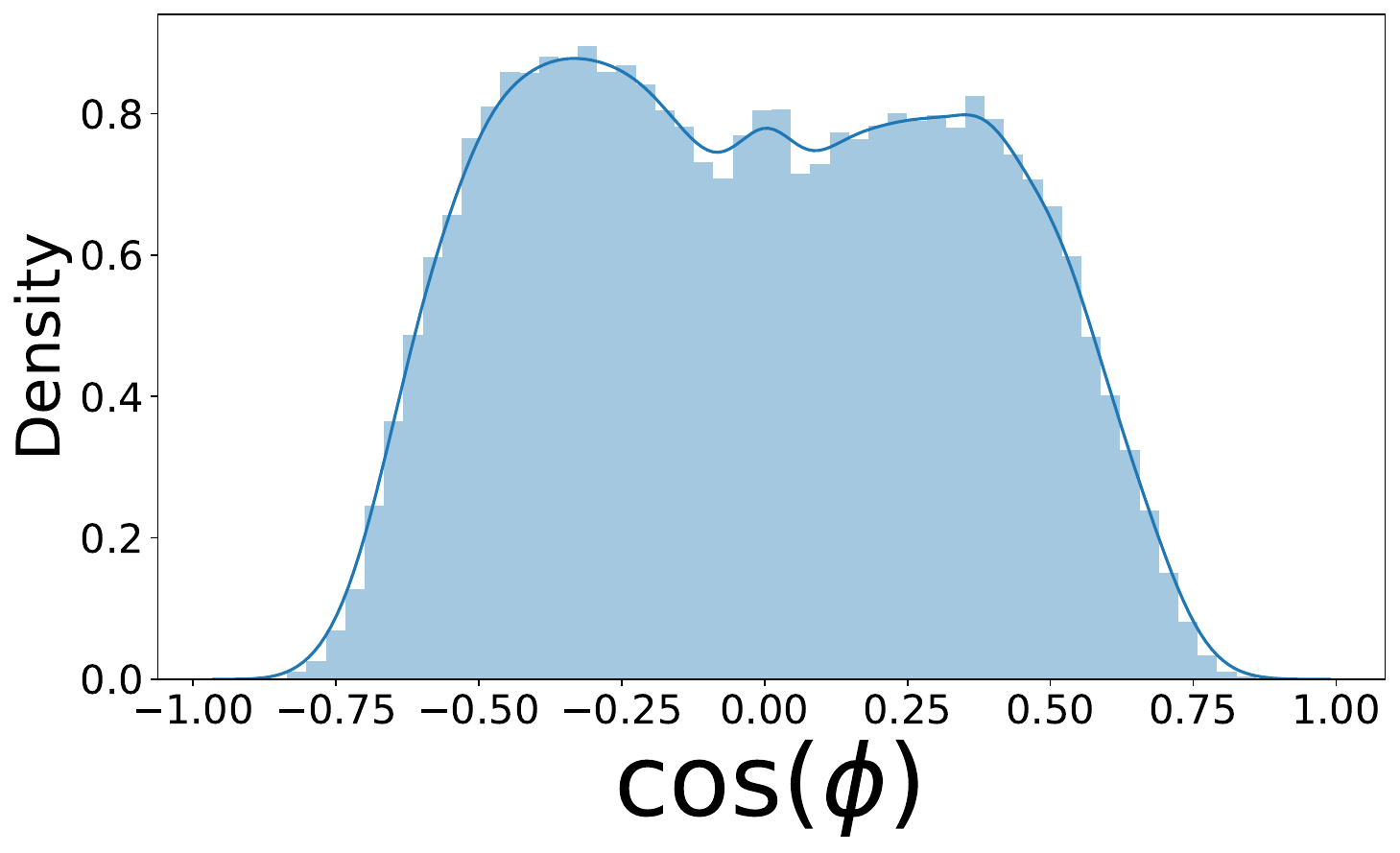}
        \caption{histogram of \(\cos (\phi)\)}
        \label{fig:conf_dom_grads(a)}
    \end{subfigure}
    \hspace{0.1\textwidth}
    \begin{subfigure}{0.25\textwidth}
        \includegraphics[width=\textwidth]{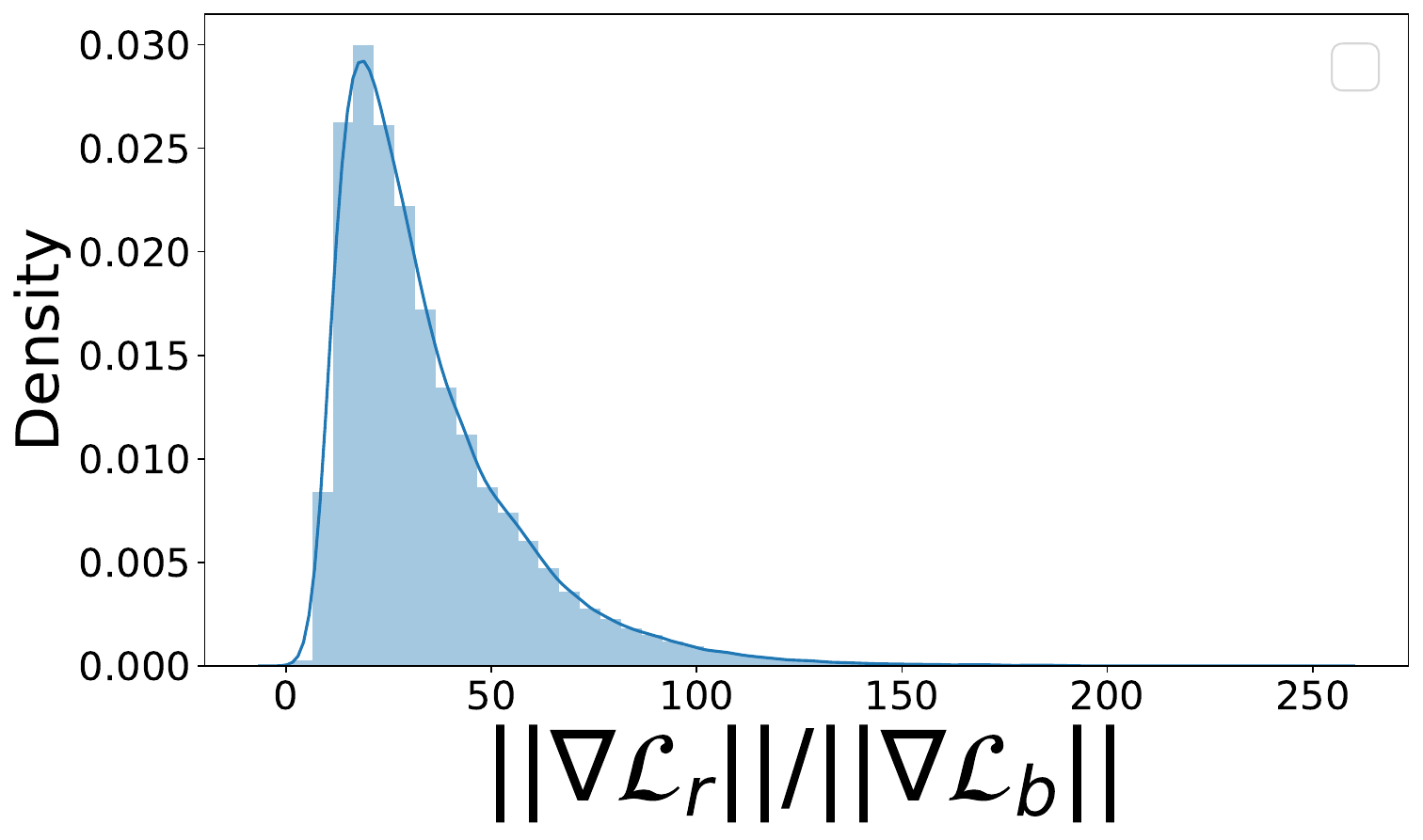}
        \caption{histogram of \(R\)}
        \label{fig:conf_dom_grads(b)}
    \end{subfigure}
    \caption{Conflicting and dominating gradients in PINNs. Here, $\phi$ is defined as the angle between \(\nabla \cL_r\) and \(\nabla \cL_b\), \(R = \frac{\|\nabla \cL_r\|}{ \|\nabla \cL_b\|} \) is the magnitude ratio between gradients.}
    \label{fig:conf_dom_grads}
\vspace{-10pt}    
\end{figure}
% \begin{wrapfigure}{r}{.45\textwidth}
% \vspace{-10pt}
%         \centering 
%     \begin{subfigure}[b]{0.22\textwidth}
%         \includegraphics[width=\textwidth]{figures/gradient_conflict_dist.pdf}
%         \caption{histogram of \(\cos (\phi)\)}
%         \label{fig:conf_dom_grads(a)}
%     \end{subfigure}
%     \begin{subfigure}[b]{0.22\textwidth}
%         \includegraphics[width=\textwidth]{figures/magnitude_imbalance_dist.pdf}
%         \caption{histogram of \(R\)}
%         \label{fig:conf_dom_grads(b)}
%     \end{subfigure}
%     \caption{Conflicting and dominating gradients in PINNs. Here, $\phi$ is defined as the angle between \(\nabla \cL_r\) and \(\nabla \cL_b\), \(R = \frac{\|\nabla \cL_r\|}{ \|\nabla \cL_b\|} \) is the magnitude ratio between gradients.}
%     \label{fig:conf_dom_grads}
% \vspace{-10pt}    
% \end{wrapfigure}
To examine these challenges in training PINNs, we record cosine value of the angle between $\nabla \cL_r$ and $\nabla \cL_b$, and the ratios of their magnitudes while training a PINN for the Helmholtz equation. Figure~\ref{fig:conf_dom_grads}(a) shows that conflicting gradients are observed in about half of the total iterations. Moreover, we observe that the magnitude of the gradient of the PDE residual is several tens to hundreds of times larger than that of the boundary loss (See Figure~\ref{fig:conf_dom_grads}(b)). That is, conflicting and dominating gradients are prevalent issues in the training of PINNs. 
\section{Methodology}\label{sec:methodology}
In this section, we provide a geometric analysis to identify a dual cone region where both the PDE residual loss and the boundary loss can decrease simultaneously. Subsequently, we introduce a general framework for DCGD algorithms, ensuring that the updated gradient falls within this region. We then propose three specifications of DCGD algorithms: projection, average, and center. All proofs for main results in this section can be found in Appendix~\ref{app:proofs}.
% building upon the motivation from the previous observations, we conduct a geometric analysis of loss gradient during training PINNs. Base on this analysis, we propose a optimization framework, Dual Cone Gradient Descent (DCGD), which consider a dual cone region as a proper update direction set to address challenges of training PINNs. Subsequently, we prove that DCGD guarantees convergence to Pareto-stationary points. Finally, we suggest three possible DCGD algorithms.
% \begin{wrapfigure}{r}{.42\textwidth}
% \vspace{-10pt}
%     \centering 
%     \begin{tabular}{cc}
%         \includegraphics[width=0.4\textwidth ]{figures/DualConeRegion.png} 
%     \end{tabular}
% \caption{Visualization of dual cone region $\rmK_t^*$ and its subspace $\rmG_t$}\label{fig:region_G} \end{wrapfigure}

\subsection{Dual Cone Region}\label{subsec:dual_cone}
The concept of a dual cone plays a pivotal role in our DCGD algorithm. Formally, a dual cone is defined as a set of vectors that have nonnegative inner product values with a given cone. 

\begin{definition}\label{def:dual_cone} (Dual cone) Let $\rmK$ be a cone of $\sR^d$. Then, the set 
$$
\rmK^* = \{ y | \ip{x}{y} \geq 0 \quad \mbox{for all } x\in \rmK, y\in \sR^d\}
$$
is called the \textit{dual cone} of $\rmK$. 
\end{definition}
For each iteration $t$, consider a cone denoted by $\rmK_t$, which is generated by rays of two gradients, $\nabla \cL_r(\theta_t)$ and $\nabla \cL_b(\theta_t)$: 
$$
\rmK_t:=\left\{c x | c\geq 0, x \in \{\nabla \cL_r(\theta_t), \nabla \cL_b(\theta_t)\}\right\}.
$$
In the context of PINNs, the dual cone of $\rmK_t$, denoted by $\rmK_t^*$, represents the set of gradient vectors where each vector is neither conflicting with the gradient of the PDE loss nor with the gradient of the boundary loss, i.e., for $u\in \rmK_t^*$, $\ip{u}{\nabla \cL_r(\theta_t)} \geq 0$ and $\ip{u}{\nabla \cL_b(\theta_t)} \geq 0$.

%\begin{figure}[htb!]
%    \centering 
%    \begin{subfigure}[b]{0.45\textwidth}
%        \includegraphics[width=0.9\textwidth]{figures/loss_DualConeRegion.pdf}
%        \subcaption{Training curves for viscous Burgers' equation.}
%        \label{fig:adverse_training}
%    \end{subfigure}    
%    \begin{subfigure}[b]{0.45\textwidth}
%        \includegraphics[width=0.9\textwidth]{figures/DualConeRegion.png}
%        \subcaption{Visualization of dual cone region}
%        \label{fig:region_G}
%    \end{subfigure}
%    \caption{  (a) Training curves  for the total loss $\cL$ $(:=\cL_r+\cL_b)$, PDE residual loss $\cL_r$, and boundary loss $\cL_b$ for viscous Burgers' equation. The PINN is trained using SGD. When $\cL \notin \rmK^*$, one of $\cL_r$ and $\cL_b$ may increase adversely even if $\cL$ continuously decreases. (b) Visualization of dual cone region $\rmK_t^*$ and its subspace $\rmG_t$}
%    \label{fig:DualconeRegion}
%\end{figure}

%\begin{figure}[htb!]
%    \centering 
%    \begin{tabular}{cc}
%        \includegraphics[width=0.5\textwidth ]{figures/DualConeRegion.png} 
%    \end{tabular}
%\caption{Visualization of dual cone region $\rmK_t^*$ and its subspace $\rmG_t$}\label{fig:region_G}
%\end{figure}

In other words, when the total gradient $\nabla \cL(\theta_t)$ is in $\rmK_t^*$ (as depicted by the region of the red line in Figure~\ref{fig:adverse_training}), the standard gradient descent taking the direction $\nabla \cL(\theta_t)$ will decrease both the PDE and boundary losses for a suitable step size. On the other hand, if $\nabla \cL(\theta_t) \notin \rmK_t^*$ (the region indicated by the blue line in Figure~\ref{fig:adverse_training}), one of the two losses will adversely increase even with sufficiently small step sizes.

This indicates that the training process of PINNs can significantly vary  depending on whether the total gradient belongs to the dual cone region. The following theorem establishes the necessary and sufficient conditions under which the total gradient falls within the dual cone region in terms of the angle and relative magnitude between the gradients of the PDE residual and boundary loss.

\begin{theorem}\label{thm:dual_cone_region}
Suppose that $\nabla \cL_r(\theta_t)$ and $\nabla \cL_b(\theta_t)$ are given at each iteration $t$.  Let $\phi_t$ be the angle between $\nabla \cL_r(\theta_t)$ and $\nabla \cL_b(\theta_t)$, and $R = \frac{\|\nabla \cL_r(\theta_t)\|}{\|\nabla \cL_b(\theta_t)\|}$ be their relative magnitude. Then, $\nabla \cL(\theta_t)  \in \rmK_t^*$ if and only if
\begin{align*}
    &(i) \; \ip{\nabla \cL_b(\theta_t)}{\nabla \cL_r(\theta_t)} \geq 0 \; \text{, or} \\    
    &(ii) \; \ip{\nabla \cL_b(\theta_t)}{\nabla \cL_r(\theta_t)} < 0 \; \text{and} \; -\cos\phi_t \leq R \leq -\frac{1}{\cos\phi_t}.
\end{align*}
\end{theorem}

Theorem~\ref{thm:dual_cone_region} provides a clear criterion for when conflicting and dominating gradients lead to adverse training in PINNs. For instance, the condition (ii) in Theorem~\ref{thm:dual_cone_region} implies that the larger $\phi_t$ (the more conflicting they are), even a slight difference in their magnitudes can result in adverse training. In particular, Theorem~\ref{thm:dual_cone_region} quantifies the extent of problematic relative magnitude between the two gradients, thereby clarifying the concept of dominating gradients, which has not been previously defined in the literature. 

\begin{figure}[htb!]
   \centering 
   \begin{tabular}{cc}
       \includegraphics[width=0.35\textwidth ]{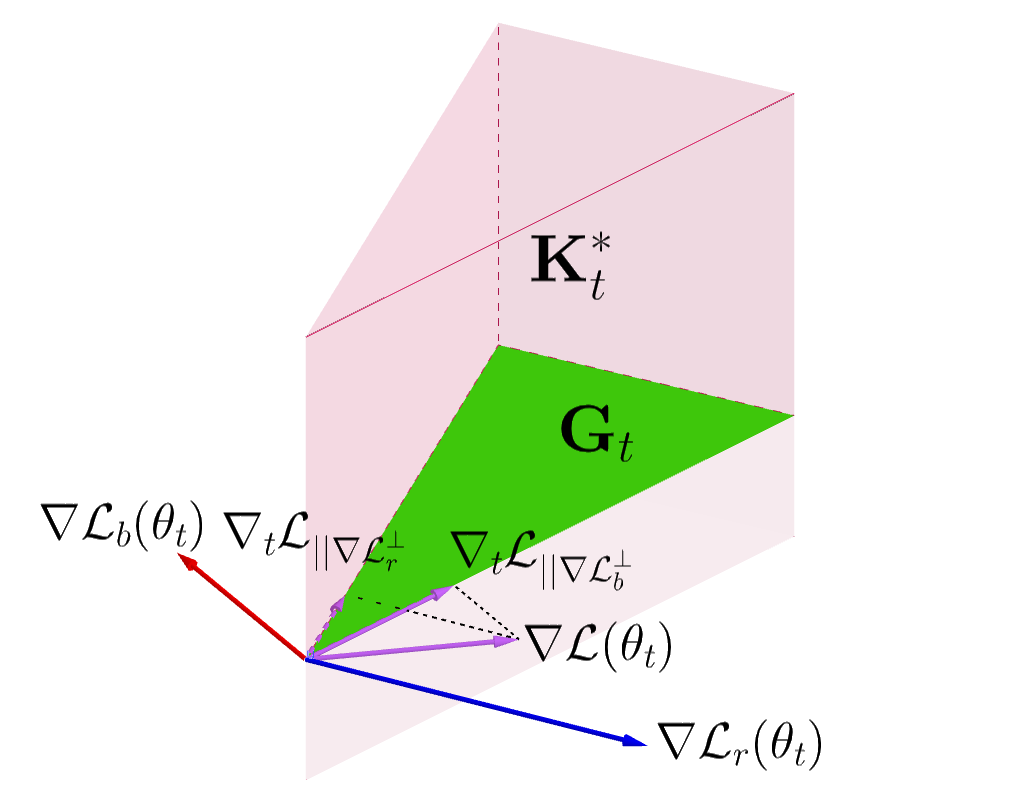} 
   \end{tabular}
\caption{Visualization of dual cone region $\rmK_t^*$ and its subspace $\rmG_t$}\label{fig:region_G}
\end{figure}

Thus, our strategy aims to devise an algorithm that chooses the updated gradient within the dual cone region at each gradient descent step. For notational simplicity, we write $\nabla_t \cL_{\|\nabla \cL_r^\perp}$ and $\nabla_t \cL_{\|\nabla \cL_b^\perp}$ to represent $\nabla \cL(\theta_t)_{\|\left(\nabla \cL_r(\theta_t)\right)^\perp}$ and $\nabla \cL(\theta_t)_{\|\left(\nabla \cL_b(\theta_t)\right)^\perp}$, respectively. In particular, we are interested in a simple and explicit subspace $\rmG_t$, defined as the set of conic combinations of $\nabla_t \cL_{\|\nabla \cL_r^\perp}$ and $\nabla_t \cL_{\|\nabla \cL_b^\perp}$:
\begin{align}\label{eq:G}
    \rmG_t := \left\{c_1 \nabla_t \cL_{\|\nabla \cL_r^\perp} + c_2 \nabla_t \cL_{\|\nabla \cL_b^\perp} \big| c_1, c_2\geq 0 \right\},
\end{align}
for two reasons. Firstly, all vectors in $\rmG_t$ are easily computable due to the explicit expression of $\rmG_t$, whereas the dual cone $\rmK^*$ is implicitly defined. Secondly, $\rmG_t$ contains two important components of $\rmK_t^*$, which are the projections of $\nabla \cL(\theta_t)$ onto $\nabla \cL_r(\theta_t)^\perp$ and $\nabla \cL_b(\theta_t)^\perp$ by its construction. The next proposition shows that $\rmG_t$ always belongs to the dual cone region as illustrated in Figure~\ref{fig:region_G}.

\begin{proposition}\label{prop:dual_cone_G}
Suppose that $\nabla \cL_r(\theta_t)$ and $\nabla \cL_b(\theta_t)$ are given at each iteration $t$. Consider $\rmG_t$, the set of conic combinations of  $\nabla_t \cL_{\|\nabla \cL_r^\perp}$ and $\nabla_t \cL_{\|\nabla \cL_b^\perp}$, defined in \Eqref{eq:G}. Then, $\rmG_t \subseteq \rmK_t^{*}$. 
\end{proposition}

Consequently, the DCGD algorithm defines the updated gradient denoted by $g_t^{\text{dual}}$ within $\rmG_t$ at each iteration $t$. A general framework for DCGD is presented in Algo~\ref{alg:DCGD_base}.

%\begin{figure}[htb!]
%    \centering 
%    \begin{tabular}{cc}
%        \includegraphics[width=0.5\textwidth ]{figures/DualConeRegion.png} 
%    \end{tabular}
%\caption{Visualization of dual cone region $\rmK_t^*$ and its subspace $\rmG_t$}\label{fig:region_G}
%\end{figure}

\begin{algorithm}[htb!]
   \caption{Dual Cone Gradient Descent (base)}
   \label{alg:DCGD_base}
\begin{algorithmic}
   \STATE {\bfseries Require:} learning rate $\lambda$, max epoch $T$, initial point $\theta_0$
   \FOR{$t=1$ {\bfseries to} $T$}
   \STATE Choose $g_t^{\text{dual}} \in \rmG_t^*$
   \STATE $\theta_t = \theta_{t-1} - \lambda g_t^{\text{dual}}$   
   \ENDFOR
\end{algorithmic}
\end{algorithm}

\subsection{Convergence Analysis}\label{subsec:convergence}

To discuss the convergence properties of DCGD, we introduce the concept of Pareto optimality (adapted to the PINN setting), which is a key in multi-objective optimization~\citep{hochman69pareto, desideri2012mgda}.

\begin{definition}(Pareto optimal and stationary)\label{def:pareto} A point $\theta \in \sR^d$ is said to be \textit{Pareto-optimal} if there does not exist $\theta'\in \sR^d$ such that 
\begin{align*}
  \cL_r(\theta') \leq \cL_r(\theta) \quad \mbox{and} \quad \cL_b(\theta') \leq \cL_b(\theta).
\end{align*}
In addition, a point $\theta \in \sR^d$ is said to be \textit{Pareto-stationary} if there exists $\alpha_1, \alpha_2$ such that 
$$
\alpha_1 \nabla \cL_r(\theta) + \alpha_2 \nabla \cL_b(\theta) = 0, \quad \alpha_1,\alpha_2\geq 0, \quad \alpha_1+\alpha_2=1.
$$    
\end{definition}

Intuitively, a Pareto-stationary point implies there is no feasible descent direction that would decrease all loss functions simultaneously. For example, consider a point $\theta_t$ at which the cosine of the angle $\phi_t$ between $\nabla \cL_r(\theta_t)$ and $\nabla \cL_b(\theta_t)$ is $-1$, i.e., $\cos (\phi_t)=-1$. Such a point is Pareto-stationary. 

The following theorem guarantees the convergence of the DCGD algorithm proposed in Algo~\ref{alg:DCGD_base} under some regularities in a non-convex setting. Assume $\theta^*:=\inf_{\theta \in \sR^d} \cL(\theta)>-\infty$. 

\begin{theorem}\label{thm:convergence_nonconvex}
 Assume that both loss functions, $\cL_b(\cdot)$ and $\cL_r(\cdot)$, are differentiable and the total gradient $\nabla \cL(\cdot)$ is $L$-Lipschitz continuous with $L > 0$. If $g_t^{\text{dual}}$ satisfies the following two conditions:
\begin{enumerate}
\item[(i)] $2\ip{\nabla \cL(\theta_t)}{g_t^{\text{dual}}} - \|g_t^{\text{dual}}\|^2 \geq 0$, 
\item[(ii)] There exists $M>0$ such that $\|g_t^{\text{dual}}\|\geq M\|\nabla \cL(\theta_t)\|$,
\end{enumerate}
then, for $\lambda \leq \frac{1}{2L}$, DCGD in Algo.~\ref{alg:DCGD_base} converges to a Pareto-stationary point, or converges as 
\begin{align}
    \frac{1}{T+1}\sum_{t=0}^T\|\nabla \cL(\theta_t)\|^2 \leq \frac{2\left(\cL(\theta_0)-\cL(\theta^*)\right)}{\lambda M(T+1)}.   
\end{align}
\end{theorem}

Theorem~\ref{thm:convergence_nonconvex} states that DCGD converges to either a Pareto-stationary point, characterized by $\phi_t$ such that $\cos(\phi_t)=-1$, or a stationary point at a rate of $\cO(1/\sqrt T)$ in the nonconvex setting. Unlike single-objective (nonconvex) optimization where the goal is to pursue a stationary point, in multi-objective optimization, it is ideal to find a Pareto-stationary point that balances all loss functions. Thus, DCGD offers significant theoretical and empirical advantages over popular optimization algorithms like SGD and ADAM, which are only guaranteed to converge to a stationary point. The convergence of DCGD to a Pareto-stationary point is empirically verified in Section~\ref{subsec:benefit}.

\subsection{Dual Cone Gradient Descent: Projection, Average, and Center}\label{subsec:algorithms}

\begin{figure*}[htb!]
    \centering 
    \begin{subfigure}[b]{0.32\textwidth}
        \includegraphics[width=\textwidth]{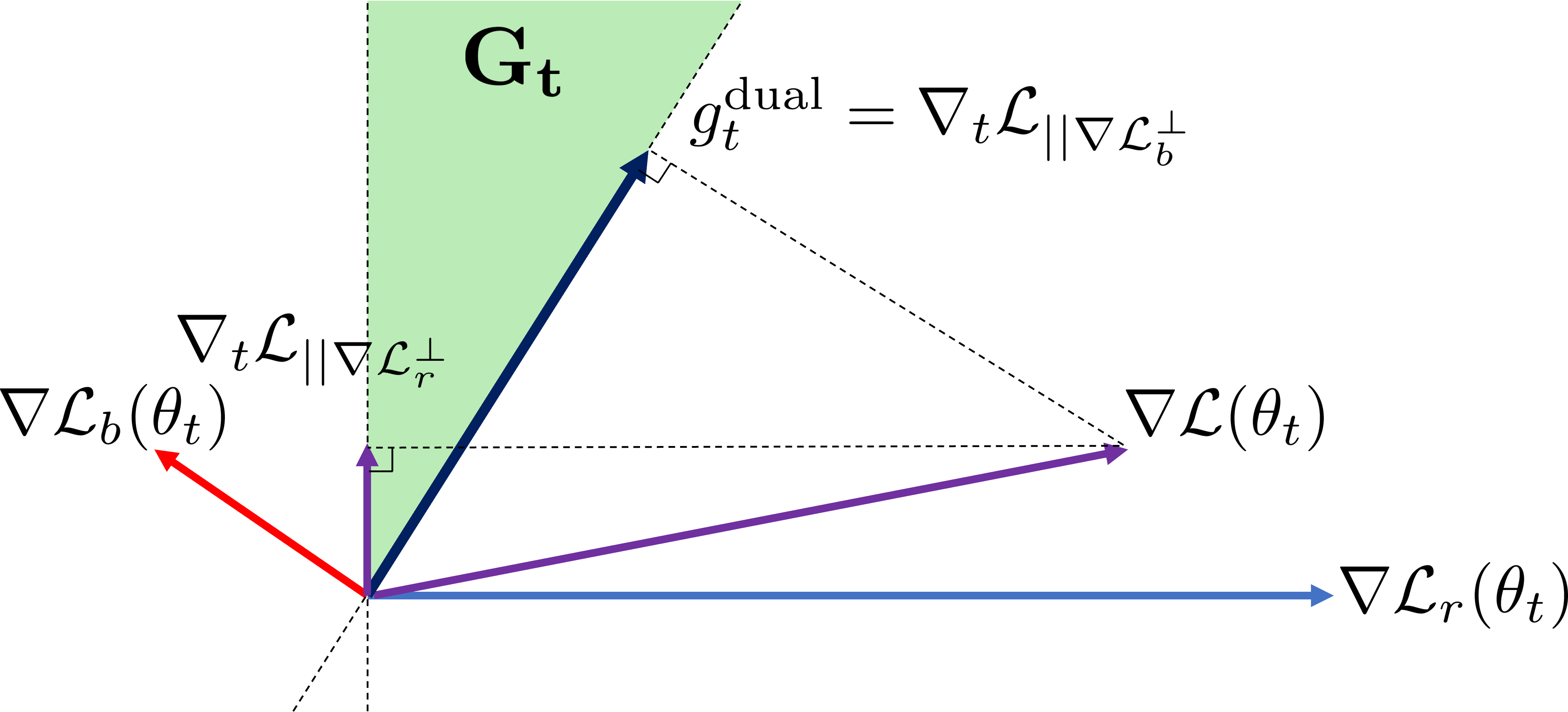}
        \caption{DCGD (Projection)}
        \label{fig:vis_proj}
    \end{subfigure}
    \begin{subfigure}[b]{0.32\textwidth}
        \includegraphics[width=\textwidth]{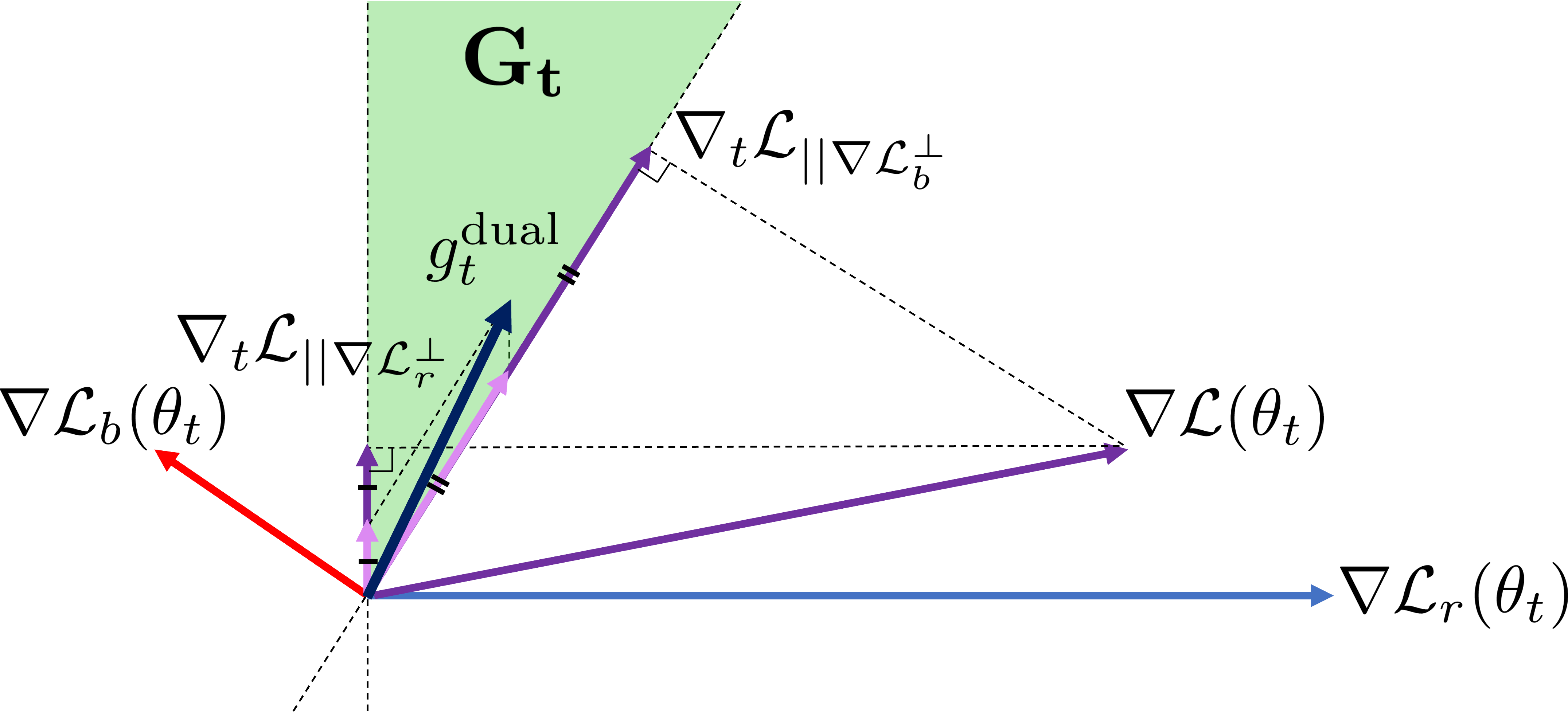}
        \caption{DCGD (Average)}
        \label{fig:vis_avg}
    \end{subfigure}
    \begin{subfigure}[b]{0.32\textwidth}
        \includegraphics[width=\textwidth]{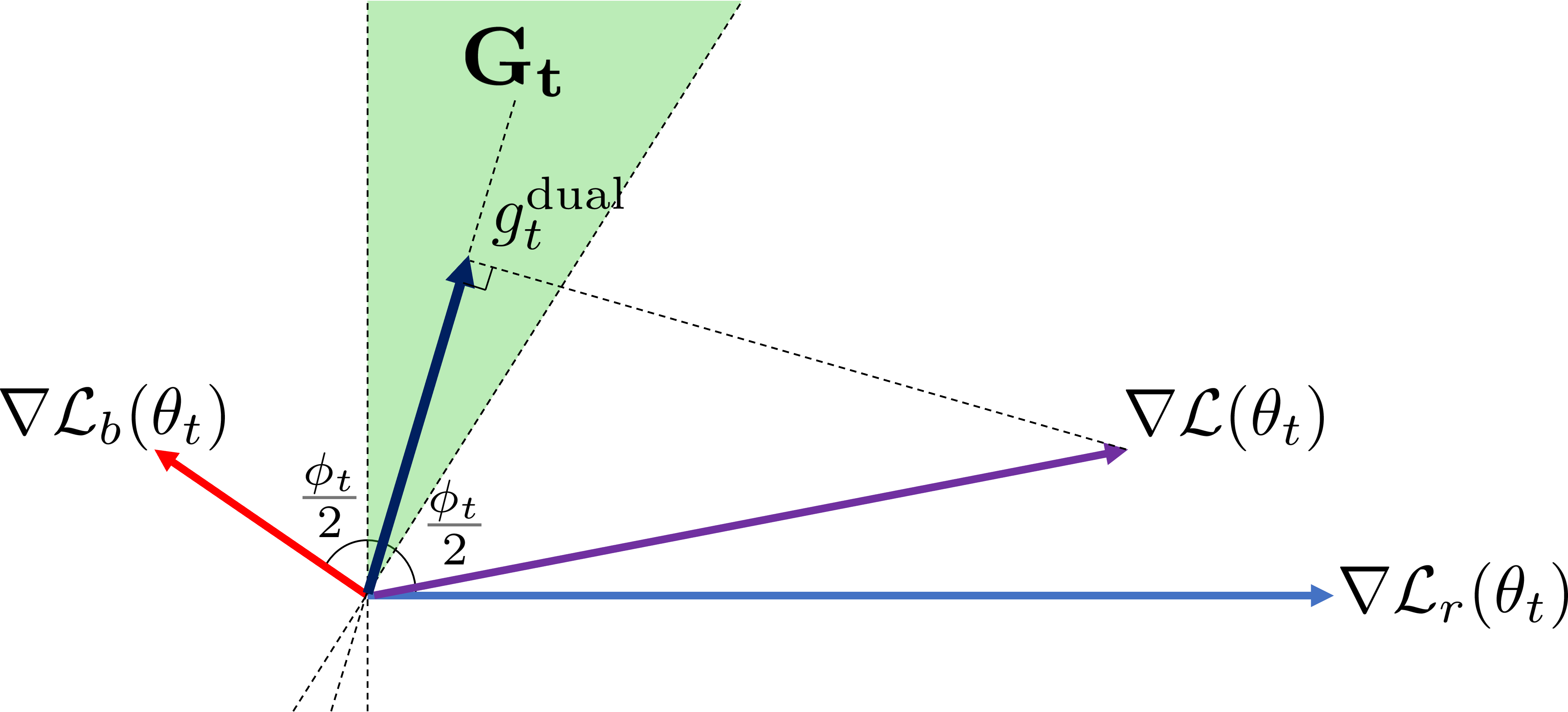}
        \caption{DCGD (Center)}
        \label{fig:vis_center}
    \end{subfigure}    
    \caption{The updated gradient \(g_t^{\text{dual}}\) of three DCGD algorithms.}
    \label{fig:dcgdalgo}
\end{figure*}

Different variants of DCGD can be designed by properly choosing the updated gradient $g_t^{\text{dual}}$ in $\rmG_t$ satisfying the conditions (i), (ii) of Theorem~\ref{thm:convergence_nonconvex}. We presents three specific algorithms: projection, average, and center. 

The first algorithm, named DCGD (Projection), uses the projection of the total gradient $\nabla \cL(\theta_t)$ onto $\rmG_t$ when $\nabla \cL(\theta_t) \notin \rmK_t^*$, which is the closest vector within $\rmG_t$ to $\nabla \cL(\theta_t)$. Specifically, the DCGD (Projection) algorithm specifies $g_t^{\text{dual}}$ as follows: \textbf{(i)} $\nabla \cL(\theta_t)$ if $\nabla \cL(\theta_t)\in \rmK_t^*$, \textbf{(ii)} $\nabla_t \cL_{\| \nabla \cL_r^\perp}$ ($c_1=1,c_2=0$) if $\nabla \cL(\theta_t)\notin \rmK_t^*$ and $\ip{\nabla \cL(\theta_t)}{\nabla \cL_r(\theta_t)}<0$, \textbf{(iii)} $\nabla_t \cL_{\| \nabla \cL_b^\perp}$ ($c_1=0, c_2=1$) if $\nabla \cL(\theta_t)\notin \rmK_t^*$ and $\ip{\nabla \cL(\theta_t)}{\nabla \cL_b(\theta_t)}<0$. See also \Eqref{eq:gt_projection} and Algo.~\ref{alg:proj}.

DCGD (Average) algorithm takes the average of $\nabla_t \cL_{\|\nabla \cL_r^{\perp}}$ and $\nabla_t \cL_{\|\nabla \cL_b^{\perp}}$ when the total gradient is outside $\rmK_t^*$, i.e., $c_1=c_2=\frac{1}{2}$ if $\nabla \cL(\theta_t) \notin \rmK_t^*$. See \Eqref{eq:gt_average} and Algo.~\ref{alg:avg}.

We note that both DCGD (Projection) and DCGD (Average) use $\nabla \cL(\theta_t)$ as $g_t^{\text{dual}}$ without any manipulation when $\nabla \cL(\theta_t)\in \rmK_t^*$. Moreover, they require determining if the total gradient is contained in the dual cone at each iteration, which may incur additional computational costs. On the other hand, $g_t^{\text{dual}}$ of DCGD (Center) is given by
\begin{align}\label{eq:gt_center}
    g_t^{\text{dual}} := \frac{\ip{g_t^c}{\nabla \cL(\theta_t)}}{\|g_t^c\|^2}g_t^c
\end{align}
where $g_t^c = \frac{\nabla \cL_b(\theta_t)}{\|\nabla \cL_b(\theta_t)\|}+\frac{\nabla \cL_r(\theta_t)}{\|\nabla \cL_r(\theta_t)\|}$, which is geometrically interpreted as the projection of $\nabla \cL(\theta_t)$ onto the angle bisector $g_t^c$ of $\nabla \cL_r(\theta_t)$ and $\nabla \cL_b(\theta_t)$. The following proposition shows that $g_t^{\text{dual}}$ of DCGD (Center) resides within $\rmG_t$

\begin{proposition}\label{prop: DCGD(center)} Consider the updated gradient $g_t^{\text{dual}}$ of DCGD (Center) defined in \Eqref{eq:gt_center}. Then, $g_t^{\text{dual}} \in \rmG_t$. 
\end{proposition}

The visualization of these three algorithms can be found in Figure~\ref{fig:dcgdalgo} and their pseudocodes are provided in Appendix~\ref{app:algorithms}. Moreover, the proposed DCGD algorithms satisfy the conditions (i) and (ii) of Theorem~\ref{thm:convergence_nonconvex}. Consequently, the following Corollary summarizes the convergence of the proposed DCGD algorithms. 
\begin{corollary}\label{cor:convergence_algorithm} We impose the same assumptions as in Theorem~\ref{thm:convergence_nonconvex}. Then, DCGD (Projection), DCGD (Average), and DCGD (Center) converge to either a Pareto-stationary point or a stationary point. 
\end{corollary}

In addition to the theoretical result in Corollary~\ref{cor:convergence_algorithm}, Appendix~\ref{app:ablation study} provides an ablation study on the empirical performance of three specific algorithms for solving benchmark PDEs.

\begin{figure}[htb!]
    \centering 
    \begin{subfigure}[b]{0.23\textwidth}
        \includegraphics[width=\textwidth]{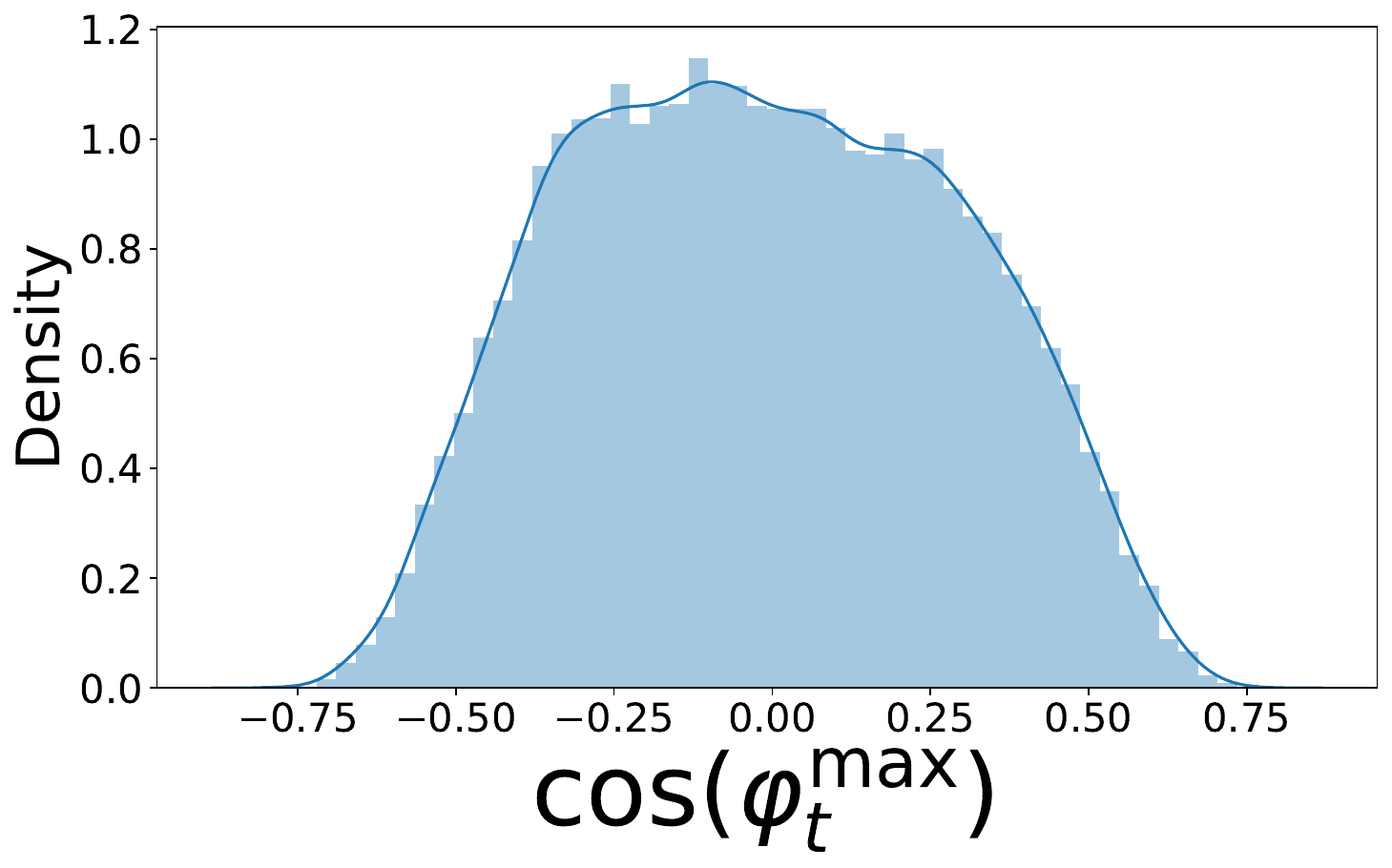}
        \caption{Adam}
        %\label{fig:subfig1}
    \end{subfigure}
    \begin{subfigure}[b]{0.23\textwidth}
        \includegraphics[width=\textwidth]{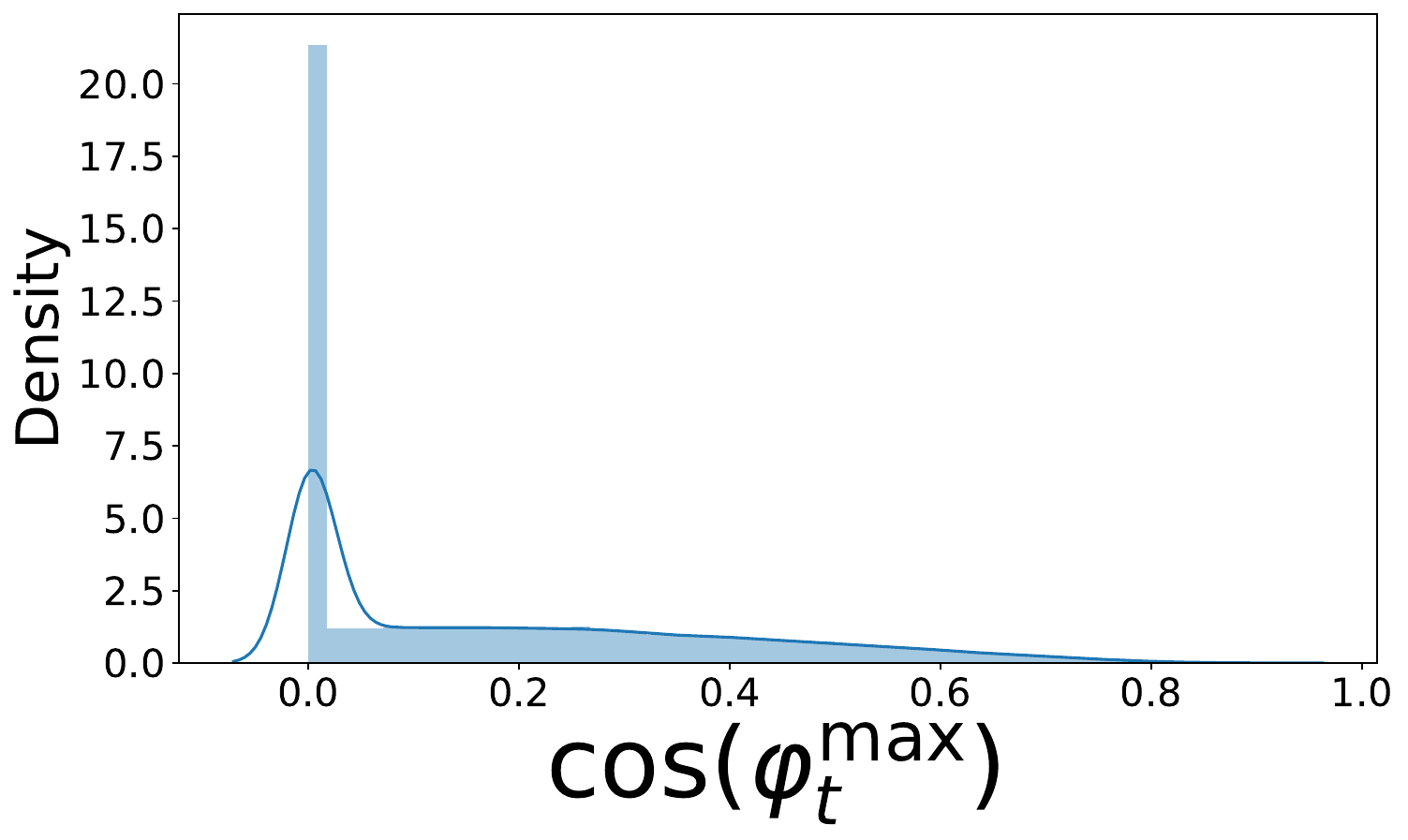}
        \caption{DCGD (Projection)}
        %\label{fig:subfig2}
    \end{subfigure}
    \begin{subfigure}[b]{0.23\textwidth}
        \includegraphics[width=\textwidth]{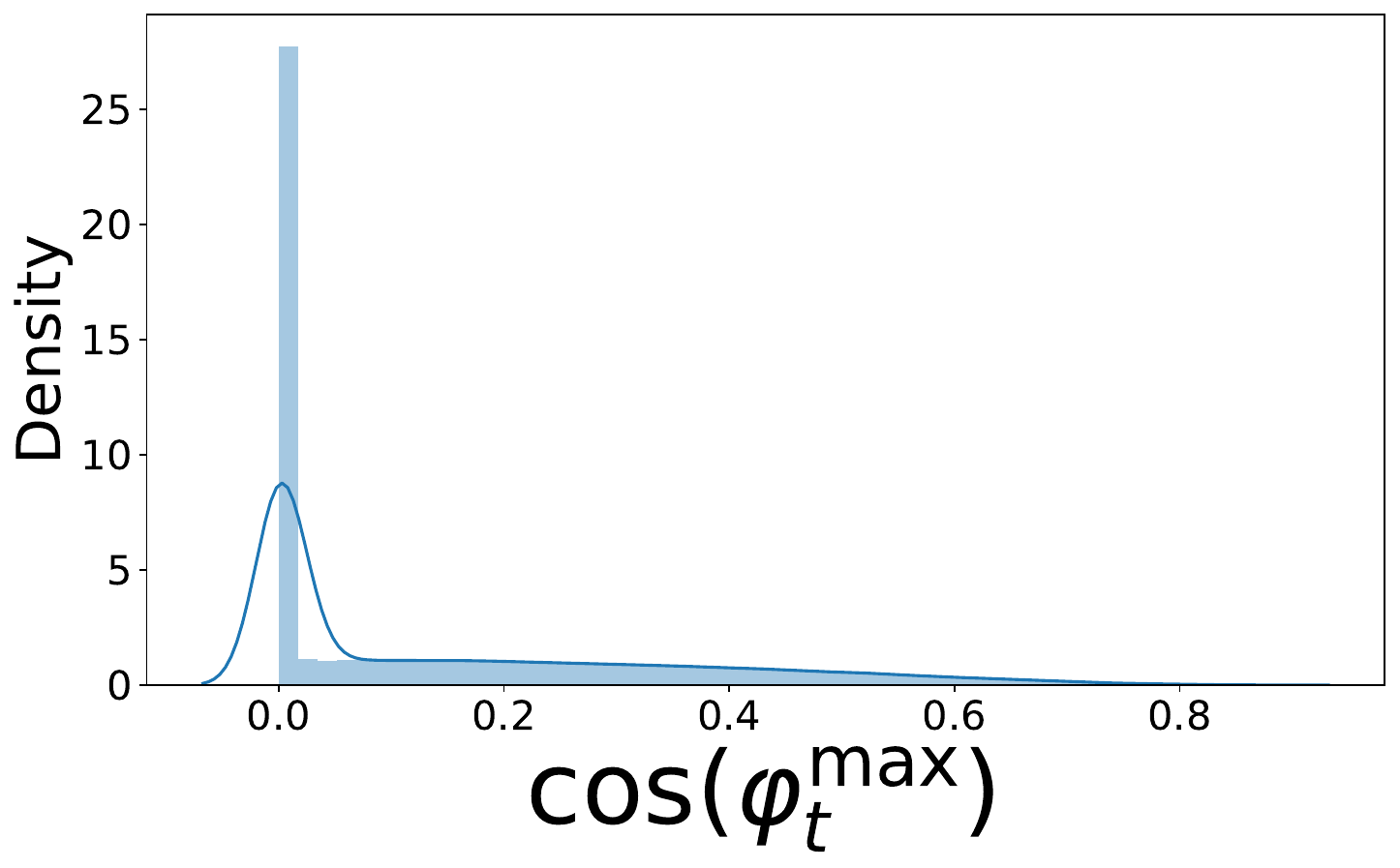}
        \caption{DCGD (Average)}
        %\label{fig:subfig3}
    \end{subfigure}
    \begin{subfigure}[b]{0.23\textwidth}
        \includegraphics[width=\textwidth]{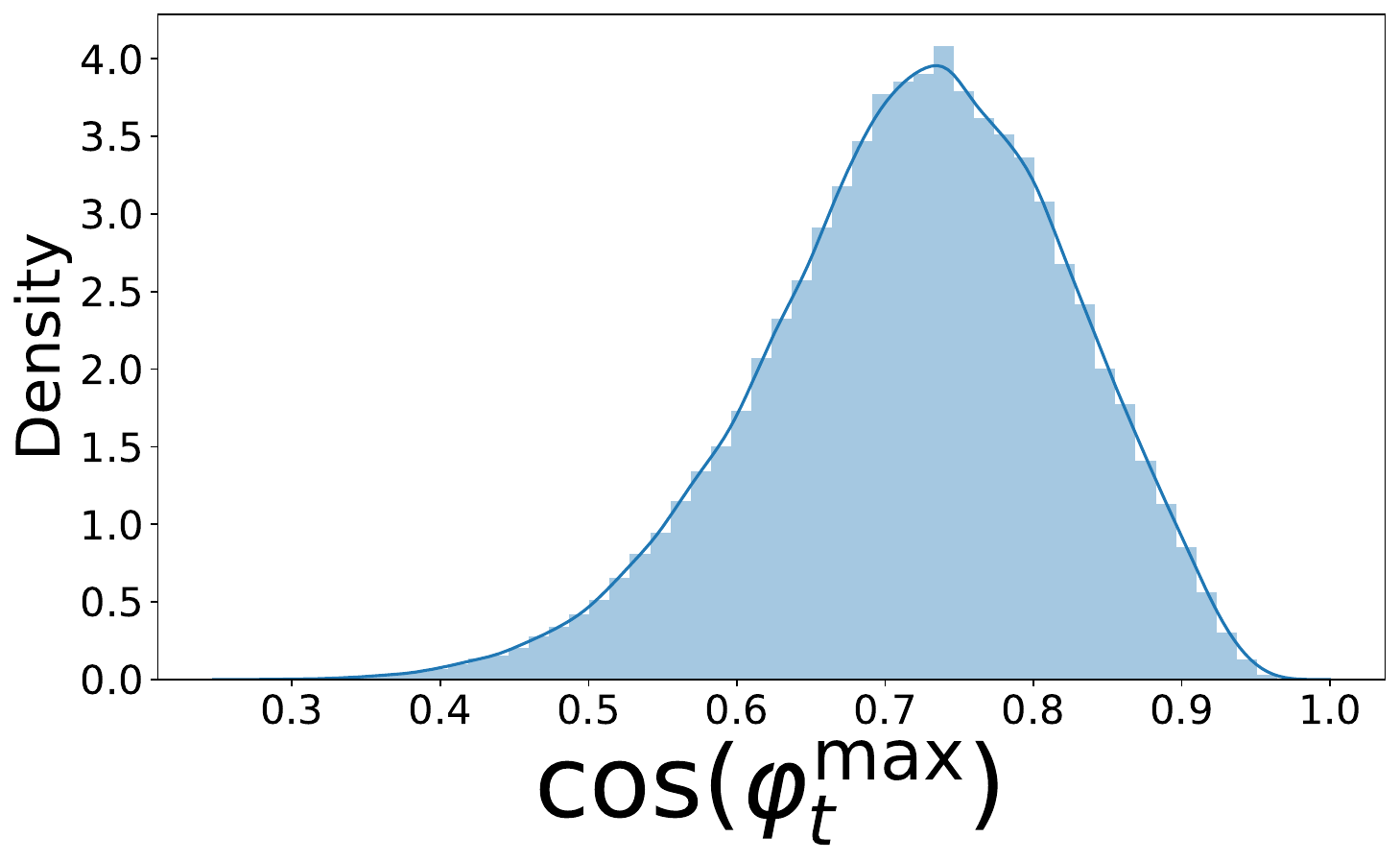}
        \caption{DCGD (Center)}
        %\label{fig:subfig4}
    \end{subfigure}
    \caption{Distribution of $\cos(\varphi_t^{\max})$ for each algorithm with $\varphi_t^{\text{max}} = \max\{\varphi_t^{r}, \varphi_t^{b}\}$ where $\varphi_t^{r}$ is the angle between the updated vector and $\nabla \cL_r(\theta_t)$, and $\varphi_t^{b}$ is the angle between the updated vector and $\nabla \cL_b(\theta_t)$.}
    \label{fig:cos_dist}
    %\vspace{-7pt}
\end{figure}

\subsection{Benefits of the DCGD framework}\label{subsec:benefit}

This subsection discusses benefits of DCGD through illustrative examples. We first investigate how the proposed DCGD algorithms resolve the conflicting gradient issue discussed in Section~\ref{sec:motivation}. Given each algorithm, at each iteration $t$, we define $\varphi_t^{r}$ as the angle between the updated vector and $\nabla \cL_r(\theta_t)$, and $\varphi_t^{b}$ as the angle between the updated vector and $\nabla \cL_b(\theta_t)$. Also, let $\varphi_t^{\text{max}} = \max\{\varphi_t^{r}, \varphi_t^{b}\}$. We highlight that both $\varphi_t^r$ and $\varphi_t^b$ are less than $\pi/2$ under DCGD algorithms, as they ensure that the updated vectors always belong to the dual cone. Figure~\ref{fig:cos_dist} plots the distributions of $\cos(\varphi_t^{\text{max}})$ for four different optimization algorithms: Adam, DCGD (Projection), DCGD (Average), and DCGD (Center) during the training of PINNs for solving the Helmholtz equation. It shows that three DCGD algorithms completely eliminate conflicting gradients in contrast to ADAM. Moreover, we observe that the distributions of $\cos(\varphi_t^{\text{max}})$ for DCGD (Projection) and DCGD (Average) are highly skewed toward zero, which implies that one of the two losses is unlikely to significantly improve. On the contrary, DCGD (Center) has a bell-shaped distribution with a mean of about 0.719, indicating that the two gradients are more aligned. This leads to a continuous reduction in both losses in a harmonious manner. Consistent with this observation, DCGD (Center) consistently outperforms DCGD (Projection) and DCGD (Average) in our experiments. Please refer to the ablation study~\ref{app:ablation study} for further comparisons.

\begin{wrapfigure}{r}{.42\textwidth}
\vspace{-10pt}
    \centering 
    \begin{subfigure}[b]{0.20\textwidth}
        \includegraphics[width=\textwidth]{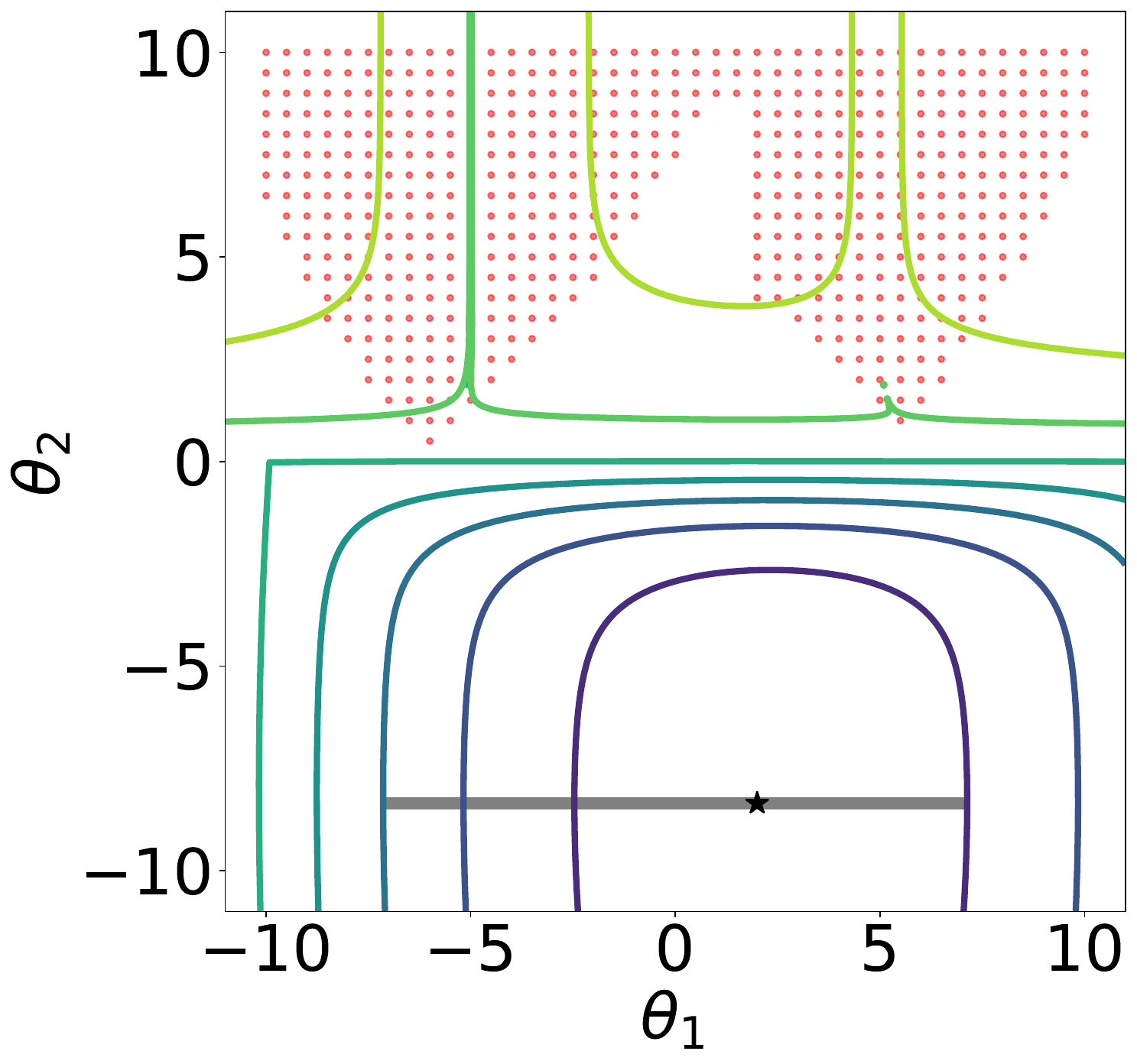}
        \caption{ADAM}
        %\label{fig:subfig1}
    \end{subfigure}
    \begin{subfigure}[b]{0.20\textwidth}
        \includegraphics[width=\textwidth]{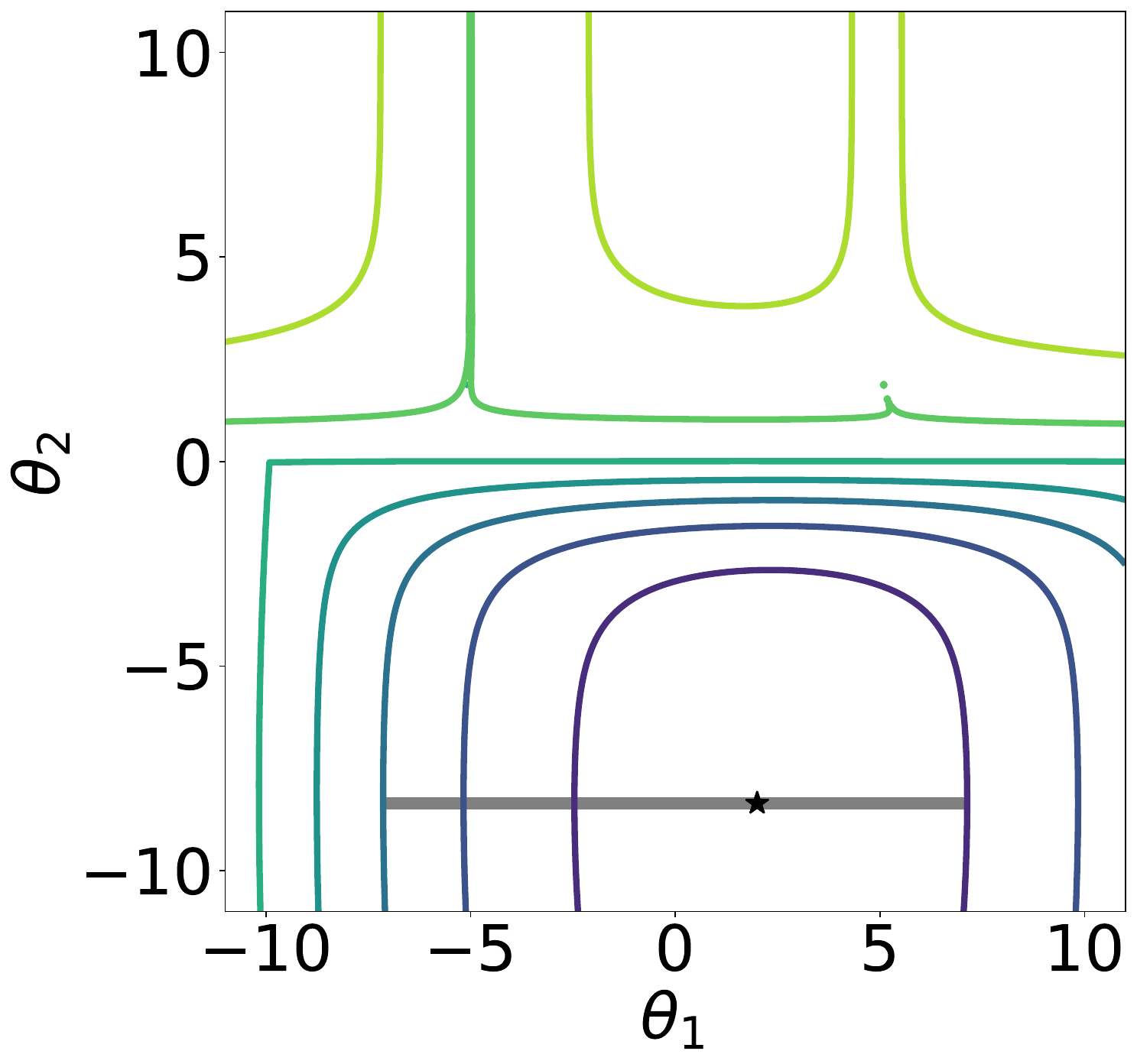}
        \caption{DCGD}
        %\label{fig:subfig2}
    \end{subfigure}
    \caption{Toy example: the region where the algorithm fails to reach a Pareto-stationary point in multi-objective optimization}
    \label{fig:toyexample}
\vspace{-10pt}
\end{wrapfigure}

%\begin{figure}[htb!]
%    \centering 
%    \begin{subfigure}[b]{0.20\textwidth}
%        \includegraphics[width=\textwidth]{figures/Toyexample/adam_convegence_region.pdf}
%        \caption{ADAM}
%        %\label{fig:subfig1}
%    \end{subfigure}
%    \begin{subfigure}[b]{0.20\textwidth}
%        \includegraphics[width=\textwidth]{figures/Toyexample/dcgd_convegence_region.pdf}
%        \caption{DCGD}
%        %\label{fig:subfig2}
%    \end{subfigure}
%    \caption{Toy example: the region where the algorithm fails to reach a Pareto-stationary point in multi-objective %optimization}
%    \label{fig:toyexample}
%\end{figure}

We empirically demonstrate that DCGD can converge to a Pareto-stationary point. Consider a (slightly modified) toy example shown in \citep{Yu2020PCGrad, Liu2021CAGrad}, which has two objective functions; see Appendix~\ref{app:toyexample} for more details. We solve the problem with 1,600 uniformly sampled initial points using Adam, DCGD (Projection), DCGD (Average), and DCGD (Center). Then, we mark with a red dot the point at which the algorithm fails to reach a Pareto-stationary point. Figure~\ref{fig:toyexample} shows that while ADAM does not reach a Pareto-stationary point across many areas, all DCGD algorithms achieve convergence to Pareto-stationary points throughout the entire space. 

Several MTL algorithms, such as PCGrad~\cite{Yu2020PCGrad}, MGDA~\cite{desideri2012mgda}, CAGrad~\cite{Liu2021CAGrad}, Aligned-MTL~\cite{senushkin2023AlignmentMTL}, and Nash-MTL~\cite{navon2022nash} have been developed based on different and independent approaches. In contrast, the proposed DCGD framework provides a principled solution to the problem of conflicting gradients by directly characterizing the dual cone. As a result, our framework unifies many of these MTL algorithms as special cases, offering significant contributions not only to PINNs but also to the MTL domain. Proofs for the unification of MTL algorithms within the DCGD framework can be found in Appendix~\ref{app:uni}.

\section{Numerical Experiment}\label{sec:experiment}
This section demonstrates the superiority of DCGD through three distinct perspectives. In Section~\ref{subsec:benchmark}, we compare the performance of DCGD on five benchmark equations with that of a range of methods, including Adam \cite{kingma2014adam}, Learning Rate Annealing (LRA) \cite{wang2021understanding}, Neural Tangent Kernel (NTK) \cite{wang2022NTK}, PCGrad \cite{Yu2020PCGrad}, MGDA \cite{desideri2012mgda}, CAGrad \cite{Liu2021CAGrad}, Aligned-MTL \cite{senushkin2023AlignmentMTL}, MultiAdam \cite{pmlr-v202-yao23c}, and DPM \cite{kim2021dpm}. Section~\ref{subsec:complex PDEs} shows that DCGD can provide more accurate solutions for failure modes of PINNs and complex PDEs where vanilla PINNs fail. In Section~\ref{subsec:comb_dcgd}, we explore the compatibility of DCGD with existing loss balancing schemes such as LRA and NTK. 

%In particular, DCGD achieves superior predictive accuracy and enhances the stability of training for complex and high-dimensional PDEs compared to existing optimally tuned models.

To compare the effectiveness of DCGD with other optimization algorithms, we measure the accuracy of the PINN solution trained by each optimizer using the relative $L^2$-error. Then, we run each experiment across $10$ independent trials and report the mean, standard deviation, max, and min of the best accuracy. 

%Lastly, we highlight that the proposed DCGD algorithms do not require extra hyperparameters to be tuned in contrast to DPM~\cite{kim2021dpm} and LRA~\cite{wang2021understanding}. This feature is particularly beneficial in deep learning applications. 

\subsection{Comparison on benchmark equations}\label{subsec:benchmark}
We solve three popular benchmark equations (the Helmholtz equation, the viscous Burgers' equation, and the Klein-Gordon equation) and two high-dimensional PDEs (5D-Heat equation and 3D-Helmholtz equation) using vanilla PINNs with different optimization techniques. For DCGD, we employ an adaptive gradient version of the DCGD (Center) algorithm, the DCGD (Center) combined with Adam (see Algo~\ref{alg:dcgd_adam}) by default for all experiments, provided in Appendix~\ref{app:ablation study}. For other methods, we perform careful hyperparameter tuning based on the recommendations in their papers. The PDE equations and detailed experimental setting are provided in Appendix~\ref{app:benchmark}. However, we do not report the performance of DPM because it is not only highly sensitive to hyperparameters but also exhibit poor performance, consistently observed in \cite{fesser2023understanding}. %Additionally, to validate the effectiveness of DCGD in solving high-dimensional PDEs, we also conducted experiments on the 5-dimensional Heat equation and the 3-dimensional Helmholtz equation. The  detailed experimental setting explained in refer Appendix~\ref{app:highdim}.

\begin{table*}[htb!]
\small\centering
\setlength{\tabcolsep}{4pt}
\begin{center}
\caption{\label{tab:benchmark}
Average of relative $L^2$ errors in 10 independent trials for each algorithm on three benchmark PDEs (3 independent trials for two high-dimensional PDEs). The value within the parenthesis indicates the standard deviation. `-' denotes that the optimizer failed to converge.}
\begin{tabular}{lccccc} 
\toprule
             & \multicolumn{5}{c}{PDE equation} \\ \midrule
Optimizer    & Helmholtz        & Burgers'         & Klein-Gordon      & Heat (5D)        & Helmholtz (3D)   \\ \midrule
Adam         & 0.0609 (0.0231)  & 0.0683 (0.0285)  & 0.0792 (0.0386)   & 0.0097 (0.0072)  & 0.6109 (0.2096)  \\
LRA          & \underline{0.0066 (0.0025)} & 0.0180 (0.0094)  & \underline{0.0069 (0.0037)} & 0.0052 (0.0056) & \underline{0.0831 (0.0123)}  \\
NTK          & 0.0358 (0.0107)  & 0.0224 (0.0061)  & 0.0223 (0.0151)   & 0.0027 (0.0012) & 0.4037 (0.2620)   \\
PCGrad       & 0.0109 (0.0031)  & \underline{0.0159 (0.0061)} & 0.0286 (0.0064)  & 0.0083 (0.0049) & 0.2532 (0.0476)  \\ 
MGDA         & 0.7590 (0.1180)  & 0.9780 (0.0462)  & 0.6690 (0.2790)   & -               & 0.9883 (0.0217)  \\
CAGrad       & 0.0735 (0.0390)  & 0.0321 (0.0063)  & 0.1850 (0.0301)   & 0.0043 (0.0016) & 0.5854 (0.3032)  \\
Aligned-MTL  & 0.6570 (0.0805)  & 0.0294 (0.0129)  & 0.5571 (0.1824)   & 0.0013 (0.0004) & 0.9138 (0.0645)  \\
MultiAdam    & 0.0211 (0.0032)  & 0.0875 (0.0303)  & 0.0228 (0.0038)   & \underline{0.0009 (0.0007)}  & 0.7809 (0.0031)  \\
DCGD         & \textbf{0.0029 (0.0005)}  & \textbf{0.0124 (0.0046)} & \textbf{0.0069 (0.0027)} & \textbf{0.0008 (0.0003)} & \textbf{0.0774 (0.0250)}  \\ \midrule\midrule
DCGD+LRA     & \textcolor{red}{0.0023 (0.0007)} & \textcolor{red}{0.0104 (0.0021)} & \textcolor{red}{0.0050 (0.0013)} & 0.0012 (0.0005)  & 0.1045 (0.0485)  \\
DCGD+NTK     & 0.0057 (0.0035)  & 0.0113 (0.0040)  & 0.0055 (0.0014)   & 0.0009 (0.0004) & 0.3525 (0.2659)  \\ \bottomrule
\end{tabular}
\end{center}
\end{table*}

Table~\ref{tab:benchmark} displays the mean and standard deviation of the relative $L^2$ errors for each optimization algorithm applied to the three PDE equations. The error plots of approximated PINN solutions and other statistics of relative $L^2$ errors are summarized in Appendix~\ref{app:benchmark}. In the result tables, we highlight \textbf{the best} and \underline{the second-best} methods. While the second best methods vary across experiments, the proposed method consistently outperforms other algorithms achieving the lowest $L^2$ errors. This result underscores the robustness and adaptability of our method for solving various PDEs.

\begin{figure}[htb!]
    \centering 
    \begin{subfigure}[b]{0.3\textwidth}
        \includegraphics[width=\textwidth]{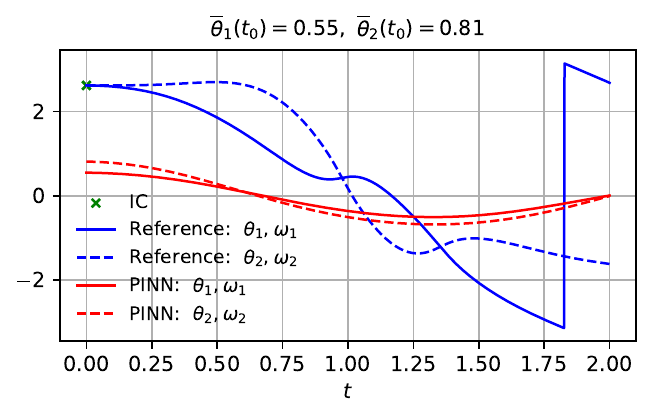}
        \subcaption{SGD}
        \label{fig:doublependulum_sgd}
    \end{subfigure}    
    \begin{subfigure}[b]{0.3\textwidth}
        \includegraphics[width=\textwidth]{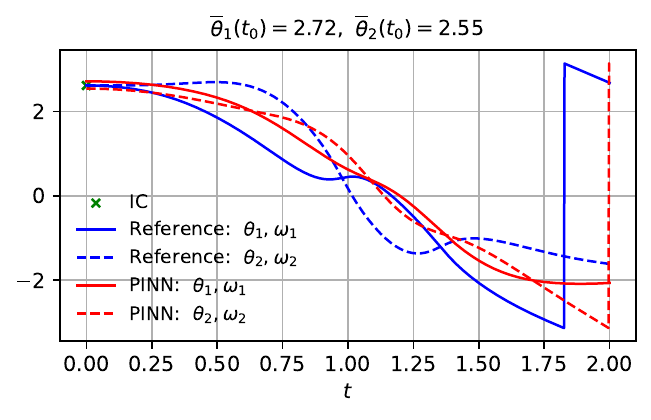}
        \subcaption{ADAM}
        \label{fig:doublependulum_adam}
    \end{subfigure}
    \begin{subfigure}[b]{0.3\textwidth}
        \includegraphics[width=\textwidth]{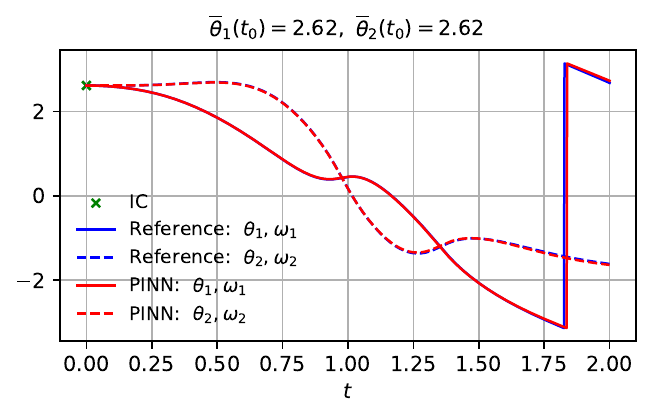}
        \subcaption{DCGD}
        \label{fig:doublependulum_dcgd}
    \end{subfigure}
    \caption{Double pendulum problem: prediction of each method. SGD and ADAM find shifted solutions, but DCGD successfully approximates the reference solution.}
    \label{fig:sol_doublependulum}
\end{figure}

\subsection{Failure model of PINNs and Complex P(I)DEs}\label{subsec:complex PDEs}
\begin{wraptable}{r}{0.42\textwidth}
\vspace{-10pt}
\small
\centering
\caption{\label{tab:complexPDE}
Relative $L^2$ errors for DCGD (Center) on Chaotic KS equation, Convection equation and Volterra IDEs. }
\begin{tabular}{lcc} \toprule
Equation        & Baseline    & DCGD              \\ \midrule
Chaotic KS      & 0.0687      & \textbf{0.0376}   \\  
Convection      & 0.4880      & \textbf{0.0246}   \\  
Volterra IDEs   & 0.0068      & \textbf{0.0011}   \\
\bottomrule
\end{tabular}
\vspace{-10pt}
\end{wraptable}

We explore more challenging problems, including failure modes of PINNs and complex PDEs, where vanilla PINNs fail to approximate solutions, and highlight the universal applicability of DCGD. We refer to Appendix~\ref{app:complex PDEs} for detailed experimental settings.

First, we revisit the problem of a double pendulum in \citet{steger2022how}, which is highly sensitive to initial conditions. The goal is to solve the trajectory of $\{(\theta_1(t), \theta_2(t))\}_{t\geq t_0}$, governed by the nonlinear differential equation as discussed in \Eqref{eq:double_pendulum}. The reference solution and its first-order derivative are represented by the blue solid and dotted lines, respectively, in Figure~\ref{fig:sol_doublependulum}. We train PINNs with SGD and ADAM to solve the double pendulum problem, where their solutions are depicted by the red solid and dotted lines in Figure~\ref{fig:doublependulum_sgd} and Figure~\ref{fig:doublependulum_adam}, respectively. The PINN solutions trained with SGD and ADAM fail to accurately approximate the reference solution. In contrast, the reference solution is successfully recovered by our DCGD algorithm (see Figure~\ref{fig:doublependulum_dcgd}). Second, we present the performance of DCGD for two challenging PDEs: the chaotic Kuramoto-Sivashinsky (KS) equation and the convection equation. For the chaotic KS equation, we combine DCGD with the causal training scheme of \cite{wang2022respecting}, the current state-of-the art result. For the convection equation, DCGD is applied to PINNsFormer of \cite{zhao2023pinnsformer}. As shown in Table~\ref{tab:complexPDE}, DCGD achieves the lowest relative $L^2$ errors for the complex PDEs compared to the existing optimally tuned strategies, demonstrating its effectiveness in overcoming failure modes of PINNs. Third, the universal applicability of DCGD is not limited to specific architectures, sampling techniques, and training schemes. For example, A-PINN, designed for solving integral equations and integro-differential equations, achieves state-of-the art results in nonlinear Volterra IDEs \cite{yuan2022pinn}. DCGD significantly improves the performance of A-PINN for solving Volterra IDEs, as shown in Table~\ref{tab:complexPDE}. Moreover, Table~\ref{tab:SPINN} shows that the performance of SPINN can be highly improved by applying DCGD for solving multi-dimensional PDEs.

\subsection{Compatibility of DCGD with existing methods}\label{subsec:comb_dcgd}
The proposed DCGD framework can be easily combined with existing PINN training strategies, including loss balancing methods. To illustrate this advantage, we have designed DCGD algorithms that integrate with LRA and NTK, named DCGD (Center) + LRA and DCGD (Center) + NTK, respectively. Please refer to Algo.~\ref{alg:DCGD_lossbalancing} for the detailed implementation.

We apply DCGD (Center) + LRA and DCGD (Center) + NTK to the same experiments described in Section~\ref{subsec:benchmark}. Tables~\ref{tab:benchmark} and \ref{tab:MaxMinBenchmarks} demonstrate that the performance of DCGD algorithms can be further enhanced across all the experiments in terms of the mean, maximum, and minimum of relative $L^2$ errors by integrating existing ideas from the literature.

%\subsection{Universal applicability of DCGD to variants of PINNs}\label{subsec:pinn_variants}

%Several variants of PINNs have been proposed to overcome the limitations and enhance the accuracy of vanilla PINNs through advanced architectures. Our method can be applied to variants of PINNs without additional effort. 

%A-PINN, designed for solving integral equations and integro-differential equations, achieves the state-of-the-art results in nonlinear Volterra IDEs \cite{yuan2022pinn}. SPINN has been recently proposed as a novel architecture to efficiently solve multi-dimensional PDEs by reducing the number of network propagations \cite{Cho2023Separable}. 

%To test the applicability of DCGD to variants of PINNs, we replicate the experiments, shown in \cite{yuan2022pinn} and \cite{Cho2023Separable}. The benchmark equations and experiment settings can be found in Appendix~\ref{app:ide}, \ref{app:spinn}. The experimental results in Table~\ref{tab:PINNs variants} show that DCGD (Center) achieves the state-of-the-art results for solving the nonlinear Volterra IDE and multi-dimensional PDE, suggesting that the performance of variants of PINNs can be improved when trained with our DCGD algorithm.

\section{Conclusion and Discussion}\label{sec:conclusion}
In this work, we provided a clear criterion for when PINNs might be adversely trained, in terms of the angle and relative magnitude ratio of the gradients of the PDE residual and boundary loss, through a geometric analysis. Based on this theoretical insight, we characterized a dual cone region where both losses can decrease simultaneously without gradient pathologies. We then proposed a general framework for DCGD, which ensures that the updated gradient falls within the dual cone region, and provided a convergence analysis. Within this general framework, we introduced three specific DCGD algorithms and conduct extensive empirical experiments. Our experimental results demonstrate that the proposed DCGD algorithms outperform other optimization algorithms. In particular, DCGD is efficient in solving challenging problems such as failure modes of PINNs and complex PDEs compared to the current state-of-the art approaches. Furthermore, DCGD can be easily combined with other strategies and applied to variants of PINNs.

Although we have presented a novel optimization algorithm, DCGD, to address challenging issues in PINNs, there still remain some interesting and important questions. For instance, one could design a more powerful DCGD specification within the dual cone region that goes beyond the projection, average, and center techniques. Also, while we mainly consider multi-objective optimization for PINNs, future work can focus on more general and complex types of multi-task learning problems.

\subsubsection*{Acknowledgement}
This work was supported by the National Research Foundation of Korea (NRF) grant funded by the Korea government (MSIT) (No.RS-2023-00253002), the Institute of Information \& communications Technology Planning \& Evaluation (IITP) grant funded by the Korea government (MSIT) (No.2020-0-01336, Artificial Intelligence Graduate School Program (UNIST)), and Startup Research Fund (1.220132.01) of UNIST (Ulsan National Institute of Science \& Technology).

{\small
\bibliographystyle{unsrtnat}
\bibliography{reference}

\begin{thebibliography}{47}
\providecommand{\natexlab}[1]{#1}
\providecommand{\url}[1]{\texttt{#1}}
\expandafter\ifx\csname urlstyle\endcsname\relax
  \providecommand{\doi}[1]{doi: #1}\else
  \providecommand{\doi}{doi: \begingroup \urlstyle{rm}\Url}\fi

\bibitem[Raissi et~al.(2019)Raissi, Perdikaris, and Karniadakis]{raissi2019physics}
Maziar Raissi, Paris Perdikaris, and George~E Karniadakis.
\newblock Physics-informed neural networks: A deep learning framework for solving forward and inverse problems involving nonlinear partial differential equations.
\newblock \emph{Journal of Computational physics}, 378:\penalty0 686--707, 2019.

\bibitem[Strelow et~al.(2023)Strelow, Gerisch, Lang, and Pfetsch]{STRELOW2023gas}
Erik~Laurin Strelow, Alf Gerisch, Jens Lang, and Marc~E. Pfetsch.
\newblock Physics informed neural networks: A case study for gas transport problems.
\newblock \emph{Journal of Computational Physics}, 481:\penalty0 112041, 2023.
\newblock ISSN 0021-9991.

\bibitem[Sharma et~al.(2023)Sharma, Chung, Akoush, and Ihme]{sharma2023review}
Pushan Sharma, Wai~Tong Chung, Bassem Akoush, and Matthias Ihme.
\newblock A review of physics-informed machine learning in fluid mechanics.
\newblock \emph{Energies}, 16\penalty0 (5):\penalty0 2343, 2023.

\bibitem[Wiecha et~al.(2021)Wiecha, Arbouet, Girard, and Muskens]{Wiecha21nano}
Peter~R. Wiecha, Arnaud Arbouet, Christian Girard, and Otto~L. Muskens.
\newblock Deep learning in nano-photonics: inverse design and beyond.
\newblock \emph{Photon. Res.}, 9\penalty0 (5):\penalty0 B182--B200, May 2021.

\bibitem[Islam et~al.(2021)Islam, Thakur, Mojumder, and Hasan]{islam2021extraction}
Mahmudul Islam, Md~Shajedul~Hoque Thakur, Satyajit Mojumder, and Mohammad~Nasim Hasan.
\newblock Extraction of material properties through multi-fidelity deep learning from molecular dynamics simulation.
\newblock \emph{Computational Materials Science}, 188:\penalty0 110187, 2021.

\bibitem[Smith et~al.(2022)Smith, Ross, Azizzadenesheli, and Muir]{smith2022hyposvi}
Jonthan~D Smith, Zachary~E Ross, Kamyar Azizzadenesheli, and Jack~B Muir.
\newblock Hypo{SVI}: Hypocentre inversion with stein variational inference and physics informed neural networks.
\newblock \emph{Geophysical Journal International}, 228\penalty0 (1):\penalty0 698--710, 2022.

\bibitem[Verma et~al.(2024)Verma, Heinonen, and Garg]{verma2024climode}
Yogesh Verma, Markus Heinonen, and Vikas Garg.
\newblock Clim{ODE}: Climate forecasting with physics-informed neural {ODE}s.
\newblock In \emph{The Twelfth International Conference on Learning Representations}, 2024.

\bibitem[Ni et~al.(2023)Ni, Feng, Ma, Ma, and Lan]{ni2023sliced}
Yuyan Ni, Shikun Feng, Wei-Ying Ma, Zhi-Ming Ma, and Yanyan Lan.
\newblock Sliced denoising: A physics-informed molecular pre-training method.
\newblock \emph{arXiv preprint arXiv:2311.02124}, 2023.

\bibitem[Yuan et~al.(2022)Yuan, Ni, Deng, and Hao]{yuan2022pinn}
Lei Yuan, Yi-Qing Ni, Xiang-Yun Deng, and Shuo Hao.
\newblock A-{PINN}: Auxiliary physics informed neural networks for forward and inverse problems of nonlinear integro-differential equations.
\newblock \emph{Journal of Computational Physics}, 462:\penalty0 111260, 2022.

\bibitem[Pang et~al.(2019)Pang, Lu, and Karniadakis]{pang2019fpinn}
Guofei Pang, Lu~Lu, and George~Em Karniadakis.
\newblock f{PINN}s: Fractional physics-informed neural networks.
\newblock \emph{SIAM Journal on Scientific Computing}, 41\penalty0 (4):\penalty0 A2603--A2626, 2019.

\bibitem[Zhang et~al.(2020)Zhang, Guo, and Karniadakis]{zhang2020spde}
Dongkun Zhang, Ling Guo, and George~Em Karniadakis.
\newblock Learning in modal space: Solving time-dependent stochastic {PDE}s using physics-informed neural networks.
\newblock \emph{SIAM Journal on Scientific Computing}, 42\penalty0 (2):\penalty0 A369--A665, 2020.

\bibitem[Kharazmi et~al.(2021)Kharazmi, Zhang, and Karniadakis]{kharazmi2021hp}
Ehsan Kharazmi, Zhongqiang Zhang, and George~Em Karniadakis.
\newblock hp-{VPINN}s: Variational physics-informed neural networks with domain decomposition.
\newblock \emph{Computer Methods in Applied Mechanics and Engineering}, 374:\penalty0 113547, 2021.

\bibitem[Jagtap and Karniadakis(2021)]{jagtap2021extended}
Ameya~D Jagtap and George~E Karniadakis.
\newblock Extended physics-informed neural networks ({XPINN}s): A generalized space-time domain decomposition based deep learning framework for nonlinear partial differential equations.
\newblock In \emph{AAAI spring symposium: MLPS}, volume~10, 2021.

\bibitem[Wu et~al.(2022)Wu, Hennigh, Kautz, Choudhry, and Byeon]{wu2022physics}
Benjamin Wu, Oliver Hennigh, Jan Kautz, Sanjay Choudhry, and Wonmin Byeon.
\newblock Physics informed rnn-dct networks for time-dependent partial differential equations.
\newblock In \emph{International Conference on Computational Science}, pages 372--379. Springer, 2022.

\bibitem[Liu et~al.(2024)Liu, Li, Kuang, Xue, Chen, Yang, Liao, and Zhang]{zhao2023pinnsformer}
Liyang Liu, Yi~Li, Zhanghui Kuang, J~Xue, Yimin Chen, Wenming Yang, Qingmin Liao, and Wayne Zhang.
\newblock {PINN}s{F}ormer: A transformer-based framework for physics-informed neural networks.
\newblock \emph{International Conference on Learning Representations}, 2024.

\bibitem[Cho et~al.(2023{\natexlab{a}})Cho, Lee, Rim, and Park]{Cho2023Hyper}
Woojin Cho, Kookjin Lee, Donsub Rim, and Noseong Park.
\newblock Hypernetwork-based meta-learning for low-rank physics-informed neural networks.
\newblock In \emph{Advances in Neural Information Processing Systems}, 2023{\natexlab{a}}.

\bibitem[Cho et~al.(2023{\natexlab{b}})Cho, Nam, Yang, Yun, Hong, and Park]{Cho2023Separable}
Junwoo Cho, Seungtae Nam, Hyunmo Yang, Seok-Bae Yun, Youngjoon Hong, and Eunbyung Park.
\newblock Separable physics-informed neural networks.
\newblock In \emph{Advances in Neural Information Processing Systems}, 2023{\natexlab{b}}.

\bibitem[Han and Lee(2023)]{han2023hierarchical}
Jihun Han and Yoonsang Lee.
\newblock Hierarchical learning to solve partial differential equations using physics-informed neural networks.
\newblock \emph{arXiv preprint arXiv:2211.08064v2}, 2023.

\bibitem[Yu et~al.(2022)Yu, Lu, Meng, and Karniadakis]{yu2022gradient}
Jeremy Yu, Lu~Lu, Xuhui Meng, and George~Em Karniadakis.
\newblock Gradient-enhanced physics-informed neural networks for forward and inverse {PDE} problems.
\newblock \emph{Computer Methods in Applied Mechanics and Engineering}, 393:\penalty0 114823, 2022.

\bibitem[Son et~al.(2023)Son, Cho, and Hwang]{son2023enhanced}
Hwijae Son, Sung~Woong Cho, and Hyung~Ju Hwang.
\newblock Enhanced physics-informed neural networks with augmented lagrangian relaxation method ({AL}-{PINN}s).
\newblock \emph{Neurocomputing}, page 126424, 2023.

\bibitem[Wang et~al.(2022{\natexlab{a}})Wang, Li, He, and Wang]{NEURIPS2022_374050dc}
Chuwei Wang, Shanda Li, Di~He, and Liwei Wang.
\newblock Is \( {L}^2 \) physics informed loss always suitable for training physics informed neural network?
\newblock In \emph{Advances in Neural Information Processing Systems}, volume~35, pages 8278--8290, 2022{\natexlab{a}}.

\bibitem[Wang et~al.(2022{\natexlab{b}})Wang, Sankaran, and Perdikaris]{wang2022respecting}
Sifan Wang, Shyam Sankaran, and Paris Perdikaris.
\newblock Respecting causality is all you need for training physics-informed neural networks.
\newblock \emph{arXiv preprint arXiv:2203.07404}, 2022{\natexlab{b}}.

\bibitem[Wu et~al.(2023)Wu, Zhu, Tan, Kartha, and Lu]{wu2023comprehensive}
Chenxi Wu, Min Zhu, Qinyang Tan, Yadhu Kartha, and Lu~Lu.
\newblock A comprehensive study of non-adaptive and residual-based adaptive sampling for physics-informed neural networks.
\newblock \emph{Computer Methods in Applied Mechanics and Engineering}, 403:\penalty0 115671, 2023.

\bibitem[Daw et~al.(2023)Daw, Bu, Wang, Perdikaris, and Karpatne]{Daw2023R3}
Arka Daw, Jie Bu, Sifan Wang, Paris Perdikaris, and Anuj Karpatne.
\newblock Mitigating propagation failures in physics-informed neural networks using retain-resample-release ({R}3) sampling.
\newblock In \emph{International Conference on Machine Learning}, volume 202 of \emph{Proceedings of Machine Learning Research}, pages 7264--7302. PMLR, 23--29 Jul 2023.

\bibitem[Vadeboncoeur et~al.(2023{\natexlab{a}})Vadeboncoeur, Akyildiz, Kazlauskaite, Girolami, and Cirak]{vadeboncoeur2023fully}
Arnaud Vadeboncoeur, {\"O}mer~Deniz Akyildiz, Ieva Kazlauskaite, Mark Girolami, and Fehmi Cirak.
\newblock Fully probabilistic deep models for forward and inverse problems in parametric {PDE}s.
\newblock \emph{Journal of Computational Physics}, 491:\penalty0 112369, 2023{\natexlab{a}}.

\bibitem[Vadeboncoeur et~al.(2023{\natexlab{b}})Vadeboncoeur, Kazlauskaite, Papandreou, Cirak, Girolami, and Akyildiz]{vadeboncoeur23aRandomgrid}
Arnaud Vadeboncoeur, Ieva Kazlauskaite, Yanni Papandreou, Fehmi Cirak, Mark Girolami, and {\"O}mer~Deniz Akyildiz.
\newblock Random grid neural processes for parametric partial differential equations.
\newblock In \emph{Proceedings of the 40th International Conference on Machine Learning}, volume 202 of \emph{Proceedings of Machine Learning Research}, pages 34759--34778, 23--29 Jul 2023{\natexlab{b}}.

\bibitem[Krishnapriyan et~al.(2021)Krishnapriyan, Gholami, Zhe, Kirby, and Mahoney]{Krishnapriyan2021failure}
Aditi Krishnapriyan, Amir Gholami, Shandian Zhe, Robert Kirby, and Michael~W Mahoney.
\newblock Characterizing possible failure modes in physics-informed neural networks.
\newblock In \emph{Advances in Neural Information Processing Systems}, volume~34, pages 26548--26560, 2021.

\bibitem[Wang et~al.(2021)Wang, Teng, and Perdikaris]{wang2021understanding}
Sifan Wang, Yujun Teng, and Paris Perdikaris.
\newblock Understanding and mitigating gradient flow pathologies in physics-informed neural networks.
\newblock \emph{SIAM Journal on Scientific Computing}, 43\penalty0 (5):\penalty0 A3055--A3081, 2021.

\bibitem[Wang et~al.(2022{\natexlab{c}})Wang, Yu, and Perdikaris]{wang2022NTK}
Sifan Wang, Xinling Yu, and Paris Perdikaris.
\newblock When and why {PINN}s fail to train: A neural tangent kernel perspective.
\newblock \emph{Journal of Computational Physics}, 449:\penalty0 110768, 2022{\natexlab{c}}.

\bibitem[Steger et~al.(2022)Steger, Rohrhofer, and Geiger]{steger2022how}
Sophie Steger, Franz~M. Rohrhofer, and Bernhard~C Geiger.
\newblock How {PINN}s cheat: Predicting chaotic motion of a double pendulum.
\newblock In \emph{The Symbiosis of Deep Learning and Differential Equations II}, 2022.

\bibitem[Wong et~al.(2022)Wong, Ooi, Gupta, and Ong]{wong2022learning}
Jian~Cheng Wong, Chinchun Ooi, Abhishek Gupta, and Yew-Soon Ong.
\newblock Learning in sinusoidal spaces with physics-informed neural networks.
\newblock \emph{IEEE Transactions on Artificial Intelligence}, 2022.

\bibitem[Yao et~al.(2023)Yao, Su, Hao, Liu, Su, and Zhu]{pmlr-v202-yao23c}
Jiachen Yao, Chang Su, Zhongkai Hao, Songming Liu, Hang Su, and Jun Zhu.
\newblock {M}ulti{A}dam: Parameter-wise scale-invariant optimizer for multiscale training of physics-informed neural networks.
\newblock In \emph{International Conference on Machine Learning}, volume 202 of \emph{Proceedings of Machine Learning Research}, pages 39702--39721. PMLR, 23--29 Jul 2023.

\bibitem[Kim et~al.(2021)Kim, Lee, Lee, Jhin, and Park]{kim2021dpm}
Jungeun Kim, Kookjin Lee, Dongeun Lee, Sheo~Yon Jhin, and Noseong Park.
\newblock {DPM}: A novel training method for physics-informed neural networks in extrapolation.
\newblock In \emph{AAAI Conference on Artificial Intelligence}, number~9, pages 8146--8154, 2021.

\bibitem[Bahmani and Sun(2021)]{bahmani2021training}
Bahador Bahmani and WaiChing Sun.
\newblock Training multi-objective/multi-task collocation physics-informed neural network with student/teachers transfer learnings.
\newblock \emph{arXiv preprint arXiv:2107.11496}, 2021.

\bibitem[Yu et~al.(2020)Yu, Kumar, Gupta, Levine, Hausman, and Finn]{Yu2020PCGrad}
Tianhe Yu, Saurabh Kumar, Abhishek Gupta, Sergey Levine, Karol Hausman, and Chelsea Finn.
\newblock Gradient surgery for multi-task learning.
\newblock In \emph{Advances in Neural Information Processing Systems}, volume~33, pages 5824--5836, 2020.

\bibitem[Li et~al.(2023)Li, Liu, and Liu]{li2023physics}
Xiaojian Li, Yuhao Liu, and Zhengxian Liu.
\newblock Physics-informed neural network based on a new adaptive gradient descent algorithm for solving partial differential equations of flow problems.
\newblock \emph{Physics of Fluids}, 35\penalty0 (6), 2023.

\bibitem[Sener and Koltun(2018)]{sener2018mgda}
Ozan Sener and Vladlen Koltun.
\newblock Multi-task learning as multi-objective optimization, 2018.

\bibitem[D\'{e}sid\'{e}ri(2012)]{desideri2012mgda}
Jean-Antoine D\'{e}sid\'{e}ri.
\newblock Multiple-gradient descent algorithm ({MGDA}) for multiobjective optimization.
\newblock \emph{Comptes Rendus Mathematique}, 350\penalty0 (5):\penalty0 313--318, 2012.

\bibitem[Liu et~al.(2021{\natexlab{a}})Liu, Liu, Jin, Stone, and Liu]{Liu2021CAGrad}
Bo~Liu, Xingchao Liu, Xiaojie Jin, Peter Stone, and Qiang Liu.
\newblock Conflict-averse gradient descent for multi-task learning.
\newblock In \emph{Advances in Neural Information Processing Systems}, volume~34, pages 18878--18890, 2021{\natexlab{a}}.

\bibitem[Liu et~al.(2021{\natexlab{b}})Liu, Li, Kuang, Xue, Chen, Yang, Liao, and Zhang]{liu2021IMTL}
Liyang Liu, Yi~Li, Zhanghui Kuang, J~Xue, Yimin Chen, Wenming Yang, Qingmin Liao, and Wayne Zhang.
\newblock Towards impartial multi-task learning.
\newblock 2021{\natexlab{b}}.

\bibitem[Navon et~al.(2022)Navon, Shamsian, Achituve, Maron, Kawaguchi, Chechik, and Fetaya]{navon2022nash}
Aviv Navon, Aviv Shamsian, Idan Achituve, Haggai Maron, Kenji Kawaguchi, Gal Chechik, and Ethan Fetaya.
\newblock Multi-task learning as a bargaining game.
\newblock In Kamalika Chaudhuri, Stefanie Jegelka, Le~Song, Csaba Szepesvari, Gang Niu, and Sivan Sabato, editors, \emph{Proceedings of the 39th International Conference on Machine Learning}, volume 162 of \emph{Proceedings of Machine Learning Research}, pages 16428--16446. PMLR, 17--23 Jul 2022.

\bibitem[Senushkin et~al.(2023)Senushkin, Patakin, Kuznetsov, and Konushin]{senushkin2023AlignmentMTL}
Dmitry Senushkin, Nikolay Patakin, Arseny Kuznetsov, and Anton Konushin.
\newblock Independent component alignment for multi-task learning.
\newblock In \emph{Proceedings of the IEEE/CVF Conference on Computer Vision and Pattern Recognition}, pages 20083--20093, 2023.

\bibitem[Kingma and Ba(2014)]{kingma2014adam}
Diederik~P Kingma and Jimmy Ba.
\newblock Adam: A method for stochastic optimization.
\newblock \emph{arXiv preprint arXiv:1412.6980}, 2014.

\bibitem[Hochman and Rodgers(1969)]{hochman69pareto}
Harold~M. Hochman and James~D. Rodgers.
\newblock Pareto optimal redistribution.
\newblock \emph{The American economic review}, 59\penalty0 (4):\penalty0 542--557, 1969.

\bibitem[Fesser et~al.(2023)Fesser, Qiu, and D'Amico-Wong]{fesser2023understanding}
Lukas Fesser, Richard Qiu, and Luca D'Amico-Wong.
\newblock Understanding and mitigating extrapolation failures in physics-informed neural networks.
\newblock \emph{arXiv preprint arXiv:2306.09478v2}, 2023.

\bibitem[Glorot and Bengio(2010)]{glorot2010understanding}
Xavier Glorot and Yoshua Bengio.
\newblock Understanding the difficulty of training deep feedforward neural networks.
\newblock In \emph{International Conference on Artificial Intelligence and Statistics}, pages 249--256. JMLR Workshop and Conference Proceedings, 2010.

\bibitem[Hao et~al.(2023)Hao, Yao, Su, Su, Wang, Lu, Xia, Zhang, Liu, Lu, et~al.]{hao2023pinnacle}
Zhongkai Hao, Jiachen Yao, Chang Su, Hang Su, Ziao Wang, Fanzhi Lu, Zeyu Xia, Yichi Zhang, Songming Liu, Lu~Lu, et~al.
\newblock {PINN}acle: A comprehensive benchmark of physics-informed neural networks for solving {PDE}s.
\newblock \emph{arXiv preprint arXiv:2306.08827}, 2023.

\end{thebibliography}
}

%%%%%%%%%%%%%%%%%%%%%%%%%%%%%%%%%%%%%%%%%%%%%%%%%%%%%%%%%%%%

%Optionally include supplemental material (complete proofs, %additional experiments and plots) in appendix.
%All such materials \textbf{SHOULD be included in the main %submission.}

%%%%%%%%%%%%%%%%%%%%%%%%%%%%%%%%%%%%%%%%%%%%%%%%%%%%%%%%%%%%

\newpage
\section*{NeurIPS Paper Checklist}

\begin{enumerate}

\item {\bf Claims}
    \item[] Question: Do the main claims made in the abstract and introduction accurately reflect the paper's contributions and scope?
    \item[] Answer: \answerYes{} % Replace by \answerYes{}, \answerNo{}, or \answerNA{}.
    Our abstract and introduction accurately reflect the paper's contributions and scope. 

\item {\bf Limitations}
    \item[] Question: Does the paper discuss the limitations of the work performed by the authors?
    \item[] Answer: \answerYes{} % Replace by \answerYes{}, \answerNo{}, or \answerNA{}.
    We have discussed the limitations of our work in Section~\ref{sec:conclusion}.

\item {\bf Theory Assumptions and Proofs}
    \item[] Question: For each theoretical result, does the paper provide the full set of assumptions and a complete (and correct) proof?
    \item[] Answer: \answerYes{} % Replace by \answerYes{}, \answerNo{}, or \answerNA{}.
    We have stated the explicit assumptions and provide complete proofs for main results in Appendix~\ref{app:proofs}.

    \item {\bf Experimental Result Reproducibility}
    \item[] Question: Does the paper fully disclose all the information needed to reproduce the main experimental results of the paper to the extent that it affects the main claims and/or conclusions of the paper (regardless of whether the code and data are provided or not)?
    \item[] Answer: \answerYes{} % Replace by \answerYes{}, \answerNo{}, or \answerNA{}.
    The pseudo code of our algorithms can be found in Appendix~\ref{app:algorithms}, and we have stated detailed instructions in Appendix~\ref{app:details}. We have also provided code to reproduce the results of main experiments in supplementary material.

\item {\bf Open access to data and code}
    \item[] Question: Does the paper provide open access to the data and code, with sufficient instructions to faithfully reproduce the main experimental results, as described in supplemental material?
    \item[] Answer: \answerYes{} % Replace by \answerYes{}, \answerNo{}, or \answerNA{}.
    We have provided our code to reproduce the results of main experiments in supplementary material.
    

\item {\bf Experimental Setting/Details}
    \item[] Question: Does the paper specify all the training and test details (e.g., data splits, hyperparameters, how they were chosen, type of optimizer, etc.) necessary to understand the results?
    \item[] Answer: \answerYes{} % Replace by \answerYes{}, \answerNo{}, or \answerNA{}.
    All experimental setting is explained in Appendix~\ref{app:details} including hyperparameter choices, training and test details. 

\item {\bf Experiment Statistical Significance}
    \item[] Question: Does the paper report error bars suitably and correctly defined or other appropriate information about the statistical significance of the experiments?
    \item[] Answer: \answerYes{} % Replace by \answerYes{}, \answerNo{}, or \answerNA{}.

    We repeated the main experiments in several times and reported several statistics including min, max, and the standard deviations for our numerical results.
    

\item {\bf Experiments Compute Resources}
    \item[] Question: For each experiment, does the paper provide sufficient information on the computer resources (type of compute workers, memory, time of execution) needed to reproduce the experiments?
    \item[] Answer: \answerYes{} % Replace by \answerYes{}, \answerNo{}, or \answerNA{}.
    See Application~\ref{app:env} for details of resources.
    
\item {\bf Code Of Ethics}
    \item[] Question: Does the research conducted in the paper conform, in every respect, with the NeurIPS Code of Ethics \url{https://neurips.cc/public/EthicsGuidelines}?
    \item[] Answer: \answerYes{} % Replace by \answerYes{}, \answerNo{}, or \answerNA{}.

\item {\bf Broader Impacts}
    \item[] Question: Does the paper discuss both potential positive societal impacts and negative societal impacts of the work performed?
    \item[] Answer: \answerNA{} % Replace by \answerYes{}, \answerNo{}, or \answerNA{}.
    
\item {\bf Safeguards}
    \item[] Question: Does the paper describe safeguards that have been put in place for responsible release of data or models that have a high risk for misuse (e.g., pretrained language models, image generators, or scraped datasets)?
    \item[] Answer: \answerNA{} % Replace by \answerYes{}, \answerNo{}, or \answerNA{}.

\item {\bf Licenses for existing assets}
    \item[] Question: Are the creators or original owners of assets (e.g., code, data, models), used in the paper, properly credited and are the license and terms of use explicitly mentioned and properly respected?
    \item[] Answer: \answerYes{} % Replace by \answerYes{}, \answerNo{}, or \answerNA{}.
    The overall experiments followed the existing experimental setup which have cited, and the code and data are available for use under the MIT license. For CausalPINNs in \cite{wang2022respecting}, the code is available under the CC-BY-NC-SA 4.0 license.

\item {\bf New Assets}
    \item[] Question: Are new assets introduced in the paper well documented and is the documentation provided alongside the assets?
    \item[] Answer: \answerNA{} % Replace by \answerYes{}, \answerNo{}, or \answerNA{}.

\item {\bf Crowdsourcing and Research with Human Subjects}
    \item[] Question: For crowdsourcing experiments and research with human subjects, does the paper include the full text of instructions given to participants and screenshots, if applicable, as well as details about compensation (if any)? 
    \item[] Answer: \answerNA{} % Replace by \answerYes{}, \answerNo{}, or \answerNA{}.

\item {\bf Institutional Review Board (IRB) Approvals or Equivalent for Research with Human Subjects}
    \item[] Question: Does the paper describe potential risks incurred by study participants, whether such risks were disclosed to the subjects, and whether Institutional Review Board (IRB) approvals (or an equivalent approval/review based on the requirements of your country or institution) were obtained?
    \item[] Answer: \answerNA{} % Replace by \answerYes{}, \answerNo{}, or \answerNA{}.

\end{enumerate}
%%%%%%%%%%%%%%%%%%%%%%%%%%%%%%%%%%%%%%%%%%%%%%%%%%%%%%%%%%%%%%
\clearpage
\newpage
\appendix

\section{Proofs for Section~\ref{sec:methodology}}\label{app:proofs}

\begin{proof}[Proof of Theorem~\ref{thm:dual_cone_region}] Recall that  $\phi_t$ is the angle between $\nabla \cL_r(\theta_t)$ and $\nabla \cL_r(\theta_t)$, and $R = \frac{\|\nabla \cL_r(\theta_t)\|}{\|\nabla \cL_b(\theta_t)\|}$ at each iteration $t$. Consider a cone $K_t$, defined as 
$$
\rmK_t:=\left\{c x | c\geq 0, x \in \{\nabla \cL_r(\theta_t), \nabla \cL_b(\theta_t)\}\right\}.
$$

\paragraph{Case (i).} Suppose that  $\ip{\nabla \cL_b(\theta_t)}{\nabla \cL_r(\theta_t)} \geq 0$. Observe that 
\begin{align*}
\ip{\nabla \cL(\theta_t)}{\nabla \cL_r(\theta_t)} &= \ip{\nabla \cL_b(\theta_t)}{\nabla \cL_r(\theta_t)}+\|\nabla \cL_r(\theta_t)\|^2 \geq 0, \\
\ip{\nabla \cL(\theta_t)}{\nabla \cL_b(\theta_t)} &= \ip{\nabla \cL_b(\theta_t)}{\nabla \cL_r(\theta_t)}+\|\nabla \cL_b(\theta_t)\|^2 \geq 0.
\end{align*}
Therefore, by the definition of the dual cone, we have $\nabla \cL(\theta_t) \in \rmK_t^*$. 

\paragraph{Case (ii).} Suppose that $\ip{\nabla \cL_b(\theta_t)}{\nabla \cL_r(\theta_t)} < 0$ and $ -\cos(\phi_t) \leq R \leq -\frac{1}{\cos(\phi_t)} $. By multiplying $\|\nabla \cL_b\|\|\nabla \cL_r\|$ to $-\cos(\phi_t) \leq R$, we get 
\begin{align}
-\cos(\phi_t) \leq R&\Leftrightarrow -\|\nabla \cL_b(\theta_t)\|\|\nabla \cL_r(\theta_t)\| \cos(\phi_t) \leq \|\nabla \cL_r(\theta_t)\|^2, \nonumber \\
&\Leftrightarrow \ip{\nabla \cL_b(\theta_t)}{\nabla \cL_r(\theta_t)} + \ip{\nabla \cL_r(\theta_t)}{\nabla \cL_r(\theta_t)} \geq 0, \nonumber \\
&\Leftrightarrow \ip{\nabla \cL(\theta_t)}{\nabla \cL_r(\theta_t)} \geq 0. \label{ineq:inner1}
\end{align}
On the other hand, by multiplying $\|\nabla \cL_b(\theta_t)\|^2 \cos(\phi_t)$ to $R \leq -\frac{1}{\cos (\phi_t)}$, we have
\begin{align}
R \leq -\frac{1}{-\cos (\phi_t)}&\Leftrightarrow\|\nabla \cL_b(\theta_t)\|\|\nabla \cL_r(\theta_t)\| \cos(\phi_t) \geq -\|\nabla \cL_b(\theta_t)\|^2, \nonumber \\
&\Leftrightarrow \ip{\nabla \cL_b(\theta_t)}{\nabla \cL_r(\theta_t)} + \ip{\nabla \cL_b(\theta_t)}{\nabla \cL_b(\theta_t)} \geq 0, \nonumber \\
&\Leftrightarrow \ip{\nabla \cL(\theta_t)}{\nabla \cL_b(\theta_t)}\geq 0. \label{ineq:inner2}
\end{align}

Therefore, we conclude that if $\ip{\nabla \cL_b(\theta_t)}{\nabla \cL_r(\theta_t)} < 0$, then $ -\cos(\phi_t) \leq R \leq -\frac{1}{\cos(\phi_t)} $ is equivalent to $\nabla \cL(\theta_t)\in \rmK_t^*$. 

\end{proof}

\begin{proof}[Proof of Proposition~\ref{prop:dual_cone_G}]

Recall that 
\begin{align}
    \rmG_t := \left\{c_1 \nabla_t \cL_{\|\nabla \cL_r^\perp} + c_2 \nabla_t \cL_{\|\nabla \cL_b^\perp} \big| c_1, c_2\geq 0 \right\}.
\end{align}

It is enough to show that we have $\ip{g}{\nabla \cL_r(\theta_t)}\geq 0$ and $\ip{g}{\nabla \cL_b(\theta_t)}\geq 0$ for any $g\in \rmG_t$. By the definition of $\rmG_t$, there exists $c_1, c_2\geq 0$ such that $g = c_1 \nabla_t \cL_{\|\nabla \cL_r^\perp} + c_2 \nabla_t \cL_{\|\nabla \cL_b^\perp}$ for all $g \in \rmG_t$. One can easily check that 
\begin{align*}
        \ip{g}{\nabla \cL_r(\theta_t)} &=  \ip{c_2 \nabla_t \cL_{\|\nabla \cL_b^\perp}}{\nabla \cL_r(\theta_t)} \\
        &= \ip{c_2\left(\nabla \cL_r(\theta_t) - \frac{\ip{\nabla \cL_b(\theta_t)}{\nabla \cL_r(\theta_t)}}{\|\nabla \cL_b(\theta_t)\|^2}\nabla \cL_b(\theta_t) \right)}{\nabla \cL_r(\theta_t)} \\
        &= c_2\left(\|\nabla \cL_r(\theta_t)\|^2-\frac{|\ip{\nabla \cL_b(\theta_t)}{\nabla \cL_r(\theta_t)}|^2}{\|\nabla \cL_b(\theta_t)\|^2} \right) \\
        &= c_2\|\nabla \cL_r(\theta_t)\|^2(1-\cos(\phi_t)) \\
        &\geq 0
\end{align*}
where $\phi_t$ is the angle between $\nabla \cL_r(\theta_t)$ and $\nabla \cL_b(\theta_t)$. One can derive that $\ip{g}{\nabla \cL_r(\theta_t)}\geq 0$ in the same manner. Therefore, we conclude that $\rmG_t \subset \rmK_t^*$. 

\end{proof}

\begin{proof}[Proof of Theorem~\ref{thm:convergence_nonconvex}]
Let $\phi_t$ be the angle between $\nabla \cL_r(\theta_t)$ and $\nabla \cL_b(\theta_t)$, and $\psi_t$ be the angle between $g_t^{\text{dual}}$ and $\nabla \cL(\theta_t)$ at the $t$-th iteration. Note that $\theta_{t+1} = \theta_t - \lambda g_t^{\text{dual}}$ where $g_t^{\text{dual}}$ satisfies the conditions (i), (ii) of Theorem~\ref{thm:convergence_nonconvex}.
    
First of all, DCGD algorithm reaches a Pareto-stationary point if $\phi_t=-1$ by Definition~\ref{def:pareto}, at which the optimization process is stopped. 

Otherwise, we first observe from the differentiability and $L$-Lipschitz continuity condition of $\nabla \cL(\cdot)$ for all $x, y \in \sR^d$:
\begin{align}
        \cL(x) - \cL(y) &= \int_0^1 \ip{\nabla \cL(y + t(x-y))}{x-y} \rd t \nonumber \\ 
        &\leq  \ip{\nabla \cL(y)}{x-y}  + \int_0^1 \ip{\nabla \cL(y + t(x-y))-\nabla \cL(y)}{x-y}\rd t    \nonumber \\
        &\leq  \ip{\nabla \cL(y)}{x-y} +\int_0^1\|\nabla \cL(y + t(x-y))-\nabla \cL(y)\|\|x-y\| \rd t \nonumber \\
        &\leq  \ip{\nabla \cL(y)}{x-y} +\int_0^1 L t \|x-y\|^2 \rd t \nonumber \\
        &= \ip{\nabla \cL(y)}{x-y}  + \frac{L}{2}\|x-y\|^2 \label{ineq:L-smooth},
\end{align}
where we have used Cauchy-Schwarz inequality for the third inequality. Using \Eqref{ineq:L-smooth} and Conditions (i), (ii) of Theorem~\ref{thm:convergence_nonconvex},
one calculates that for $\lambda \leq \frac{1}{2L}$, 
\begin{align}
 \cL(\theta_{t+1}) - \cL(\theta_{t}) &\leq - \lambda \ip{\nabla\cL(\theta_{t})}{g_t^{\text{dual}}} + \frac{L\lambda^2}{2} \| g_t^{\text{dual}} \|^2 \nonumber \\
    &\leq - \lambda \ip{\nabla\cL(\theta_{t})}{g_t^{\text{dual}}} + \frac{\lambda}{4} \| g_t^{\text{dual}} \|^2 \nonumber \\
    &= -\frac{\lambda}{4}\left(2\ip{\nabla \cL(\theta_t)}{g_t^{\text{dual}}}-\|g_t^{\text{dual}}\|^2 + 2\ip{\nabla \cL(\theta_t)}{g_t^{\text{dual}}}  \right) \nonumber \\
    &\leq -\frac{\lambda}{2} \ip{\nabla \cL(\theta_t)}{g_t^{\text{dual}}} \quad \because \text{condition (i)} \nonumber  \\
    &\leq -\frac{\lambda M}{2} \|\nabla \cL(\theta_t)\|^2 \quad \because \text{Cauchy-Swartz inequality and condition (2)}  
\end{align}

By using telescoping sums, we further obtain
\begin{align*}
\sum_{t=0}^{T} \cL(\theta_{t+1})-\cL(\theta_t) &= \cL(\theta_{T+1}) - \cL(\theta_0) \\
&\leq -\frac{\lambda M}{2}\sum_{t=0}^T \|\nabla \cL(\theta_t)\|^2,
\end{align*}
which yields 
\begin{align*}
    \frac{1}{T+1}\sum_{t=0}^T \|\nabla \cL(\theta_t)\|^2 \leq \frac{2\left(\cL(\theta_0)-\cL(\theta_{T+1})\right)}{\lambda M(T+1)} \\
    \leq \frac{2\left(\cL(\theta_0)-\cL(\theta^*)\right)}{\lambda M(T+1)}.
\end{align*}

\end{proof}

\begin{proof}[Proof of Proposition~\ref{prop: DCGD(center)}]\label{pf:center in region}

Note that $g_t^c$ is the angle bisector of $\nabla \cL_r(\theta_t)$ and $\nabla \cL_b(\theta)$. From the formula of vector projection, $g_t^{\text{dual}}$ of DCGD (Center) is the projection of $\nabla \cL(\theta_t)$ on to $g_t^c$. Thus, it is enough to show that $g_t^c$ is included in $\rmG_t$. 

We observe that 
\begin{align}
    \nabla_t \cL_{\| \nabla \cL_r^\perp} &= \nabla \cL(\theta_t) - \ip{\nabla \cL(\theta_t)}{\nabla \cL_r(\theta_t)}\frac{\nabla \cL_r(\theta_t)}{\|\nabla \cL_r(\theta_t)\|^2} \label{eq:grad_r_perp}\\
    &=\nabla \cL_b(\theta_t) - \ip{\nabla \cL_b(\theta_t)}{\nabla \cL_r(\theta_t)}\frac{\nabla \cL_r(\theta_t)}{\|\nabla \cL_r(\theta_t)\|^2},
\end{align}
and 
\begin{align}
    \nabla_t \cL_{\| \nabla \cL_b^\perp} &= \nabla \cL(\theta_t) - \ip{\nabla \cL(\theta_t)}{\nabla \cL_b(\theta_t)}\frac{\nabla \cL_b(\theta_t)}{\|\nabla \cL_b(\theta_t)\|^2} \label{eq:grad_b_perp}\\
    &=\nabla \cL_r(\theta_t) - \ip{\nabla \cL_r(\theta_t)}{\nabla \cL_b(\theta_t)}\frac{\nabla \cL_b(\theta_t)}{\|\nabla \cL_b(\theta_t)\|^2}.
\end{align}
Then, by defining $c_1 = \frac{1}{\|\nabla \cL_b(\theta_t)\|(1-\cos(\phi_t))}$ and $c_2 = \frac{1}{\|\nabla \cL_r(\theta_t)\|(1-\cos(\phi_t))}$, one can easily see that 
\begin{align*}
    c_1 \nabla_t \cL_{\| \nabla \cL_r^\perp} + c_2 \nabla_t \cL_{\| \nabla \cL_b^\perp} &= \nabla \cL_b(\theta_t) \left(c_1 - c_2 \frac{\ip{\nabla \cL_r(\theta_t)}{\nabla \cL_b(\theta_t)}}{\|\nabla \cL_b(\theta_t)\|^2}\right)\\
    &+\nabla \cL_r(\theta_t) \left(c_2 - c_1 \frac{\ip{\nabla \cL_r(\theta_t)}{\nabla \cL_b(\theta_t)}}{\|\nabla \cL_r(\theta_t)\|^2}\right) \\
    &= \frac{\nabla \cL_b(\theta_t)}{\|\nabla \cL_b (\theta_t)\|}+ \frac{\nabla \cL_r(\theta_t)}{\|\nabla \cL_r(\theta_t)\|} \\
    &=g_t^c.
\end{align*}
That is, $g_t^c$ can be expressed as $c_1\nabla_t \cL_{\| \nabla \cL_r^\perp} + c_2\nabla_t \cL_{\| \nabla \cL_b^\perp}$ for some $c_1,c_2\geq 0$. Therefore, $g_t^c$ is in $\rmG_t$.
\end{proof}

\begin{proof}[Proof of Corollary~\ref{cor:convergence_algorithm}]

We will show that $g_t^{\text{dual}}$ of each DCGD algorithm satisfies the conditions (i), (ii) of Theorem~\ref{thm:convergence_nonconvex}. Three algorithms are summarized in Algo.~\ref{alg:proj}, Algo.~\ref{alg:avg}, 
and Algo.~\ref{alg:center}. We note that a conflict threshold $\alpha$ is introduced as a stopping condition for DCGD algorithms, as they can reach a Pareto-stationary point characterized by $\phi_t= \pi$. That is, the algorithm stops when the parameter converges close to a Pareto-stationary point such that $|\cos(\phi_t)-\pi|<\alpha$. Here, we assume $\alpha\geq0$ is fixed. 

\paragraph{1. DCGD (Projection):} Note that it is trival to show that $g_t^{\text{dual}} = \nabla \cL(\theta_t)$, when $\ip{\nabla \cL (\theta_t)}{\nabla \cL_b (\theta_t)} \geq 0$, satisfies the conditions (i), (ii). Thus, we focus on the case when $\ip{\nabla \cL (\theta_t)}{\nabla \cL_b (\theta_t)} < 0$. 

First of all, we need to show the condition (i)
\begin{align*}
2 \ip{\nabla \cL(\theta_t)}{g_t^{\text{dual}}} - \|g_t^{\text{dual}}\|^2=  \|\nabla \cL(\theta_t)\|^2 -\|g_t^{\text{dual}}-\nabla \cL(\theta_t)\|^2 \geq 0,
\end{align*}
which is equivalent to that $\|\nabla \cL(\theta_t)\| \geq \|g_t^{\text{dual}}-\nabla \cL(\theta_t)\|$. Using \Eqref{eq:grad_r_perp}, one directly calculates that 
\begin{align*}
\|2\nabla_t \cL_{\|\cL_r^{\perp}} - \nabla\cL(\theta_t)\|^2 &= \left\| \nabla \cL(\theta_t)- \ip{\nabla \cL(\theta_t)}{\nabla \cL_r(\theta_t)}\frac{\nabla \cL_r(\theta_t)}{\|\nabla \cL_r(\theta_t)\|^2}\right\|^2 \\
&= \|\nabla \cL(\theta_t)\|^2.
\end{align*}
In the same manner, we have $\|2\nabla_t \cL_{\|\cL_b^{\perp}} - \nabla\cL(\theta_t)\|^2 = \|\nabla\cL(\theta_t)\|^2$.  Since $g_t^{\text{dual}}$ is chosen in $\rmG_t$, specifically $c_1=1, c_2=0$ or $c_1=0, c_2=0$, we can write 
\begin{align}
    \|g_t^{\text{dual}}-\nabla\cL(\theta_t)\| &= \left\| \left(c_1 \nabla \cL_{\|\cL_r^{\perp}}(\theta_t) + c_2 \nabla
    \cL_{\|\cL_b^{\perp}}(\theta_t) \right) - \nabla\cL(\theta_t)\right\|\nonumber \\
    &\leq \left \| \left(c_1 \nabla \cL_{\|\cL_r^{\perp}}(\theta_t) + c_2 \nabla
    \cL_{\|\cL_b^{\perp}}(\theta_t) \right) - \frac{c_1+c_2}{2} \nabla\cL(\theta_t) \right\| +\left\| \left(\frac{c_1+c_2}{2}-1\right)\nabla\cL(\theta_t)  \right\| \nonumber\\
    &\leq \left \| \frac{c_1}{2} \left(2\nabla \cL_{\|\cL_r^{\perp}}(\theta_t)- \nabla\cL(\theta_t)\right) \right \| +\left \| \frac{c_2}{2} \left(2\nabla \cL_{\|\cL_b^{\perp}}(\theta_t)- \nabla\cL(\theta_t)\right) \right \| + \left \| \left(\frac{c_1+c_2}{2}-1\right)\nabla\cL(\theta_t)  \right\| \nonumber\\
    &= \frac{c_1}{2} \| \nabla \cL (\theta_t) \| + \frac{c_2}{2} \| \nabla \cL (\theta_t) \| + \left | \frac{c_1+c_2}{2} -1 \right | \| \nabla \cL (\theta_t) \| \nonumber\\
    &= \left( \frac{c_1+c_2}{2} + \left| \frac{c_1+c_2}{2} - 1\right| \right) \| \nabla \cL (\theta_t) \| \label{eq:condition1_c1_c2}\\
    &= \| \nabla \cL (\theta_t) \|.\nonumber
\end{align}
where we have used $c_1+c_2=1$ for obtaining the last inequality. Therefore, the condition (i) is satisfied.

We further suppose that $\ip{\nabla \cL (\theta_t)}{\nabla \cL_b (\theta_t)} < 0$. Then, $g_t^{\text{dual}} = \nabla_t \cL_{\|\cL_b^{\perp}}$. Let $\phi_t$ be the angle between $\nabla \cL_b(\theta_t)$ and $\nabla \cL_r(\theta_t)$, and $\psi_t$ be the angle between $g_t^{\text{dual}}$ and $\nabla \cL(\theta_t)$. Note that $\phi_t\leq \pi - \alpha$ where $\alpha$ is conflict threshold. Otherwise, the algorithm stops when $\pi-\alpha<\phi_t\leq \pi$ (see Algo.~\ref{alg:proj}). Then, since $\psi_t = \phi_t - \frac{\pi}{2}$, we have
\begin{align*}
        \|g_t^{\text{dual}}\| &= \|\nabla_t \cL_{\|\cL_b^{\perp}} \| \\
        &= \|\nabla \cL (\theta_t)\| \cos(\psi_t) \\
        &= \|\nabla \cL (\theta_t)\| \cos\left(\phi_t-\frac{\pi}{2} \right) \\
        &\geq \|\nabla \cL (\theta_t)\| \cos\left(\frac{\pi}{2}-\alpha \right).
\end{align*}
Thus, by choosing $M=\cos\left(\frac{\pi}{2}-\alpha \right)$, the condition (ii) is satisfied. We repeat the same analysis for the case when $\ip{\nabla \cL (\theta_t)}{\nabla \cL_b (\theta_t)} < 0$.

\paragraph{2. DCGD (Average):} Similarly to DCGD (Projection), we focus on the case where DCGD (Average) specifies $c_1=c_2=\frac{1}{2}$, given by 
$$
g_t^{\text{dual}} = \frac{1}{2}\left(\nabla_t \cL_{\|\nabla \cL_r^{\perp}} + \nabla_t \cL_{\|\nabla \cL_b^{\perp}}\right),
$$
when $\nabla \cL(\theta_t) \notin \rmK_t^*$.  Eq~\ref{eq:condition1_c1_c2} with $c_1=c_2=1/2$ directly leads to 
$$
\|g_t^{\text{dual}} - \nabla \cL(\theta_t)\| \leq \|\nabla \cL(\theta_t)\|,
$$
implying that the condition (i) is satisfied.

Next, suppose  $\ip{\nabla \cL (\theta_t)}{\nabla \cL_b (\theta_t)} < 0$. Then, the condition (ii) is satisfied with $M= \frac{1}{2}\cos \left(\frac{\pi}{2}-\alpha\right)$ since 
\begin{align*}
         \|g_t^{\text{dual}}\| &= \frac{1}{2}\left\|\nabla_t \cL_{\|\cL_r^{\perp}} +   \nabla_t \cL_{\|\cL_b^{\perp}}\right\| \\         
         &\geq \frac{1}{2}\|\nabla_t \cL_{\|\cL_b^{\perp}} \| \\
         &= \frac{1}{2}\|\nabla \cL (\theta_t)\| \cos\left(\phi_t-\frac{\pi}{2} \right) \\
         &\geq \frac{1}{2}\cos\left(\frac{\pi}{2}-\alpha \right)\|\nabla \cL (\theta_t)\| .
\end{align*}

When $\ip{\nabla \cL (\theta_t)}{\nabla \cL_r (\theta_t)} < 0$, the condition (ii) is also satisfied with $M= \frac{1}{2}\cos \left(\frac{\pi}{2}-\alpha\right)$.

\paragraph{3. DCGD (Center):} the updated vecotr of DCGD (Center) is given by 
$$
g_t^{\text{dual}} = \frac{\ip{g_t^c}{\nabla \cL(\theta_t)}}{\|g_t^c\|^2}g_t^c
$$
where $g_t^c = \frac{\nabla \cL_b(\theta_t)}{\|\nabla \cL_b(\theta_t)\|}+\frac{\nabla \cL_r(\theta_t)}{\|\nabla \cL_r(\theta_t)\|}$. Since $g_t^{\text{dual}}$ is the angle bisector of $\nabla \cL_r(\theta_t)$ and $\nabla \cL_b(\theta_t)$ (see Proof of Proposition~\ref{prop: DCGD(center)}), $\psi_t$, the angle between $\nabla \cL(\theta_t)$ and $g_t^{\text{dual}}$, is less or equal to $\phi_t/2$, i.e., $\psi_t \leq \frac{\phi_t}{2}\leq \frac{\pi}{2}$.  From the fact that $g_t^{\text{dual}}$ is the projection of $\nabla \cL(\theta_t)$ onto $g_t^c$, we have 
\begin{align*}
\|\nabla \cL(\theta_t) - g_t^{\text{dual}}\| &= \|\nabla \cL(\theta_t)\| \sin(\psi_t) \\
&\leq \| \nabla \cL (\theta_t) \| \sin \left(\frac{\phi_t}{2}\right) \\
&\leq \| \nabla \cL (\theta_t) \| \sin\left(\frac{\pi - \alpha}{2}\right) \\
&\leq \|\nabla \cL (\theta_t)\|,
\end{align*}
and
\begin{align*}
    \|g_t^{\text{dual}}\| &= \|\nabla \cL(\theta_t)\| \cos(\psi_t) \\
    &\geq \| \nabla \cL (\theta_t) \| \cos\left(\frac{\phi_t}{2}\right) \\
    &\geq \| \nabla \cL (\theta_t) \| \cos\left(\frac{\pi - \alpha}{2}\right).
\end{align*}

Consequently, the conditions (i) and (ii) are satisfied for DCGD (Center). 

\end{proof}

\begin{remark}
Suppose that one employs a decaying scheme for the conflict threshold $\alpha_t$ such that $\alpha_t=\cO(t^{-\gamma})$ with where $0\leq\gamma<1$, for example, $\alpha_t=t^{-\gamma}$. In this case, the convergence rate of the DCGD algorithm to a stationary point becomes $\cO\left(\frac{1}{T^{1-\gamma}}\right)$, as $M$ in condition (ii) may depend on the conflict threshold $\alpha_t$. For all our experiments, we set $\alpha$ to be fixed. 
\end{remark}

\section{Unification of MTL algorithms within the DCGD framework}\label{app:uni}
In this section, we prove that several MTL algorithms can be understood as special cases of the DCGD framework  under the PINN's formulation.

\begin{proof}
\textbf{1. MGAD~\cite{desideri2012mgda}:} The updated gradient $g_t^{\text{MGDA}}$ of MGDA is defined by selecting the minimum-norm element from the convex combinations of $\nabla \cL_r(\theta_t)$ and $\nabla \cL_b(\theta_t)$ if there is gradient conflict \( \ip{\nabla \cL_r(\theta_t)}{\nabla \cL_b(\theta_t)} < 0 \):
\begin{align*}
g_t^{\text{MGDA}} &:= \argmin_{\alpha_1,\alpha_2\geq 0} \|u\|,\
&\text{s.t. } u =\alpha_1\nabla \cL_r(\theta_t) + \alpha_2 \nabla \cL_b(\theta_t),\
&\alpha_1+\alpha_2=1.
\end{align*}

One can easily show that \( \ip{g_t^{\text{MGDA}}}{g_t^{\text{MGDA}}} = \ip{g_t^{\text{MGDA}}}{\nabla \cL_r(\theta_t)} = \ip{g_t^{\text{MGDA}}}{\nabla \cL_b(\theta_t)} \geq 0 \).
Thus,
\[ g_t^{\text{MGDA}} \in \rmK_{t}^{*}. \]

%\paragraph{PCGrad} 
\textbf{2. PCGrad~\cite{Yu2020PCGrad}: } PCGrad uses the same update direction with DCGD (Average) when \( \ip{\nabla \cL_r(\theta_t)}{\nabla \cL_b(\theta_t)} < 0 \),
\[ g_t^{\text{PCGrad}} = \frac{1}{2}\left(\nabla_t \cL_{\|\nabla \cL_r^{\perp}} + \nabla_t \cL_{\|\nabla \cL_b^{\perp}}\right) \in \rmK_t^*. \]
and takes \( \nabla \cL_r(\theta_t) + \nabla \cL_b(\theta_t)\) when \( \ip{\nabla \cL_r(\theta_t)}{\nabla \cL_b(\theta_t)} \geq 0 \). The latter case is also contained in $\rmK_t^*$. Therefore, $g_t^{\text{PCGrad}}\in \rmK_t^*$.

%\paragraph{Nash-MTL}

\textbf{3. Nash-MTL~\cite{navon2022nash}: } Nash-MTL considers a Nash bargaining solution to balance the loss gradients. The update gradient \(g_t^{\text{Nash-MTL}}\) is be defined by
\begin{align}
    &g_t^{\text{Nash-MTL}} := G_t v_t,  \\
    &\text{s.t. }  G_t^\top G_t v_t = v_t^{-1}. \label{eq:nash condition}
\end{align}

where \( G_t = [\nabla \cL_r(\theta_t), \nabla \cL_b(\theta_t)]\). We can find the solution \(v_t\) satisfying \Eqref{eq:nash condition} as following. By letting \( v_t = \begin{bmatrix}
        v_1   \\
        v_2
    \end{bmatrix}  \),  we have
\begin{align*}
    \begin{bmatrix}
        \Vert\nabla \cL_r(\theta_t)  \Vert^2 & \ip{\nabla \cL_r(\theta_t)}{\nabla \cL_b(\theta_t)} \\
        \ip{\nabla \cL_r(\theta_t)}{\nabla \cL_b(\theta_t)} & \Vert\nabla \cL_b(\theta_t)  \Vert^2
    \end{bmatrix} \begin{bmatrix}
        v_1   \\
        v_2
    \end{bmatrix} =
    \begin{bmatrix}
        \frac{1}{v_1}   \\
        \frac{1}{v_2}
    \end{bmatrix},
\end{align*}
which is equivalent to
\begin{align}\label{eq:nash2}
    \begin{cases}
    \Vert\nabla \cL_r(\theta_t)  \Vert^2 v_1^2 + \ip{\nabla \cL_r(\theta_t)}{\nabla \cL_b(\theta_t)} v_1v_2 = 1 \\
    \Vert\nabla \cL_b(\theta_t)  \Vert^2 v_2^2 + \ip{\nabla \cL_r(\theta_t)}{\nabla \cL_b(\theta_t)} v_1v_2 = 1
    \end{cases}.
\end{align}

Therefore, we can derive \(v_2 = \frac{\Vert\nabla \cL_r(\theta_t)\Vert}{\Vert\nabla \cL_b(\theta_t)  \Vert}v_1 \). Substituting \( v_2 = \frac{\Vert\nabla \cL_r(\theta_t)\Vert}{\Vert\nabla \cL_b(\theta_t)  \Vert}v_1  \) back into the first equation of \Eqref{eq:nash2} leads to

\begin{align*}
     &\Vert\nabla \cL_r(\theta_t)  \Vert^2 v_1^2 + \Vert\nabla \cL_r(\theta_t)\Vert^2 \cos(\phi_t)  v_1^2 = 1 \\
     &\Leftrightarrow v_1 = \sqrt{\frac{1}{1+\cos{\phi_t}}}\frac{1}{\Vert \nabla \cL_r(\theta_t)   \Vert} \\
     &\Leftrightarrow v_2 = \frac{\Vert\nabla \cL_r(\theta_t)\Vert}{\Vert\nabla \cL_b(\theta_t)  \Vert}v_1  = \sqrt{\frac{1}{1+\cos{\phi_t}}}\frac{1}{\Vert \nabla \cL_b(\theta_t) \Vert} \\
     &\Leftrightarrow G_tv_t = \sqrt{\frac{1}{1+\cos{\phi_t}}} \left( \frac{\nabla \cL_r(\theta_t)}{\|\nabla \cL_r(\theta_t)\|} + \frac{\nabla \cL_b(\theta_t)}{\|\nabla \cL_b(\theta_t)\|} \right)
\end{align*}

where \(\phi_t \) is the angle between \( \nabla \cL_r(\theta_t)\) and \(\nabla \cL_b(\theta_t)\). Thus, the update gradient \(g_t^{\text{Nash-MTL}} \) has same direction with DCGD (Center). That is, $g_t^{\text{Nash-MTL}}\in \rmK_t^*$. 

\end{proof}

\section{Experimental details}\label{app:details}

\subsection{Software and hardware environments}\label{app:env} We conduct all experiments with \textsc{Python} 3.10.9 and \textsc{Pytorch} 1.13.1, CUDA 11.6.2, NVIDIA Driver 510.10 on Ubuntu 22.04.1 LTS server which equipped with AMD Ryzen Threadripper PRO 5975WX, NVIDIA A100 80GB and NVIDIA RTX A6000. 

\subsection{Toy example}\label{app:toyexample}
We slightly modify the toy example in \cite{Liu2021CAGrad} to show our proposed method can expand the region of initial points that converge to the Pareto set. Consider the following loss functions with $\theta = (\theta_1, \theta_2) \in \R^2$:
\begin{align*}
    &L_0(\theta) = L_1(\theta) + L_2(\theta) \text{ where } \\
    &L_1(\theta) = 2c_1(\theta)f_1(\theta) + c_2(\theta)g_1(\theta) \text{ and } L_2(\theta) = c_1(\theta)f_2(\theta) + c_2(\theta)g_2(\theta),\\
    &f_1(\theta) = \log(\max(0.5(-\theta_1 - 7) - \tanh(-\theta_2), 0.000005)) + 6,\\
    &f_2(\theta) = \log(\max(0.5(-\theta_1 + 3) - \tanh(-\theta_2) + 2, 0.000005)) + 6,\\
    &g_1(\theta) = ((-\theta_1 + 7)^2 + 0.1 \cdot (-\theta_2 - 8)^2)/10 - 20,\\
    &g_2(\theta) = ((-\theta_1 - 7)^2 + 0.1 \cdot (-\theta_2 - 8)^2)/10 - 20,\\
    &c_1(\theta) = \max(\tanh(0.5 \cdot \theta_2), 0),\\
    &c_2(\theta) = \max(\tanh(-0.5 \cdot \theta_2), 0).
\end{align*}

% \begin{figure}[htb!]
%     \centering 
%      \includegraphics[width=0.4\textwidth]{figures/Toyexample/losslandscape.pdf}
% \caption{The loss landscape of toy example.}\label{fig:toy_loss_landscape}
% \end{figure}

\begin{figure}[htb!]
    \centering 
    \begin{subfigure}[b]{0.40\textwidth}
        \includegraphics[width=\textwidth]{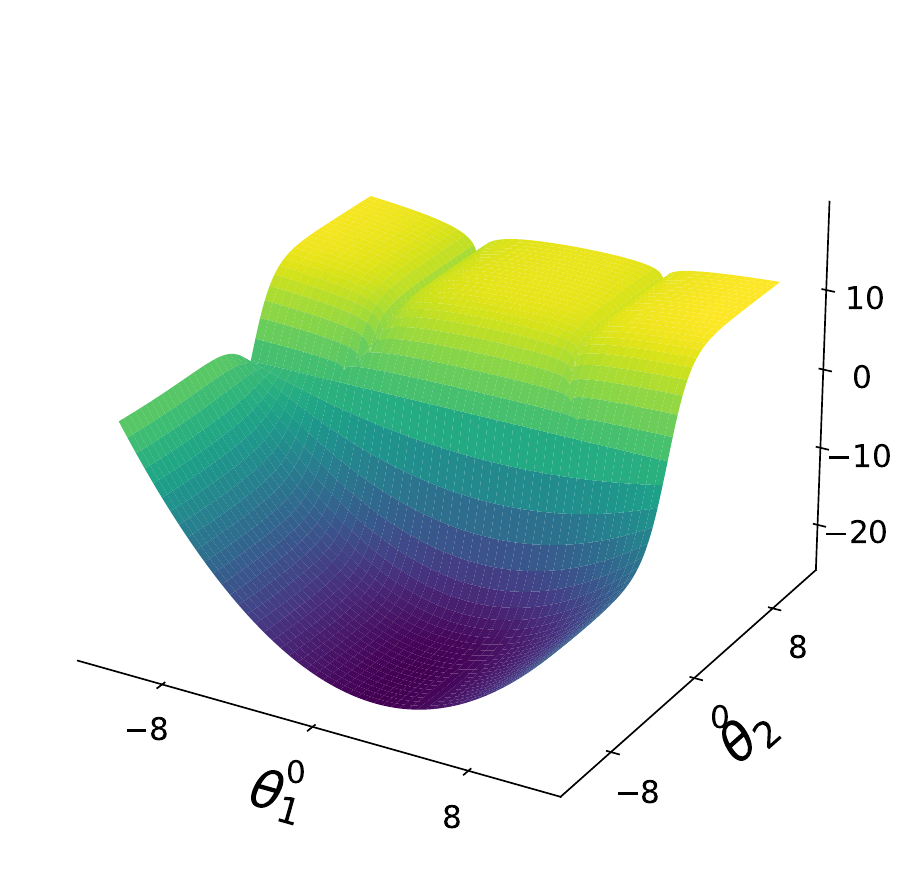}
        \caption{Loss landscape}
        %\label{fig:subfig1}
    \end{subfigure}
    \begin{subfigure}[b]{0.40\textwidth}
        \includegraphics[width=\textwidth]{figures/Toyexample/dcgd_convegence_region.pdf}
        \caption{Contour map}
        \label{Contour map}
    \end{subfigure}
    \caption{The loss landscape and contour map of the toy example. }
    \label{fig:toy_loss_landscape}
\end{figure}

The landscape and contour map of above loss function are shown in Figure~\ref{fig:toy_loss_landscape}. The Pareto set is highlighted in gray in Figure~\ref{Contour map}. We solve the above problem using Adam, DCGD (Projection), DCGD (Average), DCGD (Center) for 100,000 epochs with different initial points. The initial points are selected as $1,600$ uniform grid points within $[-10, 10]\times [-10,10]$. Then, we mark with a red dot the point at which the optimizer fails to converge to the Pareto set. 

\subsection{Details for Figure~\ref{fig:conf_dom_grads}, Figure~\ref{fig:adverse_training}, and Figure~\ref{fig:cos_dist}}

In this experiment, we use the $7$-layer fully connected neural network with 20 neurons per layer and a hyperbolic tangent activation function. We train PINN models using SGD with the learning rate of $0.01$ for $10,000$ epochs. In addition, 100 data points are sampled in boundaries and 10,000 points in the domain. 

\subsection{Details for Section~\ref{subsec:benchmark}}\label{app:benchmark}

\paragraph{Benchmark equations}We consider Helmholtz equation, viscous Burgers’ equation, and Klein-Gordon equation as the benchmark equations. 

The Helmholtz equation is described by 
\begin{align*}
    &\Delta u(x,y) +k^2u(x,y) = f(x,y),  \quad (x,y) \in \Omega, \\ 
    &u(x,y) = 0, \quad (x,y) \in \partial \Omega, \\
    &\Omega = [-1,1] \times [-1,1].
\end{align*}
The solution is given by 
$u^*(x,y) = \sin(a_1\pi x)\sin(a_2 \pi y)$
where
\begin{align*}
f(x,y) =(k^2-a_1^2\pi^2-a_2^2\pi^2)\sin(a_1\pi x)\sin(a_2 \pi y)
\end{align*}
In our experiment, we choose parameters: $k=1, a_1=1, a_2=4$ as in \cite{wang2021understanding}.

The Viscous Burgers' equation is given by
\begin{align*}
    &u_t(t,x) +uu_x(t,x) - \nu u_{xx}(t,x) = 0, (x,t) \in [0,1] \times \Omega,\\ 
    &u(0,x) = -\sin(\pi x),   x \in \Omega,  \\
    &u(t,-1) = u(t,1) = 0,  t \in [0,1], \\
    &\Omega = [-1,1]
\end{align*}
where $\nu = \frac{0.01}{\pi}$. 

The Klein-Gordon equation is 
\begin{align*}
&\Delta u(t,x) + \gamma u^k(t,x) = f(t,x), (t,x) \in [0,T] \times \Omega, \\ 
&u(0,x) = g_1(x), x \in \Omega\\
&u_t(0,x) = g_2(x), x \in \Omega \\
&u(t,x) = h(t,x), (t,x) \in [0,T] \times \partial \Omega \\
& \Omega = [0,1]
\end{align*}
We set parameters to $k=3, \gamma = 1, T=1$ and the initial conditions, $g_1(x) = x, g_2(x) = 0$ for all $x \in \Omega$ following \cite{wang2021understanding}. Then we can use the solution $u^*(t,x) = x\cos(5\pi t) + (tx)^3$ where $f(t,x)$ is derived by given equation .

We employ a 3-layer fully connected neural network with 50 neurons per layer and use the hyperbolic tangent activation function for all experiments in Section~\ref{subsec:benchmark}. At each iteration, 128 points are randomly sampled in boundaries and 10 times more points in the domain as the collocation points. We just randomly sample the points in the boundaries if there exists an analytic solution, otherwise the points were resampled from a pre-generated set for each iteration. More specifically, for the case of Viscous Burger's equation, there is pre-determined 456 boundary points and we randomly sample in this set of points. We train PINNs for 50,000 epochs with Glorot normal initialization \cite{glorot2010understanding} using DCGD algorithms, Adam~\cite{kingma2014adam}, LRA \cite{wang2021understanding}, NTK \cite{wang2022NTK}, PCGrad \cite{Yu2020PCGrad}, MultiAdam \cite{pmlr-v202-yao23c}, and DPM \cite{kim2021dpm}. 

We search for the initial learning rate among $\lambda = \{10^{-3}, 10^{-4}, 10^{-5}\}$ and use a exponential decay scheduler with a decay rate of $0.9$ and a decay step = $1,000$. For Adam, we use the default parameters: $\beta_1=0.9$, $\beta_2=0.999$, $\epsilon=10^{-8}$ as in \cite{kingma2014adam}. For LRA, we set $\alpha=0.1$, which is the best hyperparameter reported in \cite{wang2021understanding}. For MultiAdam, we use $\beta_1,\beta_2 = 0.99$ as recommended in \cite{pmlr-v202-yao23c}. For DPM, we test $\delta = \{10^{-1}, 10^{-2}, 10^{-3}\}, \epsilon = \{10^{-1}, 10^{-2}, 10^{-3}\}, w = \{1, 1.01, 1.001\}$.

To compute the effectiveness of various optimization algorithms, we evaluate the accuracy of the PINN solutions $u(\cdot;\theta)$ using the relative $L^2$-error defined as:
$$ \text{Relative } L^2 \text{ error } = \frac{\sqrt{\sum^{N}_{i=1} |u(\vx_i;\theta)-u(\vx_i)|^2}}{\sqrt{\sum^{N}_{i=1} |u(\vx_i)|^2}} $$
where $u(\cdot)$ is the true solution and $\{\vx_i\}_{i=1}^N$ is the set of test samples. Unless the equation has an analytic solution, we use the numerical reference solution for $u(\vx)$, which solved by finite element method~\cite{raissi2019physics}.

In Table~\ref{tab:MaxMinBenchmarks}, we report the best and worst-case relative $L^2$ errors of each method across 10 independent trials. 

\begin{table*}[htb!]
\begin{center} 
\begin{tabular}{lcccccc}
\toprule
Equation           & \multicolumn{2}{c}{Helmholtz} & \multicolumn{2}{c}{Burgers'} & \multicolumn{2}{c}{Klein-Gordon} \\ \midrule
Optimizer          & Max        & Min        & Max        & Min        & Max        & Min        \\ \midrule
Adam               & 0.1053      & 0.0315     & 0.1413      & 0.0413     & 0.1586    & 0.0376     \\
LRA                & \underline{0.0108} & \underline{0.0032} & 0.0391     & \textbf{0.0080} & \underline{0.0166} & \textbf{0.0037} \\
NTK                & 0.0532     & 0.0225     & 0.0358     & 0.0148     & 0.0581     & 0.0078     \\
PCGrad             & 0.0170     & 0.0070     & \underline{0.0322} & 0.0091     & 0.0399     & 0.0156     \\
MGDA               & 1.0000      & 0.3441      & 1.0617      & 0.9037      & 1.0245     & 0.2168      \\
CAGrad             & 0.1550      & 0.0330     & 0.0485     & 0.0235     & 0.2845      & 0.0872    \\
Aligned-MTL        & 0.7784      & 0.5062      & 0.0640     & 0.0152     & 0.9133     & 0.2922     \\
MultiAdam          & 0.0249     & 0.0149     & 0.1506      & 0.0537     & 0.0273     & 0.0160     \\ 
DCGD (Center)      & \textbf{0.0038} & \textbf{0.0019} & \textbf{0.0016} & \underline{0.0096} & \textbf{0.0112} & \underline{0.0042} \\ \midrule\midrule
DCGD (Center) + LRA & \textcolor{red}{0.0036} & \textcolor{red}{0.0013} & \textcolor{red}{0.0150} & \textcolor{red}{0.0056} & \textcolor{red}{0.0068} & 0.0036 \\
DCGD (Center) + NTK & 0.0157     & 0.0033     & 0.0175     & 0.0065     & 0.0079     & \textcolor{red}{0.0035} \\    
\bottomrule
\end{tabular}
\caption{\label{tab:MaxMinBenchmarks}
Maximum and minimum of relative $L^2$ errors in 10 independent trials for each algorithm. 
}
\end{center}
\end{table*}

We plot the exact solution, PINN solution, and its error for each benchmark equation in Figures~\ref{fig:Helmholtz_error}, ~\ref{fig:Klein_Gordon_error}, and ~\ref{fig:Burgers_error}

\begin{figure}[htb!]
    \centering 
    \begin{subfigure}[b]{0.3\textwidth}
        \includegraphics[width=\textwidth]{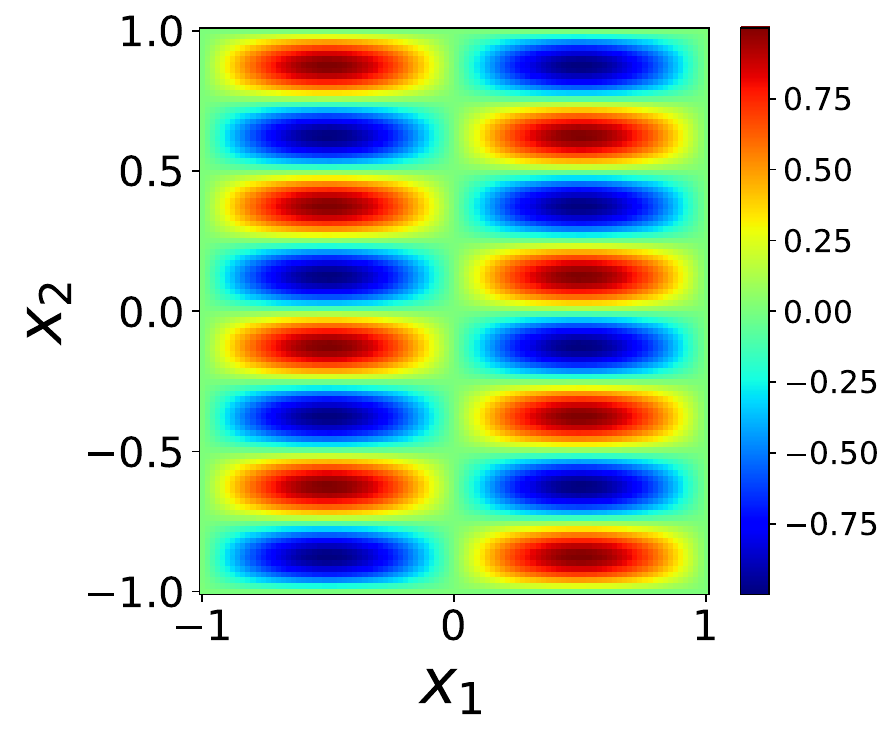}
        \caption{Exact solution}
        %\label{fig:subfig1}
    \end{subfigure}
    \begin{subfigure}[b]{0.3\textwidth}
        \includegraphics[width=\textwidth]{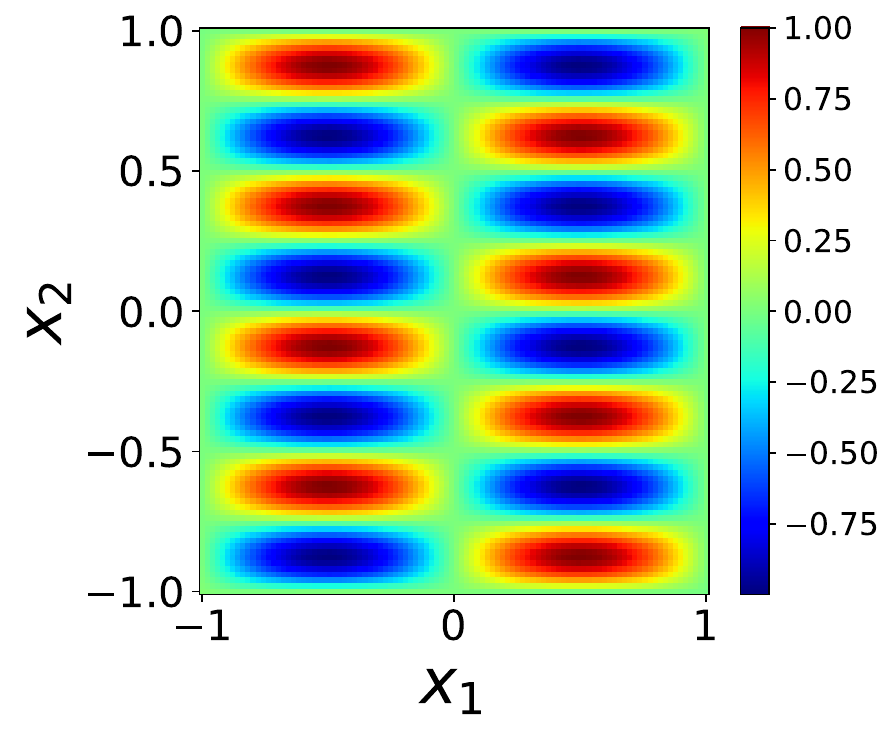}
        \caption{Prediction}
        %\label{fig:subfig2}
    \end{subfigure} 
    \begin{subfigure}[b]{0.3\textwidth}
        \includegraphics[width=\textwidth]{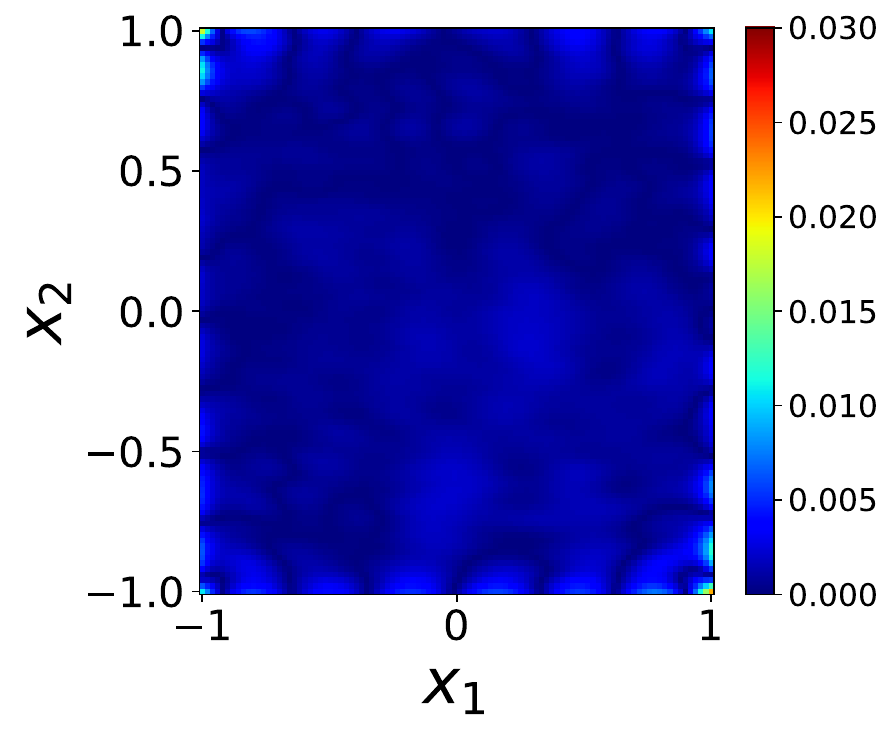}
        \caption{Absolute error}
        %\label{fig:subfig3}
    \end{subfigure}
    \caption{Helmholtz equation: approximated solution versus the reference solution.}
    \label{fig:Burgers_error}
\end{figure}

\begin{figure}[htb!]
    \centering 
    \begin{subfigure}[b]{0.3\textwidth}
        \includegraphics[width=\textwidth]{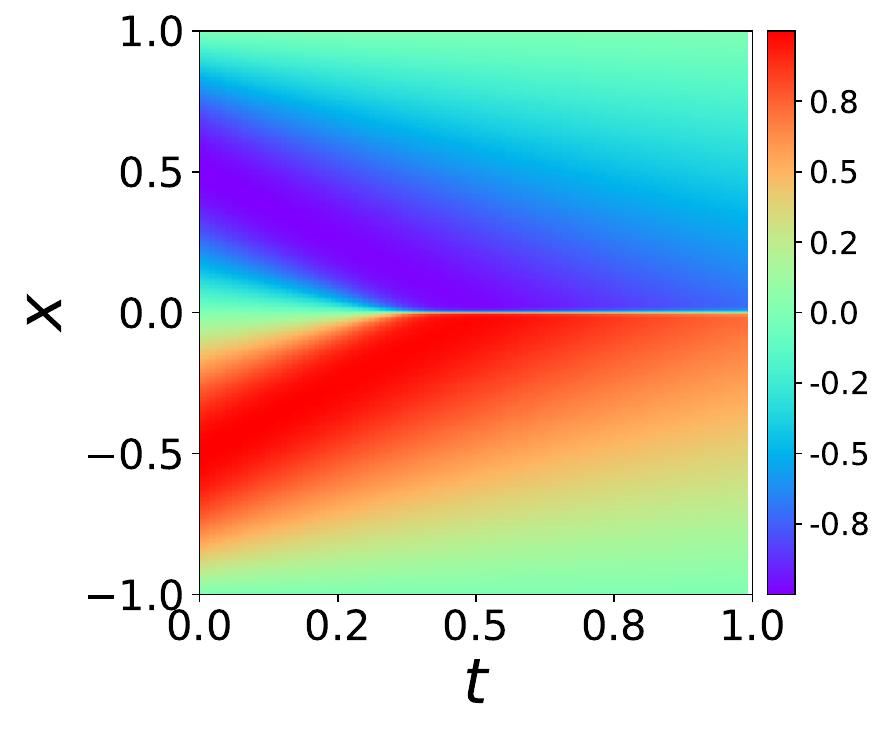}
        \caption{Exact solution}
       % \label{fig:subfig1}
    \end{subfigure}
    \begin{subfigure}[b]{0.3\textwidth}
        \includegraphics[width=\textwidth]{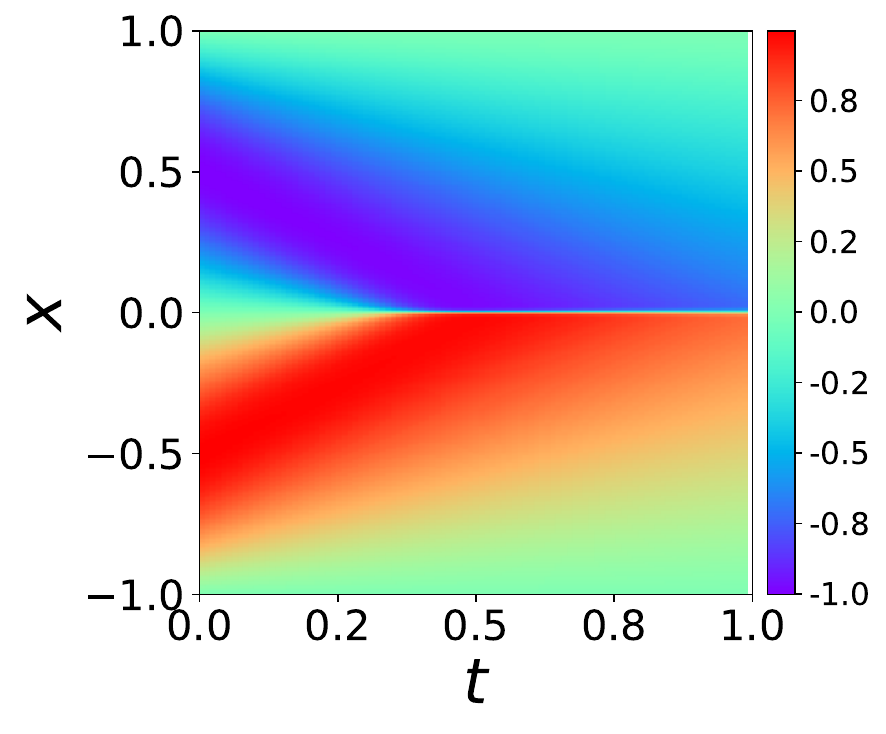}
        \caption{Prediction}
        %\label{fig:subfig2}
    \end{subfigure} 
    \begin{subfigure}[b]{0.3\textwidth}
        \includegraphics[width=\textwidth]{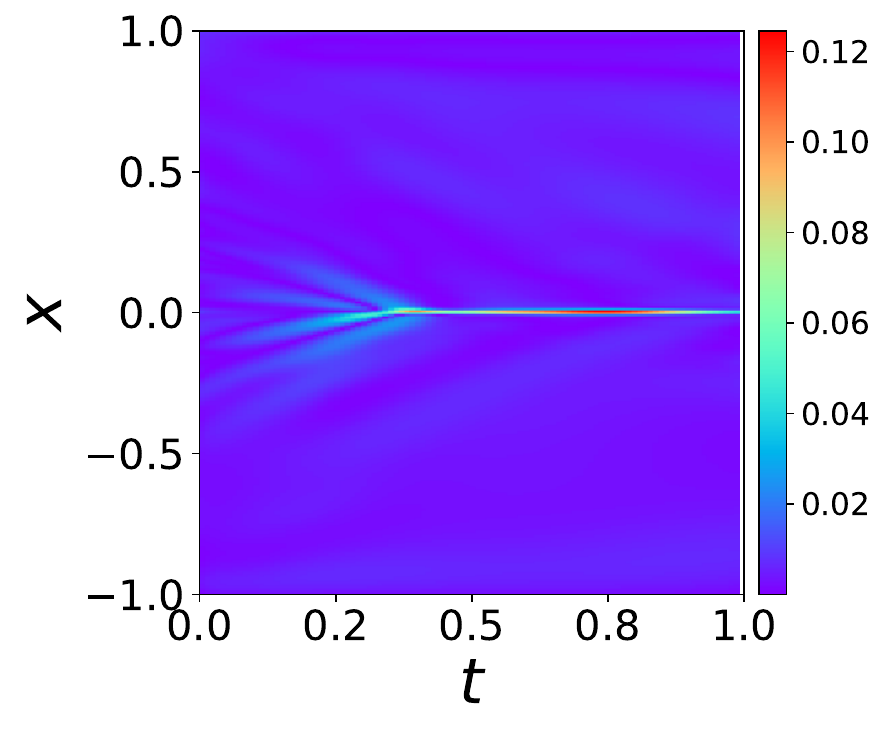}
        \caption{Absolute error}
        %\label{fig:subfig3}
    \end{subfigure}
    \caption{Burgers' equation: approximated solution versus the reference solution.}
    \label{fig:Helmholtz_error}
\end{figure}

\begin{figure}[htb!]
    \centering 
    \begin{subfigure}[b]{0.3\textwidth}
        \includegraphics[width=\textwidth]{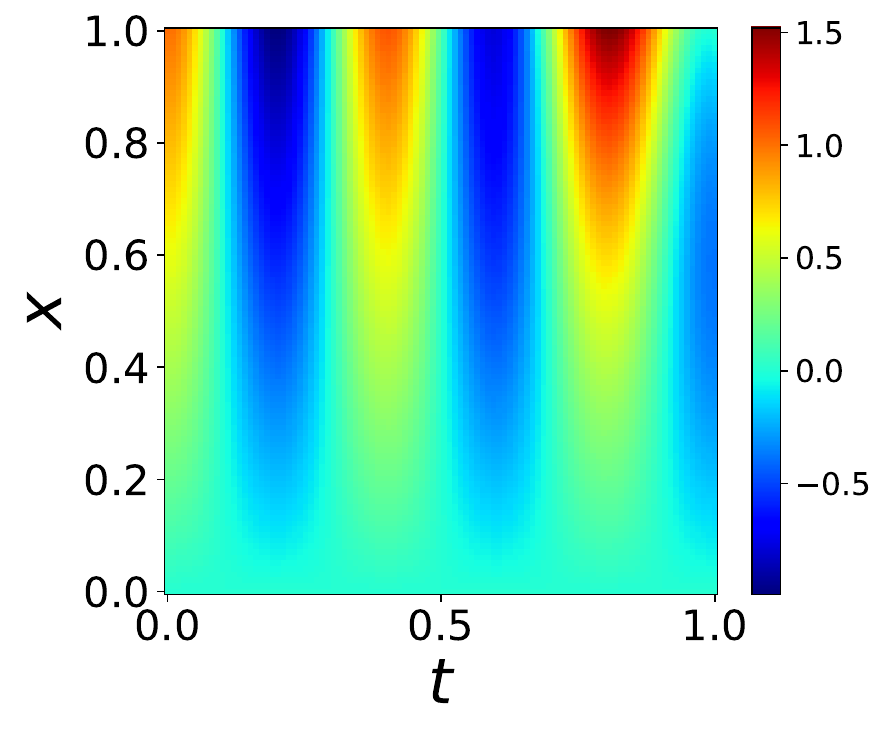}
        \caption{Exact solution}
        %\label{fig:subfig1}
    \end{subfigure}
    \begin{subfigure}[b]{0.3\textwidth}
        \includegraphics[width=\textwidth]{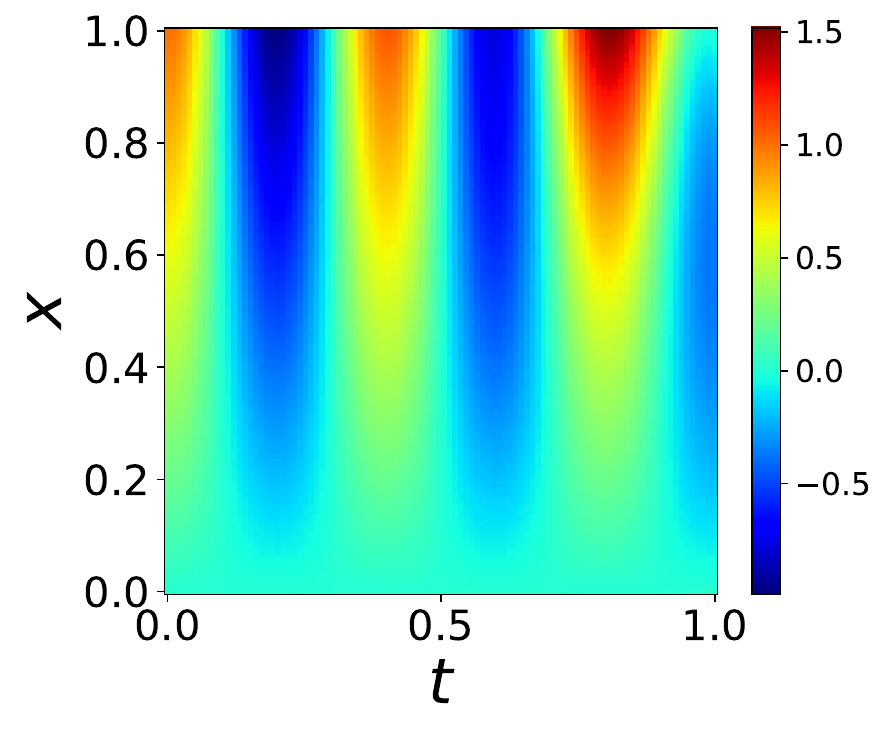}
        \caption{Prediction}
        %\label{fig:subfig2}
    \end{subfigure} 
    \begin{subfigure}[b]{0.3\textwidth}
        \includegraphics[width=\textwidth]{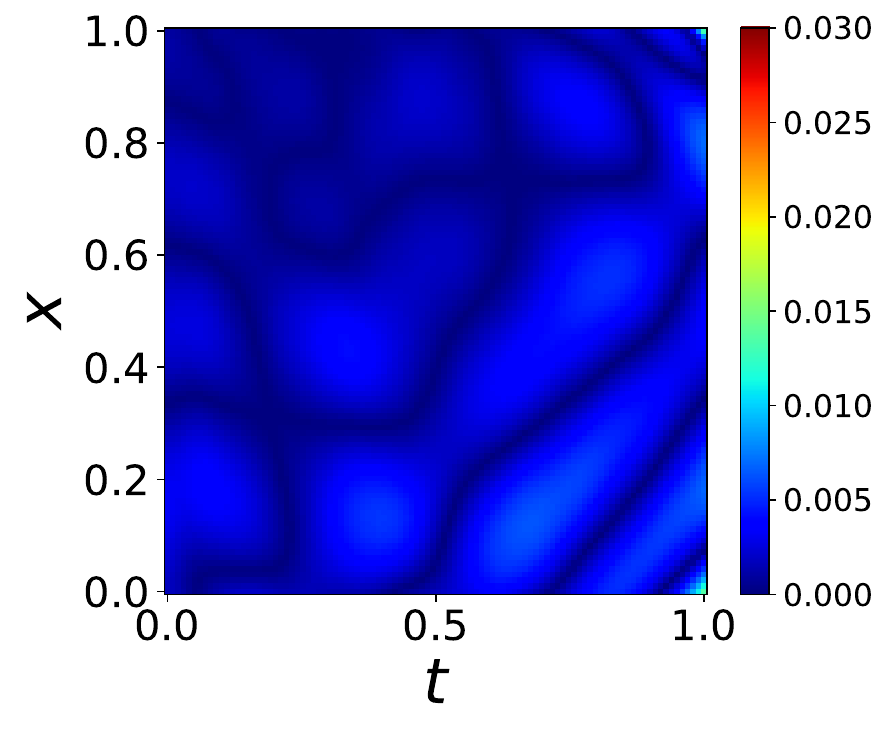}
        \caption{Absolute error}
        %\label{fig:subfig3}
    \end{subfigure}
    \caption{Klein-Gordon equation: approximated solution versus the reference solution.}
    \label{fig:Klein_Gordon_error}
\end{figure}

\paragraph{High-dimensional equations}

We consider the following $3$-dimensional Helmholtz equation
\begin{align*}
    &\Delta u(x,y,z) +k^2u(x,y,z) = f(x,y,z),  \quad &(x,y,z) \in \Omega, \\ 
    &u(x,y,z) = 0, \quad &(x,y,z) \in \partial \Omega, \\
    &\Omega = [-1,1]^3.
\end{align*}
The solution is given by 
$u^*(x,y) = \sin(a_1\pi x)\sin(a_2 \pi y)\sin(a_3 \pi z)$
where
\begin{align*}
f(x,y,z) =(k^2-a_1^2\pi^2-a_2^2\pi^2-a_3^2\pi^2)\sin(a_1\pi x)\sin(a_2 \pi y)\sin(a_3 \pi z)
\end{align*}
with $k=1, a_1=4, a_2=4, a_3=3$. 

We employ a 5-layer fully connected neural network with 128 neurons per layer and use the hyperbolic tangent activation function. At each iteration, 128 points are randomly sampled in boundaries and 500 times more points in the domain as the collocation points. We train PINNs for 30,000 epochs with Glorot normal initialization.

We use initial learning rate among $\lambda = 10^{-3}$ and use a exponential decay scheduler with a decay rate of $0.9$ and a decay step = $1,000$.

For $5$-dimensional Heat equation, we follow the experiment setting in \citet{hao2023pinnacle}. The PDE can be expressed as following: 

\begin{align*}
&u_t = k\Delta u + f(x,t), & x \in \Omega \times [0, 1] \\
&\mathbf{n} \cdot \nabla u = g(x,t), & x \in \partial\Omega \times [0, 1] \\
&u(x,0) = g(x,0), & x \in \Omega
\end{align*}
where the geometric domain $\Omega = \{x: ||x|| \leq 1\}$ and 

\begin{align*}
    f(x,t) &:= -\frac{1}{d} ||x||^2 \exp\left( -\frac{1}{2} ||x||^2 + t \right) \\
    g(x,t) &:= \exp\left( -\frac{1}{2} ||x||^2 + t \right)
\end{align*}

\subsection{Details for Section~\ref{subsec:complex PDEs}}\label{app:complex PDEs}

\paragraph{Double pendulum problem}\label{app:pendulum}
Consider the double pendulum which have two point mass pendulums with masses $m_1, m_2$, two rod with length $l_1, l_2$. Let $\theta_1, \theta_2$ is the angle that the pendulums each make with the vertical and $\Delta \theta = \theta_1 - \theta_2$. Set the gravitational acceleration $g = 9.81$.  (see \hyperref[fig:doublependulum]{Figure~\ref*{fig:doublependulum}}) 

\begin{figure}[htb!]
    \centering 
    \begin{tabular}{cc}
        \includegraphics[width=0.4\textwidth ]{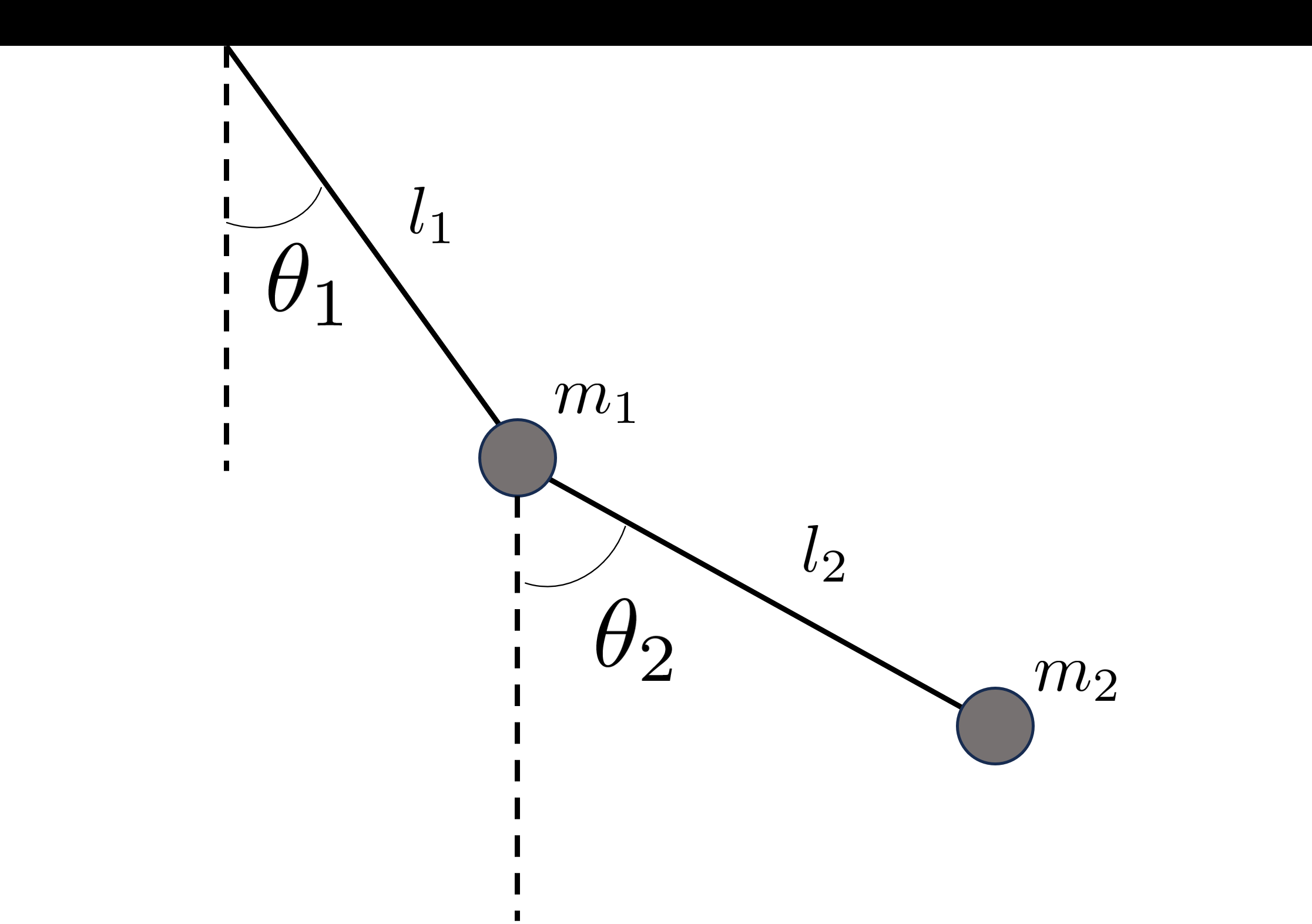}
    \end{tabular}
\caption{Simple double pendulum example}\label{fig:doublependulum}
\end{figure}

Then the dynamics of double pendulum can be described by following nonlinear differential equation system with $y=[\theta_1, \theta_2]^T$:
\begin{equation}\label{eq:double_pendulum}
y'' = \begin{bmatrix}
f_1(y, y') \\
f_2(y, y') \\
\end{bmatrix} \; \text{subject to} \;
y(t_0)= \begin{bmatrix}
\theta_1(t_0)\\
\theta_2(t_0)\\
\end{bmatrix},
y'(t_0)= \begin{bmatrix}
\omega_1(t_0)\\
\omega_2 (t_0)\\
\end{bmatrix},
\end{equation}
where
\begin{align*}
&\omega_1 = \dot{\theta}_1, \\
&\omega_2 = \dot{\theta}_2, \\
&f_1(y,y'') = \frac{m_2 l_1 \omega_1^2 \sin(2\Delta\theta) + 2m_2 l_2 \omega_2^2 \sin \Delta\theta + 2g m_2 \cos \theta_2 \sin \Delta\theta + 2g m_1 \sin \theta_1}{2l_1(m_1 + m_2 \sin^2 \Delta\theta)}, \\
&f_2(y,y'') = \frac{m_2 l_2 \omega_2^2 \sin(2\Delta\theta) + 2(m_1 + m_2) l_1 \omega_1^2 \sin \Delta\theta + 2g(m_1 + m_2) \cos \theta_1 \sin \Delta\theta}{2l_2(m_1 + m_2 \sin^2 \Delta\theta)}.
\end{align*}

In this experiment, the initial conditions are $\theta_1(t_0) = \theta_2(t_0) = \theta_0 = 150^\circ $, $\omega_1(t_0) = \omega_2(t_0) = 0$. 

We replicate the experiment setup done in \cite{steger2022how}. We use a six-layer feed-forward network of 30 neurons on each layer with the swish activation function. We train this model with Adam optimizer with the default hyperparameters for 20,000 epochs. We also train the same model using DCGD (Center). Figure~\ref{fig:pendulum_loss_curve} shows the training curves of ADAM and DCGD. 

\begin{figure}[htb!]
    \centering 
    \begin{subfigure}[b]{0.45\textwidth}
        \includegraphics[width=\textwidth]{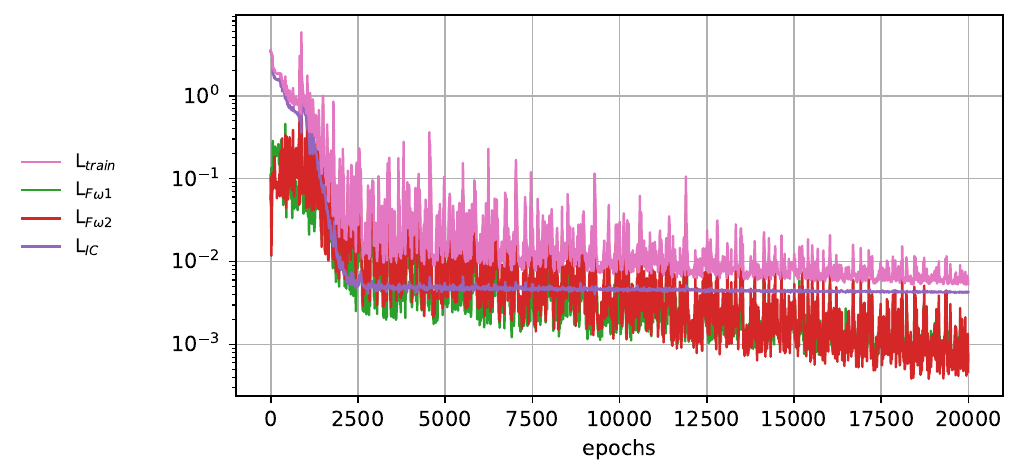}
        \caption{ADAM}
        \label{fig:pendulum_loss_curve_adam}
    \end{subfigure}
        \begin{subfigure}[b]{0.45\textwidth}
        \includegraphics[width=\textwidth]{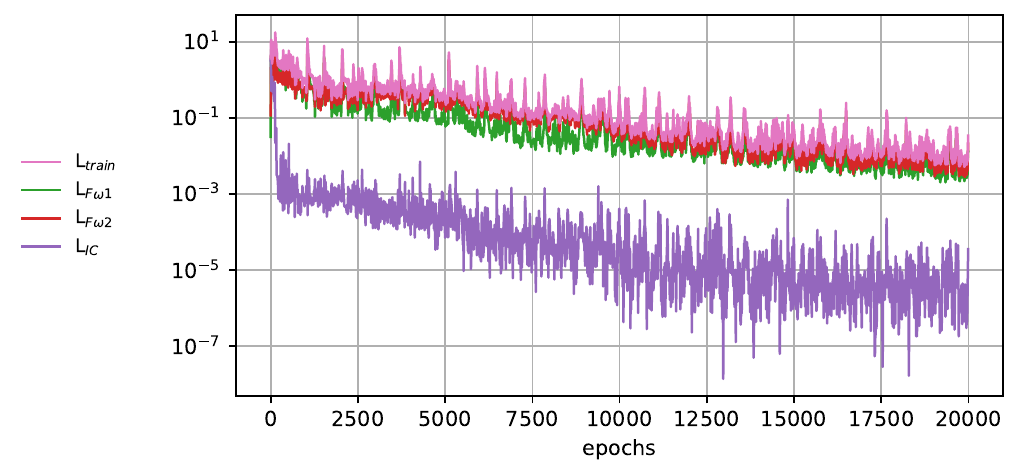}
        \caption{DCGD (Center)}
        \label{fig:pendulum_loss_curve_dcgd}
    \end{subfigure}
\caption{Loss trajectory of each method in the double pendulum problem.}\label{fig:pendulum_loss_curve}
\end{figure}

\paragraph{Convection equation}
We train PINNsFormer in \citet{zhao2023pinnsformer} to solve convection equation which can expressed as following: 
\begin{align*}
     &u_t + \beta u_x = 0, \quad \forall x \in [0, 2\pi], t \in [0, 1] \\
     &u(x,0) = \sin(x), \\
     &u(0,t) = u(2\pi,t)
\end{align*}
Where $\beta=50$. we follow the default setting of \cite{zhao2023pinnsformer} and train the model by 500 epochs.

\paragraph{chotic Kuramoto-Sivashinsky equation}
We use causal training in \citet{wang2022respecting} to solve chaotic Kuramoto-Sivashinsky equation. We use 5 layers modifed-MLP with 64 neurons per layer. and train this model 50,000 epochs for each tolerance.

\begin{align*}
     &u_t + \alpha uu_x + \beta u_{xx} + \gamma u_{xxxx}= 0, \quad \forall x \in [0, 2\pi], t \in [0, 0.5] \\
     &u(0,x) = \cos(x)(1+\sin(x))
\end{align*}
where $\alpha=100/16, \beta=100/16^2, \gamma=100/16^4$.

\begin{figure}[htb!]
    \centering 
    \begin{subfigure}[b]{0.3\textwidth}
        \includegraphics[width=\textwidth]{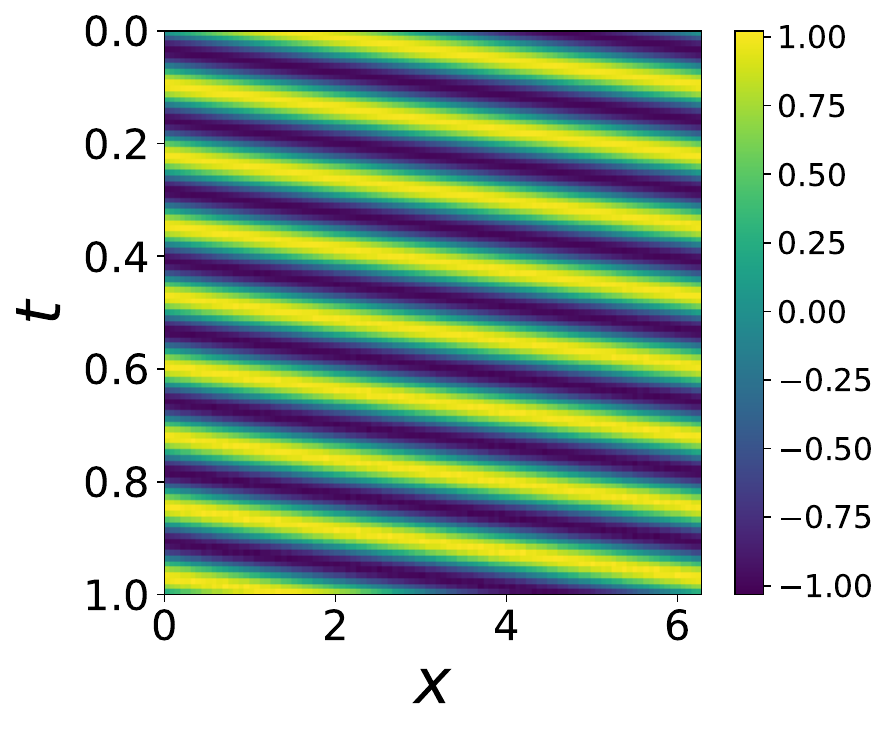}
        \caption{Exact solution}
        %\label{fig:subfig1}
    \end{subfigure}
    \begin{subfigure}[b]{0.3\textwidth}
        \includegraphics[width=\textwidth]{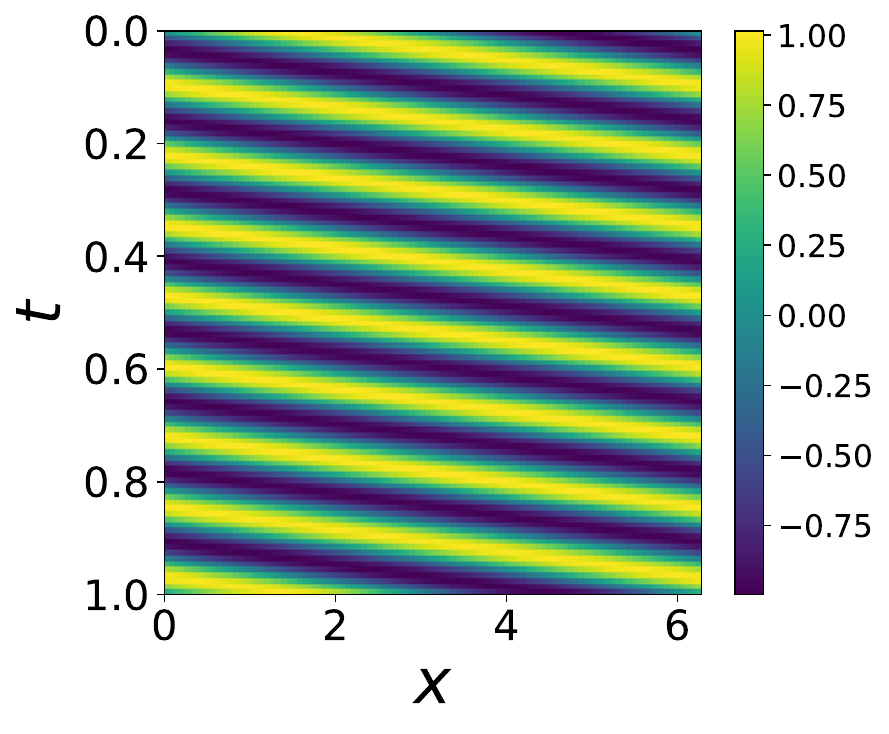}
        \caption{Prediction}
        %\label{fig:subfig2}
    \end{subfigure} 
    \begin{subfigure}[b]{0.3\textwidth}
        \includegraphics[width=\textwidth]{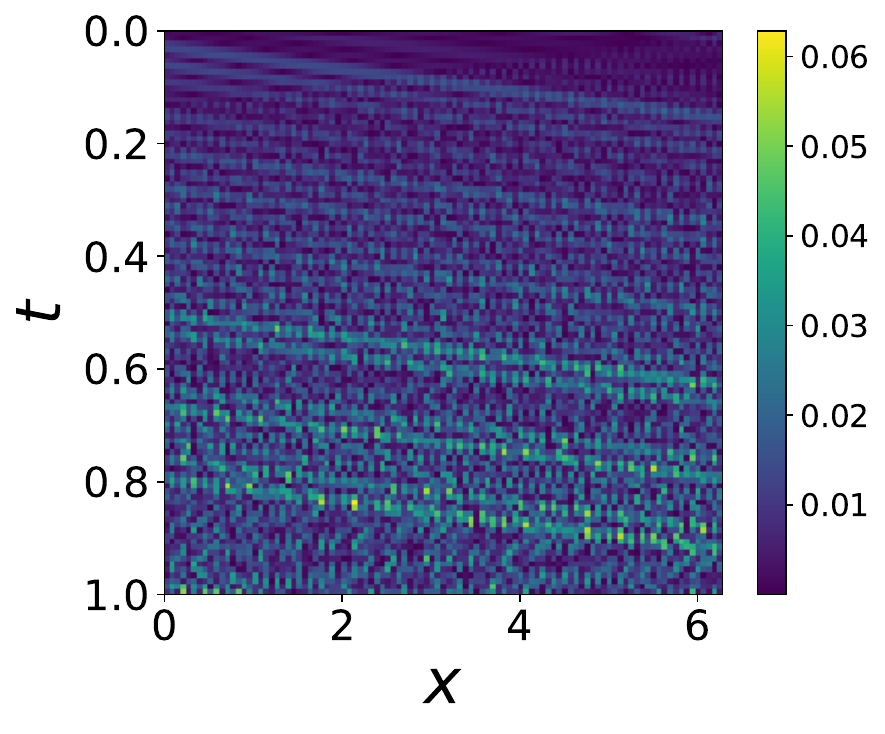}
        \caption{Absolute error}
        %\label{fig:subfig3}
    \end{subfigure}
    \caption{Convection equation: approximated solution versus the reference solution.}
    \label{fig:Convection_error}
\end{figure}

\paragraph{Auxiliary-PINN: Nonlinear integro-differential equation}\label{app:ide}

\begin{figure}[htb!]
    \centering 
    \begin{subfigure}[b]{0.3\textwidth}
        \includegraphics[width=\textwidth]{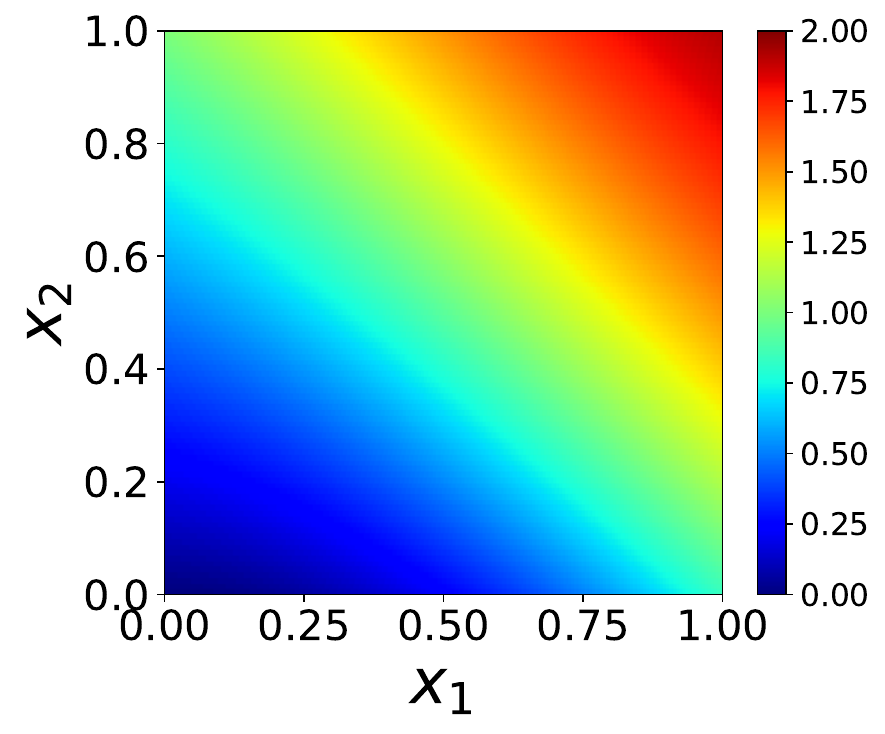}
        \caption{Exact solution}
        %\label{fig:subfig1}
    \end{subfigure}
    \begin{subfigure}[b]{0.3\textwidth}
        \includegraphics[width=\textwidth]{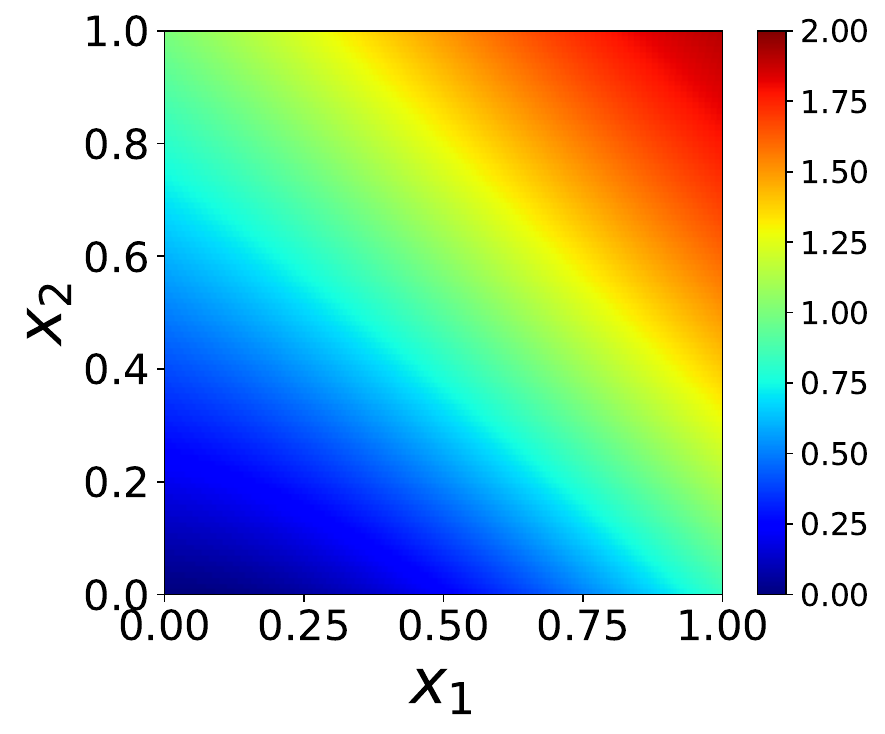}
        \caption{Prediction}
        %\label{fig:subfig2}
    \end{subfigure} 
    \begin{subfigure}[b]{0.3\textwidth}
        \includegraphics[width=\textwidth]{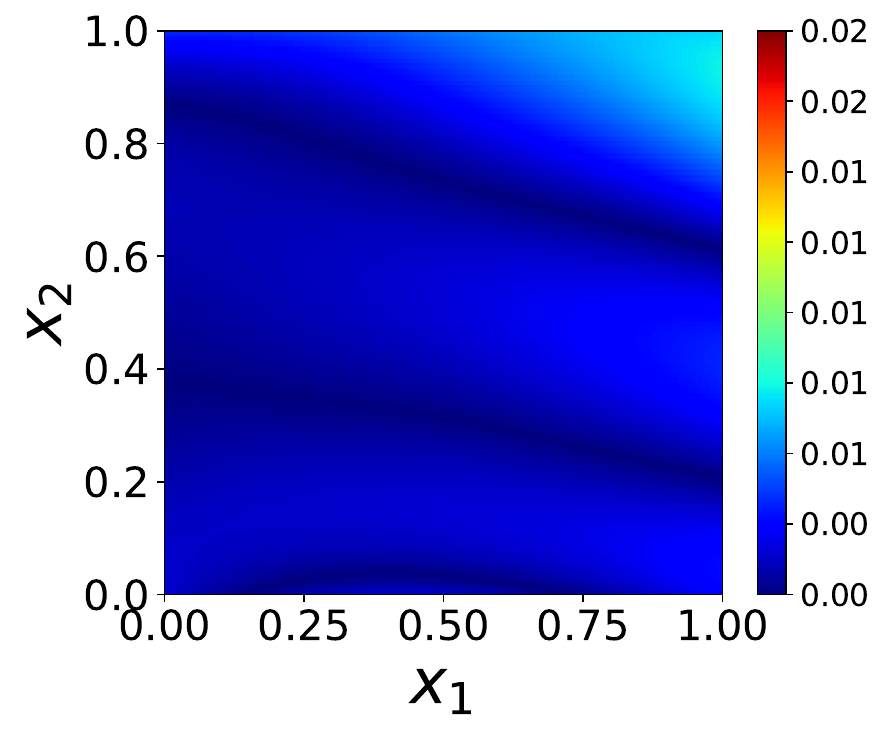}
        \caption{Absolute error}
        %\label{fig:subfig3}
    \end{subfigure}
    \caption{2D-Volterra equation: approximated solution versus the reference solution.}
    \label{fig:Volterra_error}
\end{figure}

A-PINN is a variant of PINNs, designed to solve integro-differential equations \cite{yuan2022pinn}. We apply our DCGD algorithms to A-PINN for solving the following nonlinear 2-dimensional Volterra IDE:
\begin{align*}
    \frac{\partial^2 u(t, x)}{\partial t^2} = \frac{\partial u(t, x)}{\partial x} - \frac{\partial u(t, x)}{\partial t} - u(t, x) + g(t, x) + \lambda \int_{0}^{x} \int_{0}^{t} f\cos(y_1 - y_2)u(y_1, y_2)dy_1dy_2
\end{align*}
where the boundary conditions are $u(0, x) = x$, $\frac{\partial u(0, x)}{\partial t} = \sin(x)$, and $u(t, 0) = t \sin(t)$,  $0 \leq t,\, x \leq 1$. The analytic solution is $u^*(t,x)= x+t\sin(t+x)$ with $\lambda =1$ where $g(t,x)$ is derived by given equation. 

Within the framework of A-PINN, it converts the above integral equation into the following equation by representing integrals as auxiliary output variables:
\begin{align*}
&\frac{\partial^2 u(t,x)}{\partial t^2} = \frac{\partial u(t,x)}{\partial x} - \frac{\partial u(t,x)}{\partial t} - u(t,x) + g(t,x) + \lambda v(t,x), \\
&\frac{\partial v(t,x)}{\partial x} = \int_{0}^{t} f \cos(y_1 - x) u(y_1,x) dy_1 = t w(t,x), \\
&\frac{\partial w(t,x)}{\partial t} = \cos(t - x) u(t,x),
\end{align*}
where the new variables $v$ and $w$ satisfies the boundary condition $v(t,0) = 0$, $w(0,x)=0$.

For A-PINNs, we employ a 3-layer fully connected neural network with 50 neurons per layer and a hyperbolic tangent activation function. For training, 128 points are randomly sampled in boundaries and 10 times more points in the domain as the collocation points in each epochs. We train A-PINN models for $5,000$ epochs.

%We use the loss balancing scheme in proposed in \cite{yuan2022pinn} only when training with ADAM algorithm. Also, we use the recommend hyperparameter ($\alpha = 0.1$ for LRA, $\beta_1,\beta_2 = 0.99$ for MultiAdam).

%\begin{table*}[htb!]
%\begin{center}
%\begin{tabular}{lcc}
%Optimizer          & Max      & Min         \\ \hline
%Adam               & 1.37E-02 & 1.75E-03          \\
%LRA                & 2.48E-03 & 5.34E-04          \\
%NTK                & 3.43E-03 & 6.60E-04          \\
%PCGrad             & 4.33E-03 & 4.37E-04          \\
%MultiAdam          & 2.03E-03 & \textbf{2.01E-04}          \\ 
%DCGD (Center)      & \textbf{1.77E-03} & 4.61E-04          \\ \hline
     
%\end{tabular}
%\caption{\label{tab:MaxMinVorterra}
%Min and Max relative $L^2$ errors in 10 independent trials for each algorithm. }
%\end{center}
%\end{table*}

\paragraph{Separable PINN: 3-dimensional Helmholtz equation}\label{app:spinn}

%\begin{table}
%\centering
%\begin{tabular}{ccc} \toprule
%Model & ADAM     & DCGD (Center)                                  \\ \midrule
%\multirow{2}{*}{SPINN $(N_c=16^3)$}   & 5.78E-02  & \textbf{4.47E-02} \\ 
%                                      & (3.86E-03)           &  \textbf{(1.76E-02)} \\ \midrule
%\multirow{2}{*}{SPINN $(N_c=32^3)$}   & 3.52E-02  & \textbf{1.04E-02} \\ 
%                                      & (3.49E-03)           &  \textbf{(4.81E-03)} \\ \midrule
%\multirow{2}{*}{SPINN $(N_c=64^3)$}   & 2.80E-02  & \textbf{3.21E-03} \\ 
%                                      & (6.64E-03)           &  \textbf{(1.80E-04)} \\ \midrule
%\multirow{2}{*}{SPINN $(N_c=128^3)$}  & 2.94E-02  & \textbf{1.54E-03} \\ 
%                                      & (1.23E-02)           &  \textbf{(2.62E-04)} \\ \midrule
%\multirow{2}{*}{SPINN $(N_c=256^3)$}  & 3.19E-02  & \textbf{1.97E-03} \\ 
 %                                     & (2.58E-03)           &  \textbf{(8.97E-04)} \\
%\bottomrule
%\end{tabular}
%\caption{\label{tab:PINNs variants}
%Average and standard deviation of relative $L^2$ errors for DCGD (Center) on A-PINN and SPINN. The value within the parenthesis is the standard %deviation. $N_c$ is the number of collocation points.
%}
%\end{table}

SPINN is a a novel architecture designed to effectively reduce the computational cost of PINNs, especially when addressing high-dimensional PDEs  \cite{Cho2023Separable}. To test the performance of DCGD for SPINNs, we consider the following $3$-dimensional Helmholtz equation
\begin{align*}
    &\Delta u(x,y,z) +k^2u(x,y,z) = f(x,y,z),  \quad (x,y,z) \in \Omega, \\ 
    &u(x,y,z) = 0, \quad (x,y,z) \in \partial \Omega, \\
    &\Omega = [-1,1]^3.
\end{align*}
The solution is given by 
$u^*(x,y) = \sin(a_1\pi x)\sin(a_2 \pi y)\sin(a_3 \pi z)$
where
\begin{align*}
f(x,y,z) =(k^2-a_1^2\pi^2-a_2^2\pi^2-a_3^2\pi^2)\sin(a_1\pi x)\sin(a_2 \pi y)\sin(a_3 \pi z)
\end{align*}
with $k=1, a_1=4, a_2=4, a_3=3$. 

We follow the optimal hyperparameter setting reported in \cite{Cho2023Separable}. For ADAM and DCGD (Center), the learning rate is  $0.001$. The input points are resampled every 100 epochs. Regarding model architecture, we use the SOTA model, so called (SPINN + Modified MLP). We record the mean and standard deviation of relative $L^2$ errors from $3$ independent trials in Table~\ref{tab:SPINN}, indicating that the performance of SPINN can be significantly improved when trained with DCGD for a varying number of collocation points. 

\begin{table*}[htb!]
\begin{center}
\begin{tabular}{lccc}
\toprule
Method                         & $N_c$  & Relative $L^2$ error     &  Training speed    \\ \midrule

\multirow{5}{*}{SPINN}         & $16^3$  & 0.0578 (0.0039)   & 1.65 (ms/iter) \\
                               & $32^3$  & 0.0352 (0.0035)   & 1.78 (ms/iter) \\
                               & $64^3$  & 0.0280 (0.0066)   & 2.38 (ms/iter) \\
                               & $128^3$ & 0.0294 (0.0123)   & 2.71 (ms/iter) \\
                               & $256^3$ & 0.0319 (0.0026)   & 5.12 (ms/iter) \\ \midrule
\multirow{5}{*}{SPINN + DCGD (Center)} & $16^3$  & 0.0447 (0.0176)   & 1.76 (ms/iter) \\
                                      & $32^3$  & 0.0104 (0.0048)   & 1.90 (ms/iter) \\
                                      & $64^3$  & 0.0032 (0.0002)   & 2.59 (ms/iter) \\
                                      & $128^3$ & \textbf{0.0015 (0.0003)}  & 2.85 (ms/iter) \\
                                      & $256^3$ & 0.0020 (0.0009)   & 5.34 (ms/iter) \\ \midrule
\end{tabular}
\end{center}
\caption{Helmholtz Equation (3d): Relative $L^2$ errors and training speed. $N_c$ is the number of collocation points.}
\label{tab:SPINN}
\end{table*}

\begin{figure}[htb!]
    \centering 
    \begin{subfigure}[b]{0.3\textwidth}
        \includegraphics[width=\textwidth]{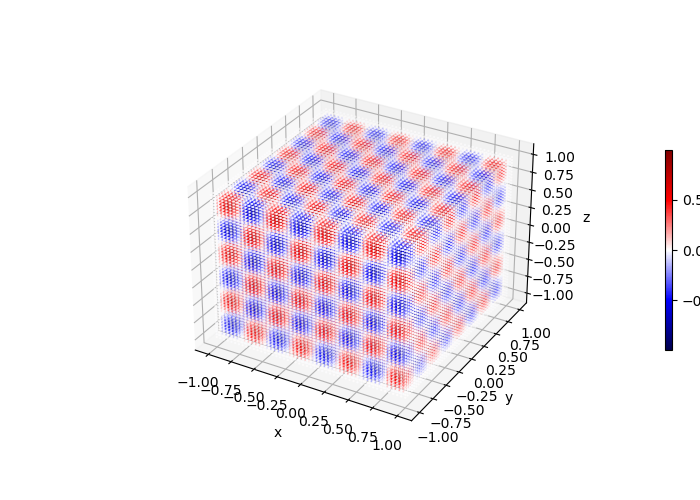}
        \caption{Exact solution}
        %\label{fig:subfig1}
    \end{subfigure}
    \begin{subfigure}[b]{0.3\textwidth}
        \includegraphics[width=\textwidth]{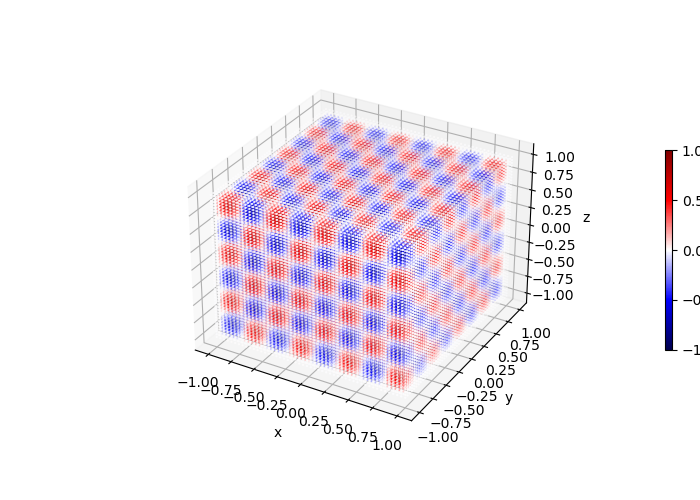}
        \caption{Prediction}
        %\label{fig:subfig2}
    \end{subfigure} 
    \begin{subfigure}[b]{0.3\textwidth}
        \includegraphics[width=\textwidth]{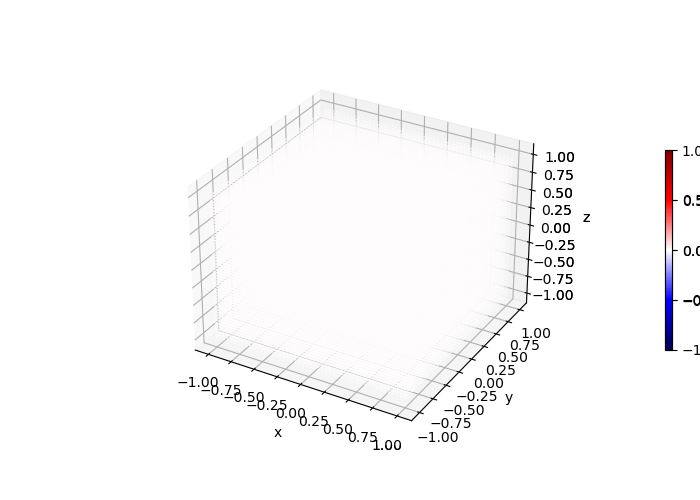}
        \caption{Absolute error}
        %\label{fig:subfig3}
    \end{subfigure}
    \caption{3D-Helmholtz equation: approximated solution versus the reference solution.}
    \label{fig:3dHelmholtz_error}
\end{figure}

\section{Supplemental results}
\subsection{Ablation study}\label{app:ablation study}

In Section~\ref{subsec:algorithms}, we introduce three specific algorithms: DCGD (Projection), DCGD (Average), and DCGD (Center). We conduct an ablation study to investigate the impact of different updated gradient schemes within the dual cone region. More specifically, we compare the performance of these three algorithms on three benchmark equations and nonlinear 2D Volterra integral equations. Detailed experimental settings can be found in Appendices~\ref{app:benchmark} and \ref{app:ide}.

Tables~\ref{tab:benchmark_ablation} and \ref{tab:MaxMin_ablation} demonstrate that DCGD (Center) outperforms the other DCGD algorithms across all experiments. Therefore, we consider DCGD (Center) as the default DCGD algorithm.

\begin{table*}[htb!]
\begin{center}
\begin{tabular}{lccc}\toprule
Equation           & Helmholtz              & Burgers'              & Klein-Gordon        \\ \midrule
Optimizer          & Mean (std)             & Mean (std)            & Mean (std)          \\ \midrule
Projection         & 0.0089 (0.0022)        & 0.0139 (0.0035)       & 0.0216 (0.0130)     \\
Average            & 0.0166 (0.0124)        & 0.0156 (0.0032)       & 0.0292 (0.0088)     \\
Center             & \textbf{0.0029 (0.0005)} & \textbf{0.0124 (0.0046)} & \textbf{0.0069 (0.0027)}   \\ \bottomrule
\end{tabular}
\caption{\label{tab:benchmark_ablation}
Average and standard deviation of relative $L^2$ errors in 10 independent trials for each DCGD algorithm.
}
\end{center}
\end{table*}

\begin{table*}[htb!]
\begin{center} 
\begin{tabular}{lcccccc}\toprule
Equation          & \multicolumn{2}{c}{Helmholtz} & \multicolumn{2}{c}{Burgers'} & \multicolumn{2}{c}{Klein-Gordon} \\ \midrule
Optimizer         & Max       & Min      & Max       & Min       & Max       & Min       \\ \midrule
Projection        & 0.0138    & 0.0062   & 0.0217    & 0.0097    & 0.0573    & 0.0120    \\
Average           & 0.0469    & 0.0078   & 0.0209    & 0.0115    & 0.0442    & 0.0178    \\ 
Center            & \textbf{0.0038} & \textbf{0.0019} & \textbf{0.0163} & \textbf{0.0096} & \textbf{0.0112} & \textbf{0.0042} \\ \bottomrule
\end{tabular}
\caption{
Min and Max Relative $L^2$ errors in 10 independent trials for each DCGD algorithm.
}\label{tab:MaxMin_ablation}
\end{center}
\end{table*}

\subsection{Computational cost}\label{app:compute}
In this section, we acknowledge that our proposed method incurs higher computational costs due to the need for backpropagation for each individual loss. Nonetheless, through a comparison of training speeds, we empirically demonstrate that DCGD achieves superior performance with computational costs comparable to those of existing competitors.

\begin{table*}[htb!]
\begin{center}
\begin{tabular}{lcc} 
\toprule
& \multicolumn{2}{c}{PDE equation} \\ \midrule
Optimizer    & Heat (5D) & Helmholtz (3D)    \\ \midrule
Adam         & 11.1 (iter/s)    & 9.05 (iter/s)   \\
L-BFGS       & 0.53 (iter/s)    & 1.19 (iter/s)   \\
LRA          & 3.39 (iter/s)    & 5.43 (iter/s)  \\
NTK          & 3.92 (iter/s)    & 7.49 (iter/s) \\
MultiAdam    & 4.10 (iter/s)    & 6.85 (iter/s) \\
PCGrad       & 5.98 (iter/s)    & 8.94 (iter/s)  \\
MGDA         & 3.46 (iter/s)      & 6.16 (iter/s)  \\
CAGrad       & 4.22 (iter/s)      & 8.80 (iter/s)\\
Aligned-MTL       & 4.10 (iter/s) & 7.97 (iter/s)\\
DCGD (Average)    & 3.78 (iter/s) & 5.42 (iter/s) \\
DCGD (Projection) & 3.70 (iter/s) & 5.64 (iter/s) \\
DCGD (Center)     & 4.35 (iter/s) &  8.90 (iter/s)\\ \bottomrule

\end{tabular}
\caption{\label{tab:highdim_speed}
Training speed in higher dimensional equations example}
\end{center}
\end{table*}

\section{Pseudo codes of algorithms}\label{app:algorithms}

This section provides pseudo codes for the proposed DCGD algorithms. 

Firstly, DCGD (Projection) uses the projection of the total gradient $\nabla \cL (\theta_t)$ onto $\rmG_t$ when $\nabla \cL (\theta_t) \notin \rmK^*_t$. Otherwise, $\nabla \cL (\theta_t)$ is used. Then the update vector $g_t^{\text{dual}}$ can be defined as follow:
\begin{align}\label{eq:gt_projection}
 \mbox{DCGD (Projection)} \qquad    g_t^{\text{dual}} =
    \begin{cases}
        \nabla \cL (\theta_t), & \mbox{if $\nabla \cL(\theta_t)\in \rmK_t^*$} \\
        \nabla_t \cL_{\|\nabla \cL_r^{\perp}}, & \mbox{if $\nabla \cL(\theta_t)\notin \rmK_t^*$ and $\ip{\nabla \cL (\theta_t)}{ \nabla \cL_r (\theta_t) < 0}$} \\
        \nabla_t \cL_{\|\nabla \cL_b^{\perp}}. & \mbox{if $\nabla \cL(\theta_t)\notin \rmK_t^*$ and $\ip{\nabla \cL (\theta_t)}{ \nabla \cL_b (\theta_t) < 0}$}
    \end{cases}
\end{align}

Secondly, DCGD (Average) uses the the average of $\nabla_t \cL_{\|\nabla \cL_r^{\perp}}$ and $\nabla_t \cL_{\|\nabla \cL_b^{\perp}}$ when the total gradient is outside $\rmK^*_t$. Otherwise, $\nabla \cL (\theta_t)$ is used. The update vector $g_t^{\text{dual}}$ of DCGD (Average) is defined as follow:
\begin{align}\label{eq:gt_average}
    \mbox{DCGD (Average)} \qquad g_t^{\text{dual}} =
    \begin{cases}
        \nabla \cL (\theta_t) & \mbox{if $\nabla \cL(\theta_t)\in \rmK_t^*$} \\
        \frac{1}{2}(\nabla_t \cL_{\|\nabla \cL_r^{\perp}} + \nabla_t \cL_{\|\nabla \cL_b^{\perp}}),  &\mbox{if $\nabla \cL(\theta_t)\notin \rmK_t^*$}         
    \end{cases} 
\end{align}

Thirdly, DCGD (Center) employs the following update vector $g_t^{\text{dual}}$, regardless of whether the total gradient $\nabla \cL (\theta_t)$ is included in $\rmK^*_t$:
\begin{align}
    \mbox{DCGD (Center)} \qquad g_t^{\text{dual}} = \frac{\ip{g_t^c}{\nabla \cL(\theta_t)}}{\|g_t^c\|^2}g_t^c \text{  where } g_t^c = \frac{\nabla \cL_b(\theta_t)}{\|\nabla \cL_b(\theta_t)\|}+\frac{\nabla \cL_r(\theta_t)}{\|\nabla \cL_r(\theta_t)\|}
\end{align}

Here, pseudo codes of these algorithms are summarized in  Algorithms~ \ref{alg:proj}, \ref{alg:avg}, and \ref{alg:center}. Note that we introduce a conflict threshold $\alpha$ as a stopping condition for DCGD algorithms, as they can reach a Pareto-stationary point characterized by $\phi_t= \pi$. That is, the algorithm stops when the parameter converges close to a Pareto-stationary point such that $|\cos(\phi_t)-\pi|<\alpha$. Throughout our experiments, we set $\alpha=10^{-8}$.

\begin{algorithm}[htb!]
   \caption{DCGD (Projection) }
   \label{alg:proj}
\begin{algorithmic}
   \STATE {\bfseries Require:} learning rate $\lambda$, max epoch $T$, initial point $\theta_0$, gradient threshold $\varepsilon$,  conflict threshold $\alpha$
   \FOR{$t=1$ {\bfseries to} $T$}
   \IF{ $\pi-\alpha < \phi_t \leq \pi$ or $\|\nabla \cL(\theta_t)\|<\varepsilon$}
   \STATE \textbf{break}
   \ENDIF
   
   \IF{$\nabla \cL(\theta_t) \notin \rmK^* $}
   \STATE $g^{dual}_t =\nabla \cL(\theta_t)$
   \ELSIF{$\nabla \cL(\theta_t) \in \rmK^* $ and $\ip{\nabla \cL(\theta_t)}{\nabla \cL_r(\theta_t)} < 0$}
   \STATE $g^{dual}_t =\nabla_{t} \cL_{\|\nabla \cL_r^\perp} $
   \ELSIF{$\nabla \cL(\theta_t) \in \rmK^* $ and $\ip{\nabla \cL(\theta_t)}{\nabla \cL_b(\theta_t)} < 0$}
   \STATE $g^{dual}_t =\nabla_{t} \cL_{\|\nabla \cL_b^\perp} $
   \ENDIF
   \ENDFOR
\end{algorithmic}
\end{algorithm}

\begin{algorithm}[htb!]
   \caption{DCGD (Average) }
   \label{alg:avg}
\begin{algorithmic}
   \STATE {\bfseries Require:}  learning rate $\lambda$, max epoch $T$, initial point $\theta_0$, gradient threshold $\varepsilon$,  conflict threshold $\alpha$

   \FOR{$t=1$ {\bfseries to} $T$}
   \IF{ $\pi-\alpha < \phi_t \leq \pi$ or $\|\nabla \cL(\theta_t)\|<\varepsilon$}
   \STATE \textbf{break}
   \ENDIF   
  
   \IF{$\nabla \cL (\theta_t) \notin \rmK^* $}
   \STATE $g^{dual}_t = \frac{1}{2}\nabla_{t} \cL_{\|\nabla \cL_r^\perp} + \frac{1}{2} \nabla_{t} \cL_{\|\nabla \cL_b^\perp}$
   \ELSE
   \STATE $g^{dual}_t = \nabla \cL (\theta_t) $
   \ENDIF
   \STATE $\theta_t = \theta_{t-1} - \lambda g^{dual}_t$

   \ENDFOR
\end{algorithmic}
\end{algorithm}

\begin{algorithm}[htb!]
   \caption{DCGD (Center)}
   \label{alg:center}
\begin{algorithmic}
   \STATE {\bfseries Require:}  learning rate $\lambda$, max epoch $T$, initial point $\theta_0$, gradient threshold $\varepsilon$,  conflict threshold $\alpha$

   \FOR{$t=1$ {\bfseries to} $T$}
   \IF{ $\pi-\alpha < \phi_t \leq \pi$ or $\|\nabla \cL(\theta_t)\|<\varepsilon$}
   \STATE \textbf{break}
   \ENDIF      
    
   \STATE $g_t^c =  \frac{\nabla\cL_b(\theta_t)}{\|\nabla\cL_b(\theta_t)  \|} + \frac{\nabla\cL_r(\theta_t) }{\|\nabla\cL_r(\theta_t) \|}$
   \STATE $g^{dual}_t = \frac{\ip{g_t^c}{\nabla \cL (\theta_t)}}{\|g_t^c\|^2}g_t^c$
   \STATE $\theta_t = \theta_{t-1} - \lambda g^{dual}_t$

   \ENDFOR
\end{algorithmic}
\end{algorithm}

DCGD algorithms can be easily combined with other optimizers or strategies thanks to its flexible framework. For example, one can design a DCGD algorithm combined with ADAM to leverage advantages of adaptive gradient methods, see Algo.~\ref{alg:dcgd_adam}. For our experiments, we consider the DCGD (center) combined with ADAM as the default. Algorithm~\ref{alg:DCGD_lossbalancing} presents the psedo code for DCGD combined with a loss balancing method such as LRA and NTK.

\begin{algorithm}[htb!]
   \caption{DCGD with Adam }
   \label{alg:dcgd_adam}
\begin{algorithmic}
   \STATE {\bfseries Require:} learning rate $\lambda$, max epoch $T$, betas $\beta_1, \beta_2$, DCGD operator $\text{DCGD}(\cdot)$
   \FOR{$t=1$ {\bfseries to} $T$}
   \STATE $g^{dual}_t = \text{DCGD}(\cL_r(\theta),\cL_b(\theta))$
   \STATE $g_t \leftarrow g^{dual}_t$
   \STATE $m_t \leftarrow \beta_1 m_{t-1} + (1 - \beta_1)g_t$
   \STATE $v_t \leftarrow \beta_2 v_{t-1} + (1 - \beta_2)g_t^2$
   \STATE $\hat{m}_t \leftarrow \frac{m_t}{1 - \beta_1^t}$
   \STATE $\theta_t \leftarrow \theta_{t-1} - \gamma_t \frac{\hat{m}_t}{\sqrt{\hat{v}_t} + \epsilon}$
   \ENDFOR
\end{algorithmic}
\end{algorithm}

\begin{algorithm}[htb!]
   \caption{DCGD with a loss balancing method}
   \label{alg:DCGD_lossbalancing}
\begin{algorithmic}
   \STATE {\bfseries Require:} learning rate $\lambda$, max epoch $T$, loss balancing operator $\text{LB}(\cdot)$
   \FOR{$t=1$ {\bfseries to} $T$}
   \STATE $(\beta_r,\beta_b)_t = \text{LB}(\cL_r(\theta_t), \cL_b(\theta_t))$
   \STATE $\cL_b(\theta_t)  \leftarrow \beta_b\cL_b(\theta_t)$
   \STATE $\cL_r(\theta_t)  \leftarrow \beta_r\cL_r(\theta_t) $
   \STATE Choose $g_t^{\text{dual}} \in \rmK^*_t$
   \STATE $\theta_t = \theta_{t-1} - \lambda g_t^{\text{dual}}$
   \ENDFOR
\end{algorithmic}
\end{algorithm}

\end{document}